\def\inclapp{1}
\def\viewchanges{1}

\def\usehyperlinks{1}

\def\preprint{1}

\if\preprint1
    \documentclass{article}
\else
    \documentclass[USenglish]{article}	
\fi

\usepackage[dvipsnames]{xcolor}

\usepackage[backgroundcolor=white,bordercolor=orange]{todonotes}

\if\preprint1
    \usepackage[preprint]{tmlr}
    \usepackage{natbib}
    \usepackage[utf8]{inputenc} 
    \usepackage[T1]{fontenc}    
\else
    \usepackage[utf8]{inputenc}				
    \usepackage[big,online]{dgruyter}	
    \usepackage{lmodern} 
    \usepackage{microtype}
    \usepackage[numbers,square,sort&compress]{natbib}
    \usepackage{anyfontsize}
\fi

\usepackage{booktabs}       
\usepackage{amsfonts}       
\usepackage{nicefrac}       
\usepackage{microtype}      
\usepackage{amsthm}
\usepackage{amsmath}
\usepackage{amssymb}
\usepackage{enumitem}
\usepackage{bbm}
\usepackage{url}            

\usepackage{breakurl}

\if\preprint1
    \if\usehyperlinks1
    	\usepackage[colorlinks, citecolor=blue, linkcolor=magenta, breaklinks]{hyperref}       
    \fi
\fi

\usepackage{graphicx,wrapfig,lipsum}
\usepackage{subfigure}
\usepackage{verbatim}
\usepackage{bm}
\usepackage{adjustbox}
\usepackage{footnote}
\usepackage{algorithm}
\usepackage{algorithmic}
\usepackage{tikz}
\usepackage{cleveref}
\usepackage{multirow}

\if\inclapp0
	\usepackage{xr}
\fi

\if\preprint1
    \newtheorem{theorem}{Theorem}[section]
    \newtheorem{lemma}[theorem]{Lemma}
    \newtheorem{definition}[theorem]{Definition}
    \newtheorem{prop}[theorem]{Proposition}
    \newtheorem{rem}[theorem]{Remark}
    \newtheorem{cor}[theorem]{Corollary}
    \newtheorem{example}[theorem]{Example}
    \newtheorem{assumption}{Assumption}
\else
    \theoremstyle{dgthm}
    \newtheorem{theorem}{Theorem}
    \newtheorem{cor}{Corollary}
    \newtheorem{prop}{Proposition}
    \newtheorem{lemma}{Lemma}

    \theoremstyle{dgdef}
    \newtheorem{definition}{Definition}
    \newtheorem{example}{Example}
    \newtheorem{rem}{Remark}
    \newtheorem{assumption}{Assumption}
\fi

\crefname{assumption}{assumption}{assumptions}
\Crefname{assumption}{Assumption}{Assumptions}
\Crefname{rem}{Remark}{Remarks}

\let\P\undefined%
\newcommand{\1}{\mathbf{1}}
\newcommand{\P}{\mathbb{P}}

\newcommand{\A}{\mathcal{A}}
\newcommand{\E}{\mathbb{E}} 
\newcommand{\N}{\mathbb{N}} 
\newcommand{\R}{\mathbb{R}} 
\newcommand{\argmin}{\operatorname{argmin}} 
\newcommand{\id}{\operatorname{id}} 
 
\newcommand{\omb}{(\omega)}
\newcommand{\proj}[1]{\operatorname{proj}_{#1}}
\newcommand{\ind}[1]{\mathbbm{1}_{#1}}

\let\del\undefined
\let\com\undefined

\if\viewchanges1
	
	\newcommand{\del}[1]{{\color{red}{#1}}}
	\newcommand{\com}[1]{{\color{orange}{#1}}}
	
\else
	
	\newcommand{\del}[1]{}
	\newcommand{\com}[1]{}
	
\fi

\newcommand{\code}{\url{https://github.com/ FlorianKrach/PD-NJODE}}

\allowdisplaybreaks

\title{Nonparametric Filtering, Estimation and Classification using Neural Jump ODEs}

\if\preprint1
	\author{%
        \name Jakob Heiss \email{jakob.heiss@berkeley.edu} \\
        \addr Department of Mathematics\\
        ETH Zurich\\
        \addr Department of Statistics\\
        UC Berkeley
        \AND
        \name Florian Krach \email{florian.krach@math.ethz.ch} \\
        \addr Department of Mathematics\\
        ETH Zurich
        \AND 
         \name Thorsten Schmidt \email{thorsten.schmidt@stochastik.uni-freiburg.de}\\
        \addr Department of Mathematics\\
        Albert-Ludwigs-Universität Freiburg
        \AND
        \name Félix B. Tambe-Ndonfack \email{felix.ndonfack@stochastik.uni-freiburg.de}\\
        \addr Department of Mathematics\\
        Albert-Ludwigs-Universität Freiburg
	}

    \providecommand{\keywords}[1]{\textbf{{Keywords:}} \textit{#1}}
\else
        \articletype{Research Article}
	\received{Month	DD, YYYY}
	\revised{Month	DD, YYYY}
        \accepted{Month	DD, YYYY}
        \journalname{Statistics \& Risk Modeling}
        \journalyear{2025}
        \journalvolume{XX}
        \journalissue{X}
        \startpage{1}
        \DOI{10.1515/sample-YYYY-XXXX}
    
    \runningtitle{Neural Jump ODEs}
    
    \author[1]{Jakob Heiss}
    \author*[2]{Florian Krach}
    \author[3]{Thorsten Schmidt} 
    \author[3]{Félix B. Tambe-Ndonfack} 
    \runningauthor{Heiss, Krach, Schmidt, Tambe-Ndonfack} 
    \affil[1]{\protect\raggedright 
    ETH Zurich, Department of Mathematics, Zurich, Switzerland and University of California, Department of Statistics, Berkeley, USA, e-mail: jakob.heiss@berkeley.edu}
    \affil[2]{\protect\raggedright 
    ETH Zurich, Department of Mathematics, Zurich, Switzerland, e-mail: florian.krach@math.ethz.ch}
    \affil[3]{\protect\raggedright 
    Albert-Ludwigs-Universität Freiburg, Department of Mathematics, Freiburg, Germany, e-mail: thorsten.schmidt@stochastik.uni-freiburg.de, felix.ndonfack@stochastik.uni-freiburg.de}
    	
    \abstract{%
    Neural Jump ODEs model the conditional expectation between observations by neural ODEs and jump at arrival of new observations. They have demonstrated effectiveness for fully data-driven online forecasting in settings with irregular and partial observations, operating under weak regularity assumptions \citep{herrera2020neural,krach2022optimal}. This work extends the framework to input-output systems,  enabling direct applications in online filtering and classification. We establish theoretical convergence guarantees for this approach, providing a robust solution to $L^2$-optimal filtering. Empirical experiments highlight the model's superior performance over classical parametric methods, particularly in scenarios with complex underlying distributions. These results emphasize the approach's potential in time-sensitive domains such as finance and health monitoring, where real-time accuracy is crucial.%
    }
    
\fi

\begin{document}

\maketitle

\if\preprint1
    \begin{abstract}
        Neural Jump ODEs model the conditional expectation between observations by neural ODEs and jump at arrival of new observations. They have demonstrated effectiveness for fully data-driven online forecasting in settings with irregular and partial observations, operating under weak regularity assumptions \citep{herrera2020neural,krach2022optimal}. This work extends the framework to input-output systems,  enabling direct applications in online filtering and classification. We establish theoretical convergence guarantees for this approach, providing a robust solution to $L^2$-optimal filtering. Empirical experiments highlight the model's superior performance over classical parametric methods, particularly in scenarios with complex underlying distributions. These results emphasize the approach's potential in time-sensitive domains such as finance and health monitoring, where real-time accuracy is crucial.
    \end{abstract}
\fi

\section{Introduction}\label{sec:Introduction}

Time series analysis is important in many fields, such as meteorology, climate science, economics, finance, health care, medicine and more. The probably most important, but at the same time also most challenging, tasks are forecasting (or estimation), filtering, and classification of time series data, in particular when done \emph{online}, i.e., incorporating new observations directly upon their arrival. 
Classically, parametric approaches were used for these tasks, where certain models with a few interpretable parameters were assumed such that the analysis could be done analytically. 
In contrast to this, fully data-driven machine learning (ML) approaches have recently become more important due to their excellent empirical performance. 
Those approaches using neural networks can be regarded as nonparametric, when the size of the neural network is a priori unbounded.
They are highly flexible and offer theoretical guarantees through their universal approximation property. 

Recurrent neural networks (RNNs) \citep{rumelhart1985learning, jordan1997serial} are the most prominent approaches for time series. However, they are limited to regular, equidistant observations without missing values. In this work we focus on 
neural jump ODEs (NJODEs), introduced in \citet{herrera2020neural} and refined in \citet{krach2022optimal,andersson2024extending,krach2024learning}. This approach is highly suited for time series estimation and is able to work with irregularly and incompletely observed data. It employs path-dependence in a natural and explainable way through truncated signatures, as explained in detail in Remark \ref{rem:signature}.

Although adaptations for filtering were already proposed in \citet[Section~6]{krach2022optimal}, their main and natural application area is online \emph{forecasting}, where discrete and potentially incomplete past observations of a continuous-time process $X$  are used to optimally \emph{estimate}\footnote{Here we have a situation in mind where the process $X$ might only be partially observed and therefore use the word estimation. When $X$ is fully observed in the future, one would typically refer to forecasting $X$. However, we use both words interchangeably throughout the paper.} future values of this process $X$. 
In particular, the assumption for NJODEs has so far been that the \emph{output process} $V$ (i.e., the process to be estimated) coincides with the \emph{input process} $U$  (i.e., the process which is (partially) observed and used as input to the model). In this work, we generalize the framework to online \emph{filtering} of input-output systems, called 
input-output NJODE (IO NJODE) in the following. 
%

\subsection{Applications}

Applications of the approach are classical stochastic filtering (with discrete observations), filtering with partial and irregular observations, and parameter filtering, in addition to forecasting. 
Moreover, the proposed input-output configuration is applicable to online \emph{classification}, where arbitrary features can be used as input $U$ to predict conditional class probabilities. 
Since it can process sequential, real-time data inputs and continually generates predictions based on the most up-to-date information, the framework can be applied efficiently in time-sensitive domains like financial market analysis or health monitoring systems, where input-output relationships are central to decision making. In particular, we would like to point out two specific application cases:

First, in medical applications, the (unobserved) health state of a patient evolves continuously over time, and typically, rare scheduled examinations provide information on the health state, leading to a typical small data problem, as for example studied in \cite{hackenberg2022deep, hackenberg2025investigating}.

Second, in economic models of credit risk like the Merton model \cite{merton1974pricing}, the firm value evolves dynamically over time but is often difficult to assess and therefore considered unobserved in many works like \cite{duffie2001term,frey2009pricing,frey2012pricing,schmidt2008structural} among many others, see also  \Cref{sec:Geometric Brownian Motion with Uncertain Parameters} for a more detailed discussion. The observation is gathered through expert opinions, like in ratings, or when quarterly reports are due. In particular, the latter can be considered pre-scheduled, which has motivated research as in \cite{gehmlich2018dynamic,fontana2018general}.

\subsection{Contribution}\label{sec:Contribution}
When input and output processes $U$ and $V$ coincide, the setting of \citet{krach2022optimal} is recovered. Otherwise, the IO NJODE proposed in this work constitutes a targeted approach for the filtering task.
We provide a thorough theoretical framework for the general input-output setting and give minimal assumptions (\Cref{sec:Problem Setting}) to prove convergence (\Cref{sec:Convergence Guarantees}) of our IO NJODE model (\Cref{sec:proposed method}) to the conditional expectation, which is the $L^2$-optimal solution to the filtering and estimation problem. 
To establish these theoretical guarantees, we need to introduce a new objective function adapted to the input-output setting. In \Cref{sec:Implications of the choice of Objective Function}, we compare the original and this new objective function and discuss which one should be used depending on the task at hand.
Several examples for which our model can be applied are given in \Cref{sec:Examples}. Besides classical stochastic filtering, we put a special focus on parameter filtering, where the parameters $\alpha$ of a stochastic process $X^{\alpha}$ should be estimated from observations of $X^{\alpha}$. Moreover, we discuss that online classification tasks fall into our framework, by choosing the indicators for each class label as output process $U$. 
Several experiments are provided in \Cref{sec:Experiments}.
For clarity and ease of reading and comprehension, we build on the setting in \citet{krach2022optimal}, but note that the extensions to a dependent observation framework, noisy observations, and long-term predictions can equivalently be incorporated as described in \citet{andersson2024extending,krach2024learning} (cf.~\Cref{rem:cond indep,rem:noise adapted loss,rem:learning long-term predictions}).

\subsection{Related Work}\label{sec:Related Work}
As mentioned before, this work is based on the Neural Jump ODE literature that started with \citet{herrera2020neural} and was lifted to a new level of generality in \citet{krach2022optimal}.
The extensions of \citet{andersson2024extending,krach2024learning} are conceptually parallel developments to this work, both building on \citet{herrera2020neural}. The ideas of these 3 works should be understood as building blocks that can be stacked together as desired. This concept is realized in our implementation (available at \code). For more related work on the background of Neural Jump ODEs, we refer the interested reader to \citet{herrera2020neural,krach2022optimal,andersson2024extending}.

In this work, we use particle filters \citep{djuric2003particle} as baseline method to approximate the true conditional expectation when no analytic expression is available for it. 
In contrast to our model which is fully data-driven, classical particle filters need full knowledge of the underlying distributions to be applicable. Deep learning extensions of particles filters \citep{maddison2017filtering,le2017auto,corenflos2021differentiable,lai2022variational} provide more flexible approaches where the distributions can be learned; however, they are usually restricted to discrete state space models with regular and complete observations, while our method naturally deals with random (and therefore irregular) observation times and missing values. Deep learning extensions of the Kalman filter, like KalmanNet \citep{revach2022kalmannet} or Deep Kalman Filters \citep{krishnan2015deep}, as well as related works \citep{karl2016deep,naesseth2018variational,archer2015black,krishnan2017structured} have similar limitations as deep particle filters. Hence, without adjustments these methods cannot be used as baselines for our model in general settings.

\section{Problem Setting}\label{sec:Problem Setting}
In the following we provide a rigorous theoretical framework, which is, for simplicity, based on a combination of \citet[Section~2]{krach2022optimal} and \citet[Section~4.2]{andersson2024extending}. Including extensions for a dependent observation framework, noisy observations, and long-term predictions as proposed in \citet[Sections~3 and~4]{andersson2024extending} and \citet[Section~2]{krach2024learning} is straight forward, as outlined in \Cref{rem:cond indep,rem:noise adapted loss,rem:learning long-term predictions}.

\subsection{Stochastic Process, Random Observation Times and Observation Mask}
\label{sec:Stochastic Process, Random Observation Times and Observation Mask}
Let $d = d_U + d_V \in \N$ be the dimension and $T>0$ be a fixed time horizon. We work on a filtered probability space $(\Omega, \mathcal{F}, \mathbb{F},\mathbb{P})$, where the filtration $\mathbb{F}=(\mathcal{F}_t)_{t \in \mathbb{R}_+}$ satisfies the usual conditions, i.e.\ , the $\sigma$-field $\mathcal{F}$ is $\mathbb{P}$-complete, $\mathbb{F}$ is right-continuous and $\mathcal{F}_0$ contains all $\mathbb{P}$-null sets of $\mathcal{F}$.  On this filtered probability space, we consider  an adapted, $d$-dimensional, c\`adl\`ag  stochastic process $Z :={(Z_t)}_{t \in [0,T]}$.
We split  $Z =(U,V)$ into the $d_U$-dimensional input process $U$ and the $d_V$-dimensional  output process $V$. The running maximum process of $Z$ is defined as
$$Z^\star_t:= \sup_{0 \leq s \leq t} |Z_s|_1= \sup_{0 \leq s \leq t} (|U_s|_1+|V_s|_1), \quad 0 \leq t \leq T.$$
We denote the set of random times where the output process jumps by  $\mathcal{J}=\{(\omega,t) \colon  \Delta V_t(\omega) \neq 0 \}$, with  $\Delta V_t := V_t - V_{t-}$. 

In addition to the input-output process, we consider an exogenous observation \emph{masking}, which captures if a  coordinate of $Z$ is masked and therefore excluded from the data set. We assume that a random number of $n\in \N$ observations take place at the random times\footnote{In particular $n$ and $t_i$ are a random variables, even though we denote them with lowercase letters.}
$$0=t_0 < t_1 < \dotsb < t_{n} \le  T $$
and denote by $\bar n=\sup \left\{k \in \N \, | \, \P(k = n) > 0 \right\} \in \N \cup\{\infty\}$ the maximal value of $n$.
Note that this setup is extremely flexible and allows, in particular, a possibly unbounded number of observations in the finite time interval $[0,T]$.
Moreover, let 
$$ \tau(t):=  \max\{ t_k \colon  t_k \leq t \} $$
denote the last observation time (or zero if no observation was made yet) before time $t \in [0,T]$.
Finally, the observation mask is a sequence of random variables $M = (M_k)_{0 \leq k \leq \bar n}$ taking values in  $\{ 0,1 \}^{d}$. If $M_{k}^j=1$, then  the $j$-th coordinate $Z_{t_k}^j$ is observed at observation time $t_k$. By abuse of notation, we  also write $M_{t_k} := M_{k}$. 
For notational simplicity, we will  set $t_k(\omega)=\infty$ for $k > n(\omega)$ and for all appearing processes their value at $t_k=\infty$ to $0$.

\subsection{Information \texorpdfstring{$\sigma$}{sigma}-algebra}\label{sec:information sigma-algebra}
The information available at time $t$ is given by the values of the input process $U$ at the observation times when not masked. This leads 
to the filtration of the currently available information $\mathbb{A} := (\mathcal{A}_t)_{t \in [0, T]}$ given by
\begin{equation*}
\A_t := \boldsymbol{\sigma}\left(U_{t_i, j}, t_i, M_{t_i} \mid t_i \leq t,\, j \in \{1 \leq l \leq d_U \mid M_{t_i, l} = 1 \} \right),
\end{equation*}
where $\boldsymbol{\sigma}(\cdot)$ represents the generated $\sigma$-algebra.
By the definition of $\tau$, we have $\mathcal{A}_t = \mathcal{A}_{\tau(t)}$ for all $t \in [0, T]$. 
Additionally,  for any fixed observation (or stopping) time $t_k$, the stopped and pre-stopped\footnote{The stopped $\sigma$-algebra \citep[Definition 2.37]{KarandikarRao2018} is defined as $\mathcal{F}_{\tau} = \{A \in \sigma(\cup_{t} \mathcal{F}_{t}) : A \cap \{\tau \leq t\} \in \mathcal{F}_{t} \, \forall t\}$, where $\tau$ is the stopping time. The pre-stopped $\sigma$-algebra \citep[Definition 8.1]{KarandikarRao2018} is defined as $\mathcal{F}_{\tau-} = \sigma\left(\mathcal{F}_0 \cup \{A \cap \{t < \tau\} : A \in \mathcal{F}_t, \, t < \infty\}\right)$, where $\tau$ is the stopping time.} $\sigma$-algebras at $t_k$ are
\begin{equation*}
\begin{split}
\mathcal{A}_{t_k} &:= \boldsymbol{\sigma}\left(U_{t_i, j}, t_i, M_{t_i} \,\middle|\, i\leq k,\, j \in \{1 \leq l \leq d_U | M_{t_i, l} = 1  \} \right), \\
\mathcal{A}_{t_k-} &:= \boldsymbol{\sigma}\left(U_{t_i, j}, t_i, M_{t_i}, t_k \,\middle|\, i < k,\, j \in \{1 \leq l \leq d_U | M_{t_i, l} = 1  \} \right).
\end{split}
\end{equation*} 

\subsection{Notation and Assumptions on the Stochastic Process \texorpdfstring{$Z$}{Z}}\label{sec:Notation and assumptions}
Since we are working with an input-output system, we are interested in the conditional expectation of the output process $V$, given the observation of the input process (collected in the information $\sigma$-algebra), denoted as $\hat{V}=(\hat V_t)_{0 \leq t \leq T}$ and defined via $$\hat{V}_t := \E[V_t | \A_t].$$ 
We define the interpolation of the observations of $\bar U:=(U_t,\tilde n_t, t)_{t \in [0,T]}$ made until time $t$, where $\tilde n_{t} := \left( \sum_k M_{k,j} \ind{t_k \leq t} \right)_{1\leq j \leq d_U}$ counts the coordinate-wise observations,  by $\tilde U^{\leq t}$ for any $0 \leq t \leq T$. At time $0 \leq s \leq T$, its $j$-th coordinate is given by
\begin{equation*}
\tilde U^{\leq t}_{s,j} := 
\begin{cases}
	\bar U_{t_{a(s,t,j)},j} \frac{t_{b(s,t,j)} - s}{t_{b(s,t,j)} - t_{b(s,t,j)-1}} + \bar U_{t_{b(s,t,j)},j} \frac{s - t_{b(s,t,j)-1}}{t_{b(s,t,j)} - t_{b(s,t,j)-1}}, & \text{if }  t_{b(s,t,j)-1} < s \leq t_{b(s,t,j)}, \\
	\bar U_{t_{a(s,t,j)},j}, & \text{if } s \leq t_{b(s,t,j)-1} , 
\end{cases}
\end{equation*}
where 
\begin{equation*}
\begin{split}
a(s,t, j) &:=  \max\{ 0 \leq a \leq n \vert t_a \leq \min(s,t), M_{t_a,j}  = 1 \}, \\
b(s,t,j) &:=  \inf\{ 1 \leq b \leq n \vert s \leq t_b \leq t, M_{t_b,j}  = 1 \},
\end{split}
\end{equation*}
letting $t_\infty := T$ and $\inf \emptyset := \infty$. Then,
$\tilde{U}^{\leq t}$ is a continuous version (without information leakage and with time-consistency) of the rectilinear (``forward fill'') interpolation of the observations.
According to this definition, $\tilde{U}^{\leq \tau(t)}$ is measurable with respect to $\mathcal{A}_{\tau(t)}$ for any $t$, and for any $r \geq t$ and $s \leq \tau(t)$, we have $\tilde{U}^{\leq \tau(t)}_s = \tilde{U}^{\leq \tau(r)}_s$.
As the paths of $\tilde{U}^{\leq \tau(t)}$ are piecewise linear, they belong to $BV^c([0,T])$, the set of continuous $\mathbb{R}^{d_U}$-valued paths of bounded variation on $[0,T]$ (cf.~\Cref{def:BV}).

We note that the Doob-Dynkin Lemma \citep[Lemma~1.14]{kallenberg2021foundations} implies the existence of measurable functions $F_j : [0,T] \times [0,T] \times BV^c([0,T]) \to \R$ such that $\hat V_{t,j} = F_j(t, \tau(t), \tilde U^{\leq \tau(t)}) $.

For the remainder of this article, we will assume that the following set of assumptions holds throughout. Those are a maximally weakened version of the assumptions made in \citet[Section~2.3]{krach2022optimal}. 

\begin{assumption} \label{assumption:1}
We assume that the input process  $U$ is fully observed at $t_0=0$.  Moreover, we assume that for every $1\leq k, l \leq \bar n$, $M_k$ is independent of  $t_l$ and of $n$ and $\P (M_{k,j} =1 ) > 0$ for all  $d_U+1 \leq j \leq d$ on the training set (during training, every coordinate of the output process $V$ can be observed with positive probability at any observation time).
\end{assumption}

\begin{assumption} \label{assumption:3}
Almost surely $V$ is  not observed at a jump, i.e., $\P( \{ t_j \in \mathcal{J} | j\leq n \} ) = \P( \{ \Delta V_{t_j} \neq 0 | j\leq n \} ) = 0$.
\end{assumption} 

\begin{assumption}\label{assumption:4}
$F_j$ are  continuous and differentiable in their first coordinate $t$ such that their partial derivatives with respect to $t$, denoted by $f_j$, are again continuous and the following integrability condition holds
\begin{small}
\begin{equation}\label{equ:assumption4 bound}
\E\left[ \frac{1}{n} \sum_{i=1}^n \left(|F_j(t_i, t_i, \tilde U^{\leq t_i})|^2 + |F_j(t_{i-1}, t_{i-1}, \tilde U^{\leq t_{i-1}})|^2 \right) +  \int_0^T | f_j(t, \tau(t), \tilde U^{\leq \tau(t)})  |^2 dt  \right] < \infty.
\end{equation}
\end{small}
\end{assumption} 
\begin{assumption}\label{assumption:5}
$U^\star$ is $L^{2}$-integrable, i.e., $\E[(U^\star_T)^2] < \infty$, and $\E\Big[\frac{1}{n}\displaystyle \sum_{i=1}^n V_{t_i}^2 \Big] < \infty$.
\end{assumption} 

\begin{assumption} \label{assumption:6}
The random number of observations $n$ is integrable, i.e., $\E[n] < \infty$.
\end{assumption} 

\begin{assumption}\label{ass:independence}
    The process $Z = (U,V)$ is independent of the observation framework consisting of the random variables $n, (t_i, M_i)_{i \in \N}$.
\end{assumption}

\begin{rem} \label{assumption:2}
The probability that any two observation times are closer than $\epsilon>0$ converges to $0$ when $\epsilon$ does, i.e., if $\delta_{\min}( \omega) := \min_{0 \leq i \leq n( \omega)} |t_{i+1}( \omega) - t_i( \omega)|$ then $\lim_{\epsilon \to 0} \P (\delta_{\min} < \epsilon) = 0$. This is equivalent to $\P(\delta_{\min}=0)=0$, which is satisfied when the observation times are strictly increasing.
\end{rem}
\begin{rem}\label{rem:ass1}
    \Cref{assumption:1} can easily be generalized to settings where $U_0$ is only partly or not observed \citep[see][Remark~2.2]{krach2022optimal}. Furthermore, the assumption that all coordinates can be observed at any observation time can be weakened as shown in \Cref{sec:Different observation times for input and output variables}.
\end{rem}

\begin{rem}\label{rem:cond indep}
    \Cref{ass:independence} puts us in the same setting as was hard-coded through the product probability space in \citet[Section~2.2]{krach2022optimal}. Moreover, the independence \Cref{assumption:1,ass:independence} can easily be replaced by conditional independence assumptions as formulated in \citet[Section~4]{andersson2024extending}.
\end{rem}

As shown in \citet[Proposition~2.5]{krach2022optimal} and also implied by \Cref{lem:L2 identity}, the conditional expectation process $\hat{V}$ is the optimal process in $L^2(\Omega, \mathbb{A} , \P)$ approximating  $(V_t)_{t \in [0,T]}$ and it is unique up to $\P$-null-sets. 

Throughout this paper, we will use the following distance functions (based on the observation times) between processes and define indistinguishability with them. This is a slight modification for the present setting from the corresponding definition in \citet[Equation~5 and Remark~2.8]{krach2022optimal}. 
The proposed pseudo metric controls on the one side the prediction before the jump at $t_k$ and on the other side the realized jump size.

\begin{definition}\label{def:indistinguishability}
Fix $r \in \N$ and set $c_0(k) := (\P (n \geq k))^{-1}$.
The family of (pseudo) metrics $d_k$, $1 \le k \le \bar n$,  for  two c\`adl\`ag $\mathbb{A}$-adapted processes $\eta, \xi  : [0,T] \times \Omega \to \R^{r}$  is defined as 
\begin{equation}\label{equ: pseudo metric}
    d_k (\eta, \xi) = c_0(k)\,  \E\left[ \1_{\{ k \le n\}} \left( | \eta_{t_k-} - \xi_{t_k-} |_1 + | \eta_{t_k} - \xi_{t_k} |_1 \right) \right].
\end{equation}
We call the processes \emph{indistinguishable at observation points}, if $d_k(\eta, \xi)=0$ for every $1 \leq k \leq \bar n$.
\end{definition}

    The motivation for this distance is that we can only approximate the target process $V$ at times, where we can potentially observe it. In particular, if $V$ is never observed within a certain time interval, i.e., all observation times have probability $0$ to lie within this interval, then it is impossible to learn the dynamics of $V$ within this time interval.
    It is important to note that observations of the target process $V$ (and hence \Cref{assumption:1}) are only needed for the training of the model, while inference works solely based on the observations of the input process $U$. Furthermore it seems important to stress that, since observation times are random and expectations are taken, \emph{all possible} observation times are taken into account in the above definition when indistinguishability is considered \citep[for more details see][Section~5.2]{andersson2024extending}.

\section{The Neural Jump ODE Model for Input-Output Processes}\label{sec:proposed method}
We want to model the relationship between an input process $U$ and an output process $V$, which may have different dimensions and exhibit complex dynamics. More precisely, we want to approximate the functions $F_j$, which map the observations of $U$ to the optimal prediction of $V$. 
To achieve this, we define the Input-Output Neural Jump ODE (IO NJODE) model, which leverages neural networks to capture the underlying structure of these processes. In the following, we outline the definition of the IO NJODE model, incorporating the necessary modifications to the original framework defined in \citet[Section~3.3]{krach2022optimal}.

One key component of the model is the truncated signature.  The intuition to rely on  signatures is the following: since a path of $X$ can be reconstructed from its signature, any statistic being a function of the path, can equivalently be represented as a function of the signature. The truncated signature contains the typically most relevant information up to the order $m$ and therefore serves as a tractable approximation of path information. More precisely, let $I$ denote a compact interval in $\R$ and ${X}:I\rightarrow\R^{\delta}$ be a continuous path with finite variation.
The truncated signature transforms this path ${X}$ into a finite-dimensional feature vector $\pi_m({X})$, i.e., the first $m+1$ terms (levels) of the signature of ${X}$ (see \Cref{sec:Signatures} for more details).
A well-known result in stochastic analysis states that continuous functions of finite variation paths can be uniformly approximated using truncated signatures. 
We will require a generalization of this result that allows for an additional finite-dimensional input, which was proven in \citet[Proposition 3.8]{krach2022optimal}.
For clarity, we provide the result and details in the Appendix in \Cref{prop:universal_approx_sig}.

\begin{rem}[Path dependence and explainability]
\label{rem:signature}
Incorporating path dependence via signatures can be viewed as a tractable extension from the Markov property to path dependence. Moreover, signatures decode the statistical properties of the observed paths in a very efficient and direct way - in particular the \emph{truncated signature} focuses on the first $m$ signature components, which are typically the most relevant ones. Since the employed neural networks will not be very deep, the resulting model is quite explicit and therefore contributes to the explainability of the approach.
\end{rem}

As a reviewer pointed out, there are different possibilities to use the signature for our purpose, as we explain in \Cref{rem:jumpyInterpolations,rem:logSignature}.

\begin{rem}\label{rem:jumpyInterpolations}
    Recent research \citep{cuchiero_universal_2025} allows to apply signatures directly to paths with jumps. Hence, we could use the standard rectilinear interpolation instead of our piecewise-linear construction. However, these results come at the price of more involved metrics.
\end{rem}

\begin{rem}\label{rem:logSignature}
    The log-signature \citep[Section~1.3.5]{Chevyrev2016APO} could be used instead of the full signature, since it is integrated in a neural network allowing for nonlinear combinations of the terms (signatures themselves have the universal approximation property with linear combinations, while log-signatures are universal only with nonlinear combinations). This could potentially improve the efficiency and therefore the performance of the model, however, a detailed analysis is left for future research.
\end{rem}

We introduce, following  \citet[Definition~3.12]{krach2022optimal},  a special class of neural networks with bounded outputs derived from any set of neural networks.
\begin{definition}\label{def:bounded output NN}
Define for dimension $\delta \in \N$  the Lipschitz continuous and bounded output activation function $\Gamma_\gamma$ with trainable parameter $\gamma > 0$ by
\begin{equation*}
\Gamma_\gamma : \R^\delta \to \R^\delta, x \mapsto x \cdot \min\left(1, \frac{\gamma}{|x|_2}\right)
\end{equation*}
and the class of \emph{bounded output neural networks} by
\begin{equation*}
\mathcal{N} := \{  f_{(\tilde \theta, \gamma)} := \Gamma_\gamma \circ \tilde f_{\tilde \theta} \, | \, \gamma > 0, \tilde f_{\tilde \theta} \in \tilde{\mathcal{N}}  \},
\end{equation*}
where $\tilde{\mathcal{N}}$ can be any set of neural networks. We use the notation $\tilde f_{\tilde \theta} \in \tilde{\mathcal{N}}$ and $ f_\theta \in \mathcal{N}$ for $\theta=(\tilde \theta, \gamma)$ to highlight the trainable weights $ \tilde \theta$ (and $\gamma$) of the respective (bounded output) neural networks.
\end{definition} 
The key property of the truncation function $\Gamma_\gamma$ is that $\Gamma_\gamma(x)=x$ for all $x$ with $|x|_2 \le \gamma$.

In the following, we consider $\tilde{\mathcal{N}}$ to be the class of feed-forward neural networks that use Lipschitz continuous activation functions. This implies that for any natural numbers $d_1, d_2$ and for any compact set $\mathcal{X} \subset \R^{d_1}$, the set $\tilde{\mathcal{N}}$ is a dense subset of the set of continuous functions $C(\mathcal{X} , \R^{d_2})$ in terms of the supremum norm. Moreover, we assume that all affine functions and, in particular, the identity are in $\tilde{\mathcal{N}}$. Recall that  $\pi_m$ is the truncated signature of order $m$. 

To the (random) observation times $t_1,\dots,t_n$ we associate the pure-jump  process
\begin{align*} 
J_t:=\sum_{i=1}^{n} 1_{\left[t_i, \infty\right)}(t), \quad 0 \le t \le T,
\end{align*}
which simply counts the observations until the current time $t$. Very useful is that the increment $dJ_t$  either vanishes or equals $\Delta J_t = 1$ (since we assumed that the observation times are \emph{strictly} increasing), which guarantees a high tractability of the following model class. 

\begin{definition}[IO NJODE]\label{def:PD-NJODE}
    The  \textit{Input-Output Neural Jump ODE} model is given by
    \begin{equation}\label{equ:PD-NJODE}
\begin{split}
H_0 &= \rho_{\theta_2}\left(0, 0, \pi_m (0), U_0 \right), \\
dH_t &= f_{\theta_1}\left(H_{t-}, t, \tau(t), \pi_m (\tilde U^{\leq \tau(t)} -U_0 ), U_0,\tilde U^{\star}_t, n_t,\delta_t \right) dt  \\
& \quad + \left( \rho_{\theta_2}\left( H_{t-}, t, \pi_m (\tilde U^{\leq \tau(t)}-U_0 ), U_0,\tilde U^{\star}_t, n_t,\delta_t \right) - H_{t-} \right) dJ_t, \\
G_t &= \tilde g_{\tilde \theta_3}(H_t).
\end{split}
\end{equation}
$f_{\theta_1},~\rho_{\theta_2} \in \mathcal{N}$ are bounded output feedforward neural networks and $\tilde g_{\tilde \theta_3} \in \tilde{\mathcal{N}}$ is a feedforward neural network. They have trainable weights $\theta = (\theta_1, \theta_2, \tilde \theta_3) \in \Theta$, where $\theta_i = (\tilde \theta_i, \gamma_i)$ for $i \in \{1,2 \}$ and the set of all potential weights for the NJODE is denoted by $\Theta$. 
\end{definition}

While $H=H^{\theta_1,\theta_2}$ governs the dynamic evolution of the model, the final output $G$ is obtained after the application of a (standard) feedforward neural network to $H$, $G = \tilde g_{\tilde \theta_3}(H^{\theta_1,\theta_2})$. To emphasize the dependence of the model output on the trainable parameters and its input, we will also  write $G^{\theta }(\tilde{U}^{\leq \tau(\cdot)})$.

\subsection{Objective Functions}
Let $\mathbb{D}$ represent the collection of all c\`adl\`ag processes that take values in $\R^{d_V}$ and are adapted to the filtration $\mathbb{A}$.
Define the objective functions $\Psi: \, \mathbb{D} \to \R$ and $L :   \Theta   \to \R$  as
\begin{align}
\Psi(\eta) &:= \E\left[ \frac{1}{n} \sum_{i=1}^n    \left\lvert \proj{V} (M_i) \odot ( V_{t_i} - \eta_{t_i} ) \right\rvert_2^2 + \left\lvert \proj{V} (M_i) \odot (V_{t_i} - \eta_{t_{i}-} ) \right\rvert_2^2  \right], \label{equ:Psi} \\
 L(\theta) &:= \Psi(G^{\theta}(\tilde{U}^{\leq \tau(\cdot)})), \label{equ:Phi}
\end{align}
where $\odot$ is the element-wise (Hadamard) product, $\proj{V}$ denotes the projection onto the coordinates corresponding to the output variable $V$. We call $L$  the \emph{theoretical} objective function.

We note that the objective functions differ from \citet[Equations~11 and~12]{krach2022optimal}, since here we square both terms separately. This is indeed a necessary adjustment to establish the same theoretical results  in the more general setting of input-output systems. In \Cref{sec:Implications of the choice of Objective Function}, we  discuss the different implications of the two definitions in more detail. 

\begin{rem}\label{rem:noise adapted loss}
    To allow for observation noise in the setting of input-output systems \citep[as in][Section~3]{andersson2024extending} we have to distinguish two cases. If there is only observation noise on the input process $U$ but not on $V$, the framework can be applied equivalently without any changes (as long as $U$ still satisfies all assumptions). Indeed, the model will learn to filter the information from $U$ to optimally predict $V$.
    However, if the target process $V$ has noisy observations, then \eqref{equ:Psi} would lead to learning and overfitting to the noise instead of filtering it. In this case, we need to adapt \eqref{equ:Psi} as in \citet[Section~3]{andersson2024extending}, as
    \begin{equation*}
    \Psi_{\text{noisy}}(\eta) := \E\left[ \frac{1}{n} \sum_{i=1}^n \left\lvert \proj{V} (M_i) \odot (O_{t_i} - \eta_{t_{i}-} ) \right\rvert_2^2  \right],
    \end{equation*}
    where $O_{t_i} = V_{t_i} + \epsilon_i$ are the noisy observations of the target process $V$. Indeed, this is consistent with the input-output setting upon using the weaker (pseudo) metrics of \citet[Definition~2.4]{andersson2024extending}. However, \eqref{equ:Psi} should be used in the case of noise-free observations of $V$, due to the higher ``inductive strength'' \citep[see][Section~0.5.1]{KrachPhDThesis}.
\end{rem}

For the definition of the \emph{empirical} objective function, we assume to have $N$ independent realizations of the process $Z$, of the observation mask $M$ and of observation times:
 let $Z^{(j)} \sim Z$, $M^{(j)} \sim M$ and $(n^{(j)}, t_1^{(j)}, \dotsc, t_{n^{(j)}}^{(j)}) \sim ( n, t_1, \dotsc, t_{n})$ be i.i.d random variables or processes, respectively. The training data is a single realization of these. We denote $G^{\theta, j} := G^{\theta }(\tilde{U}^{\leq \tau(\cdot), (j)})$. The empirical objective function $\hat L_N$ equals the Monte Carlo approximation of the loss function with those $N$ samples, given by
\begin{equation}\label{equ:appr loss function}
\hat L_N(\theta) := \frac{1}{N} \sum_{j=1}^N  \frac{1}{n^{(j)}}\sum_{i=1}^{n^{(j)}}   \left\lvert \proj{V}(M_{i}^{(j)}) \odot \left( V_{t_i^{(j)}}^{(j)} - G_{t_i^{(j)}}^{\theta, j } \right) \right\rvert_2^2   
 + \left\lvert \proj{V}(M_{i}^{(j)}) \odot \left( V_{t_i^{(j)}}^{(j)} - G_{t_{i}^{(j)}-}^{\theta, j } \right) \right\rvert_2^2 ,
\end{equation}
and converges $\P$-a.s. to $L(\theta)$ as $N \to \infty$, according to the law of large numbers (see \Cref{thm:MC convergence Yt}).

\section{Convergence Guarantees}\label{sec:Convergence Guarantees}
In this section, we present the main results establishing convergence of the IO NJODE model to the optimal prediction. We begin with the convergence result for the theoretical objective function $L$, followed by demonstrating convergence when using its Monte Carlo approximation $\hat{L}$, i.e., the implementable loss function. In particular, using either $L$ or $\hat L$, we show that the model output $G^\theta$ converges to the true conditional expectation $\hat V$ in the metrics $d_k$. Hence, in the limit, the model output is indistinguishable at observations from $\hat V$. Importantly, the convergence holds when evaluating the trained model on the training set or on an independent test set (see the proof for more details).

\begin{rem}\label{rem:learning long-term predictions}
    We prove convergence to the conditional expectation $\hat V_t = \E[V_t | \mathcal{A}_t]$, which is the optimal prediction of $V$ at time $t$ given \emph{all} information up to time $t$. Using the generalized training procedure of \citet[Section~2]{krach2024learning}, our model can learn the long-term predictions $\hat V_{t,s} = \E[V_t | \mathcal{A}_{s \wedge t}]$ of $V$ at time $t$ given the information only up to $\min(s,t)$, for any $s,t \in [0,T]$. Indeed, this enhanced training procedure applies equivalently in the case of input-output systems.
\end{rem}

For the proof of the convergence results, we follow \citet[Section~4]{krach2022optimal} and extend their results to our settings. 

\subsection{Convergence of the Theoretical Objective Function \texorpdfstring{$L$}{L}}

Initially, we recall the easy results of \citet[Lemma~4.2 and Lemma 4.3]{krach2022optimal} in the present setting, with a small extension which can be proven identically. The first lemma resides on the independence of $(n,M,t_1,\dots,t_n)$ and $V$ (\Cref{ass:independence}) together with the classical property of $L^2$-optimality of the conditional expectation. 
\begin{lemma}\label{lem:L2 identity}
For an $\mathbb{A}$-adapted process $\eta$, it holds that
\begin{multline*}
\E\left[\dfrac{1}{n} \sum_{i=1}^n \left\lvert \proj{V}(M_{t_i}) \odot ( V_{t_i} - \eta_{t_i-} ) \right\rvert_2^2\right] \\
	= \E\left[ \dfrac{1}{n}\sum_{i=1}^n \left\lvert \proj{V}(M_{t_i}) \odot ( V_{t_i} - \hat{V}_{t_i-} ) \right\rvert_2^2\right] + \E\left[\dfrac{1}{n}\sum_{i=1}^n \left\lvert \proj{V}(M_{t_i}) \odot (  \hat{V}_{t_i-} - \eta_{t_i-}) \right\rvert_2^2\right],
\end{multline*}
and
\begin{multline*}
\E\left[\dfrac{1}{n} \sum_{i=1}^n \left\lvert \proj{V}(M_{t_i}) \odot ( V_{t_i} - \eta_{t_i} ) \right\rvert_2^2\right] \\
	= \E\left[ \dfrac{1}{n}\sum_{i=1}^n \left\lvert \proj{V}(M_{t_i}) \odot ( V_{t_i} - \hat{V}_{t_i} ) \right\rvert_2^2\right] + \E\left[\dfrac{1}{n}\sum_{i=1}^n \left\lvert \proj{V}(M_{t_i}) \odot (  \hat{V}_{t_i} - \eta_{t_i}) \right\rvert_2^2\right] .
\end{multline*}
\end{lemma}
Observe that an immediate consequence of this result is that $\hat V$ is $L^2$ optimal in the sense
\begin{align}\label{eq:L2 optimailty}
    \E\left[\dfrac{1}{n} \sum_{i=1}^n \left\lvert \proj{V}(M_{t_i}) \odot ( V_{t_i} - \hat V_{t_i} ) \right\rvert_2^2\right] &= 
    \min_{\eta \in \mathbb{D}}
    \E\left[\dfrac{1}{n} \sum_{i=1}^n \left\lvert \proj{V}(M_{t_i}) \odot ( V_{t_i} - \eta_{t_i} ) \right\rvert_2^2\right].
\end{align}

The next result 
derives a highly tractable, linear representation of certain piecewise constant functions of an additional, independent random variable.

\begin{lemma}\label{lem:expectation weighted sum over t_i terms}
Consider a stochastic process $\varphi=(\varphi_t)_{t \in T}$, bounded from below, an independent random variable $Y \sim \operatorname{Unif}([0,1])$ and let  $\bar t := \sum_{i=1}^n \1_{(\frac{i-1}{n}, \frac{i}{n}]}(Y) \, t_i$. Then, 
    \begin{equation*}
        \E [\varphi\left(\bar t\right)] = \E\Big[\sum_{i=1}^n \frac{1}{n} \varphi(t_i)\Big].
    \end{equation*}
    Similarly, if $\bar t := \sum_{i=1}^n \1_{(\frac{i-1}{2n}, \frac{i}{2n}]}(Y) \, t_i + \sum_{i=1}^n \1_{(\frac{n+i-1}{2n}, \frac{n+i}{2n}]}(Y) \, t_{i-1}$, 
    \begin{equation*}
        \E [\varphi\left(\bar t\right)] = \frac{1}{2}\bigg( \E\Big[\sum_{i=1}^n \frac{1}{n} \varphi(t_i)\Big] + \E\Big[\sum_{i=1}^n \frac{1}{n} \varphi(t_{i-1})\Big] \bigg).
    \end{equation*}
\end{lemma}
For intuition, we provide the short proof of this result, following \citet[Lemma~4.3]{krach2022optimal}.
\begin{proof}
Since $\varphi$ is bounded from below, the expectation is well-defined, although possibly infinite. Furthermore, by independence 
\begin{align*}
    \qquad\qquad\qquad \qquad\qquad \E [\varphi\left(\bar t\right)] &= 
    \E \bigg[ \int_0^1 \varphi\bigg(\sum_{i=1}^n \1_{(\frac{i-1}{n}, \frac{i}{n}]}(y) \, t_i \bigg) dy \bigg] \\
    &= \int_\Omega \int_0^1 \varphi\bigg(\sum_{i=1}^{n(\omega)} \1_{(\frac{i-1}{n(\omega)}, \frac{i}{n(\omega)}]}(y) \, t_i(\omega) \,,\, \omega\bigg) dy \, d\P \\
    &= \int_\Omega \sum_{i=1}^{n(\omega)} \int_{(\frac{i-1}{n(\omega)}, \frac{i}{n(\omega)}]} \varphi(t_i(\omega),\omega) \,dy \, d\P = \E\Big[\sum_{i=1}^n \frac{1}{n} \varphi(t_i)\Big].  \qquad \qquad\qquad\qquad\qedhere
\end{align*}
\end{proof}

Now, we can state the first main convergence result which shows that minima of  the theoretical loss function $L$ defined in Equation \eqref{equ:Phi} converge to the minimal value of $\Psi$ defined in Equation \eqref{equ:Psi}. The result is adapted from \citet[Theorem~4.1]{krach2022optimal}. We let $\Theta_m \subset \Theta$ be the set of combined weights for all 3 neural networks of the IO NJODE that are bounded in 2-norm by $m$ and correspond to neural networks of width and depth bounded by $m$. Overall, the model has the weights $\theta = (\theta_1, \theta_2, \tilde \theta_3) \in \Theta_m$ with $\theta_i = (\tilde{\theta}_i, \gamma_i)$ for $i \in \{ 1,2\}$.
We denote by $\Theta_m^i$ and $\tilde{\Theta}_m^i$ the projection of the set $\Theta_m$ on the weights $\theta_i$ and $\tilde{\theta}_i$, respectively.

\begin{theorem}\label{thm:1}
Assume that Assumptions~\ref{assumption:1} to \ref{assumption:6} hold and
let $\theta^{\min}_m \in \Theta_m^{\min} := \displaystyle \argmin _{\theta \in \Theta_m}\{ L(\theta) \}$ for every $m \in \N$.  
Then, for $m \to \infty$, the value of minima of the loss function $L(\theta_m^{\min})$  converge to the minimal value of $\Psi$, which is uniquely achieved by $\hat{V}$ up to indistinguishability on observation points, i.e.,
\begin{equation*}
L(\theta_m^{\min}) \xrightarrow{m \to \infty} \min_{\eta \in \mathbb{D}} \Psi(\eta) = \Psi(\hat{V}).
\end{equation*}
Furthermore, for every $1 \leq k \leq \bar n$, $G^{\theta_m^{\min}}$ converges to $\hat{V}$ in the pseudo metric $d_k$  as $m \to \infty$.
\end{theorem}

\begin{proof}
\textbf{Step 1:} We start by showing that the objective function $\Psi$ has the unique minimizer $\hat{V} \in \mathbb{D}$ up to indistinguishability. We first show that $\hat{V}$ is a minimizer. 
By  \Cref{assumption:3}, $
\P( \{ \Delta V_{t_i} \neq 0 | i\leq n \} )=0$, such that
$V_{t_{i^-}}= V_{t_i}$ with probability one. 
Then, \Cref{lem:L2 identity} together with  \Cref{{eq:L2 optimailty}} implies that
\begin{small}
\begin{align*}
    \Psi(\hat{V}) &= \E\bigg[ \frac{1}{n} \sum_{i=1}^n  \left(  \left\lvert \proj{V}(M_{t_i}) \odot ( V_{t_i} - \hat{V}_{t_i} ) \right\rvert_2^2 + \left\lvert \proj{V} (M_i) \odot (V_{t_i} - \hat{V}_{t_{i}-} ) \right\rvert_2^2 \right) \bigg]\\
    & =  \min_{\eta \in \mathbb{D}} \E\bigg[ \frac{1}{n} \sum_{i=1}^n  \left(  \left\lvert \proj{V} (M_i) \odot ( V_{t_i} - \eta_{t_{i}} ) \right\rvert_2^2 \right) \bigg] 
    + \min_{\eta \in \mathbb{D}} \E\bigg[ \frac{1}{n} \sum_{i=1}^n  \left(  \left\lvert \proj{V} (M_i) \odot ( V_{t_i} - \eta_{t_{i}-} ) \right\rvert_2^2 \right) \bigg]\\
    & \leq   \min_{\eta \in \mathbb{D}} \E\bigg[ \frac{1}{n} \sum_{i=1}^n  \left(  \left\lvert \proj{V} (M_i) \odot ( V_{t_i} - \eta_{t_{i}} ) \right\rvert_2^2  + \left\lvert\proj{V} (M_i) \odot ( V_{t_i} - \eta_{t_{i}-} ) \right\rvert_2^2 \right) \bigg]\\
    &= \min_{\eta \in \mathbb{D}} \Psi(\eta).
\end{align*}
\end{small}
Next, we show uniqueness at observation points. We start by deriving some needed results. Let $c_1 := \E[n]^{\frac{1}{2}} \in (0,\infty)$, then we use the Hölder inequality with the fact that $n(\omega) \geq 1$ $\P$-a.s.~to get
\begin{equation}\label{equ:HI}
    \E\left[ \left\lvert \xi \right\rvert_2 \right] 
	= \E\left[ \frac{\sqrt{n}}{\sqrt{n}} \left\lvert \xi \right\rvert_2 \right] 
	\leq c_1 \, \E\left[ \frac{1}{n} \left\lvert \xi \right\rvert_2^2 \right]^{1/2}.
\end{equation}
According to \Cref{assumption:1},  $c_2(k) := \min_{d_U +1 \leq j \leq d} \P (M_{k,j} = 1) > 0$. 
Then, we have for any $1 \leq k \leq \bar n$ by the independence of $\proj{V}(M_{k, j})$ from $t_k$, $n$ and $\mathcal{A}_{t_k-}$ that
\begin{small}
\begin{equation*}
\begin{split}
    \E &\left[ \1_{\{k \leq n\}} \bigg(\left\lvert \proj{V}(M_{t_k}) \odot ( \hat{V}_{t_k-} - \eta_{t_k-} ) \right\rvert_1 + \left\lvert \proj{V}(M_{t_k}) \odot ( \hat{V}_{t_k} - \eta_{t_k} ) \right\rvert_1 \bigg)\right] \\
    &= \E\bigg[ \1_{\{k \leq n\}} \sum_{j=d_U + 1}^{d} M_{k, j} \Big( \left\lvert  \hat{V}_{t_k-,j} - \eta_{t_k-,j}  \right\rvert +  \left\lvert  \hat{V}_{t_k, j} - \eta_{t_k, j}  \right\rvert \Big) \bigg] \\
    &=\sum_{j=d_U + 1}^{d} \E\left[ M_{k, j} \right] \, \E\left[ \1_{\{k \leq n\}}\bigg(  \left\lvert  \hat{V}_{t_k-,j} - \eta_{t_k-,j}  \right\rvert + \left\lvert  \hat{V}_{t_k, j} - \eta_{t_k, j}  \right\rvert \bigg) \right] \\
    &\geq  c_2 (k)\, \E\bigg[ \1_{\{k \leq n\}} \bigg(  \left\lvert  \hat{V}_{t_k-} - \eta_{t_k-}  \right\rvert_1 + \left\lvert  \hat{V}_{t_k} - \eta_{t_k}  \right\rvert_1 \bigg) \bigg] .
 \end{split}
\end{equation*}
\end{small}
By the equivalence between $1$-norm and $2$-norm, we  have for some constant $c_3 >0$
\begin{small}
\begin{equation}\label{equ:M split}
\begin{split}
 \E&\left[ \1_{\{k \leq n\}} \bigg(  \left\lvert  \hat{V}_{t_k-} - \eta_{t_k-}  \right\rvert_2 + \left\lvert  \hat{V}_{t_k} - \eta_{t_k}  \right\rvert_2 \bigg) \right] \\
  & \leq \frac{c_3}{c_2(k)} \E\left[ \1_{\{k \leq n\}} \bigg(\left\lvert \proj{V}(M_{t_k}) \odot ( \hat{V}_{t_k-} - \eta_{t_k-} ) \right\rvert_2  +  \left\lvert \proj{V}(M_{t_k}) \odot ( \hat{V}_{t_k} - \eta_{t_k} ) \right\rvert_2 \bigg) \right].
\end{split}
\end{equation}
\end{small}
To finish the last part of this first step, consider $\eta \in \mathbb{D}$ which is  not indistinguishable from $\hat{V}$ at observations. Hence, by \Cref{def:indistinguishability}, there exists some $k \in \{1,\dots,\bar n \}$ such that $d_k(\hat{V},\eta)>0.$ We have by \Cref{lem:L2 identity}
\begin{small}
\begin{equation*}\label{equ:thm1-lower-bound-eta}
\begin{split}
    \Psi(\eta) &= \E\left[\frac{1}{n}\sum_{i=1}^n \left( \left\lvert \proj{V}(M_{i}) \odot (V_{t_i} - \eta_{t_i} ) \right \rvert_2^2 + \left\lvert \proj{V}(M_{i}) \odot ( V_{t_i} - \eta_{t_i-} ) \right \rvert_2^2 \right)\right]\\
    &= \E\left[ \dfrac{1}{n}\sum_{i=1}^n \left\lvert \proj{V}(M_{t_i}) \odot ( V_{t_i} - \hat{V}_{t_i-} ) \right\rvert_2^2\right] + \E\left[\dfrac{1}{n}\sum_{i=1}^n \left\lvert \proj{V}(M_{t_i}) \odot (  \hat{V}_{t_i-} - \eta_{t_i-}) \right\rvert_2^2\right]\\ 
    &\quad +   \E\left[ \dfrac{1}{n}\sum_{i=1}^n \left\lvert \proj{V}(M_{t_i}) \odot ( V_{t_i} - \hat{V}_{t_i} ) \right\rvert_2^2\right] + \E\left[\dfrac{1}{n}\sum_{i=1}^n \left\lvert \proj{V}(M_{t_i}) \odot (  \hat{V}_{t_i} - \eta_{t_i}) \right\rvert_2^2\right]\\
    &= \Psi(\hat{V}) + \E\left[\dfrac{1}{n}\sum_{i=1}^n \left\lvert \proj{V}(M_{t_i}) \odot (  \hat{V}_{t_i-} - \eta_{t_i-}) \right\rvert_2^2 +  \proj{V}(M_{i}) \odot (\hat {V}_{t_i} - \eta_{t_i} ) \bigg\rvert_2^2\right],
\end{split}
\end{equation*}
\end{small}
and it is now easy to see that the second addend is greater than $0$: indeed, we have
\begin{align}
    \E & \left[\dfrac{1}{n}\sum_{i=1}^n \left\lvert \proj{V}(M_{t_i}) \odot (  \hat{V}_{t_i-} - \eta_{t_i-}) \right\rvert_2^2 +  \left\lvert \proj{V}(M_{i}) \odot (\hat {V}_{t_i} - \eta_{t_i} ) \right\rvert_2^2\right] \notag \\
    &=\E  \bigg[\dfrac{1}{n}\sum_{i=1}^{\bar n} \1_{\{ i \leq n \}} \bigg( \left\lvert \proj{V}(M_{t_i}) \odot (  \hat{V}_{t_i-} - \eta_{t_i-}) \right\rvert_2^2 + \left\lvert \proj{V}(M_{i}) \odot (\hat {V}_{t_i} - \eta_{t_i} ) \right\rvert_2^2 \bigg)\bigg] \notag \\
    & \geq \E\left[\tfrac{1}{n}  \1_{\{ k \leq n \}} \left( \left\lvert \proj{V}(M_{t_k}) \odot (  \hat{V}_{t_k-} - \eta_{t_k-}) \right\rvert_2^2 + \left\lvert \proj{V}(M_{t_k}) \odot ( \hat{V}_{t_k} - \eta_{t_k}) \right\rvert_2^2 \right)\right] \notag\\ 
    & \geq \dfrac{1}{c_1^{2}} \bigg(\E\left[ \1_{\{ k \leq n \}} \left\lvert \proj{V}(M_{t_k}) \odot (  \hat{V}_{t_k-} - \eta_{t_k-}) \right\rvert_2\right]^2 
     + \E\left[ \1_{\{ k \leq n \}} \left\lvert \proj{V}(M_{t_k}) \odot (  \hat{V}_{t_k} - \eta_{t_k}) \right\rvert_2\right]^2 \bigg) 
     \label{temp:676}
\end{align}
using  \Cref{equ:HI} for both addends separately. Moreover,
\begin{align}\label{equ:thm1-positive expectation}
     \eqref{temp:676}
    & \geq \dfrac{1}{2 c_1^{2}} \E\left[ \1_{\{ k \leq n \}} \left( \left\lvert \proj{V}(M_{t_k}) \odot (  \hat{V}_{t_k-} - \eta_{t_k-}) \right\rvert_2 +  \left\lvert \proj{V}(M_{t_k}) \odot (  \hat{V}_{t_k} - \eta_{t_k}) \right\rvert_2 \right)\right]^2 \notag \\   
    & \geq \frac{1}{2} \left( \frac{c_2(k)}{c_1 c_3} \right)^2  \E\left[ \1_{\{ n \geq k \}} \left( \left\lvert    \hat{V}_{t_k-} - \eta_{t_k-} \right\rvert_2 +  \left\lvert   \hat{V}_{t_k} - \eta_{t_k} \right\rvert_2 \right)\right]^2  = \frac{1}{2} \left( \frac{c_2(k)}{c_0(k) c_1 c_3} \right)^2 d_k(\hat{V}, \eta)^2 > 0, \notag
\end{align} 
where we have used the Cauchy-Schwarz inequality $(a+b)^2 \leq 2a^2+2b^2$,  Equation \eqref{equ:M split}  and finally \eqref{equ: pseudo metric}.
This gives  $\Psi(\eta) > \Psi(\hat{V})$, and therefore  $\Psi$ has the unique minimizer $\hat{V}$. 

\textbf{Step 2:} Bounding the distance between $G^{\theta_m^\star}$ and $\hat{V}$.

For the second step, we start by showing that the IO NJODE model \eqref{equ:PD-NJODE} can approximate $\hat{V}$ arbitrarily well. Let us fix the dimension $d_H :=d_V$ and choose $\tilde \theta_3^\star$ such that $ \tilde g_{\theta_3^{\star}} = \id$. We define for $\varepsilon > 0$,   $N_\varepsilon := \lceil 2 (T+1) \varepsilon^{-2} \rceil  $,\footnote{$\lceil \cdot \rceil$ is the ceiling function.} and $\mathcal{P}_\varepsilon$ as the closure of piecewise linear functions $A_{ N_\varepsilon  }$ bounded by $N_\varepsilon$, defined in Proposition~\ref{prop:compact subset BV}, which is compact. Using Assumption \ref{assumption:4}, for any $j \in [d_U+1,d]$, the continuous function $F_j$ representing $\hat V_j$  can be written as $ F_j(t, \tau(t), \tilde U^{\leq \tau(t) } -U_0, U_0 )$. We obtain the approximation $\hat F_j$ which is a function of the truncated signature only  by Proposition~\ref{prop:universal_approx_sig}, i.e., there exists $m_0 = m_0(\varepsilon) \in \N$ and  continuous  $\hat F_j$ such that
\begin{equation*}
\sup_{(t, \tau, U) \in [0,T]^2\times \mathcal{P}_\varepsilon } \left| F_j(t, \tau, U ) - \hat F_j(t, \tau, \pi_{m_0}( U -U_0 ), U_0 )\right| \leq \varepsilon/2.
\end{equation*}
Due to the uniformly bounded variation of functions in $\mathcal{P}_\varepsilon$, the set of truncated signatures $\pi_{m_0}(\mathcal{P}_\varepsilon)$ is bounded in $\R^{d'}$ for some ${d'} \in \N$ (depending on $d_U$ and $m_0$). Its closure, denoted by $\Pi_\varepsilon$, is compact. By the universal approximation theorem for neural networks \citep[Theorem 2.4]{10.5555/70405.70408}, there exists $m_{1,1} \in \N$ and neural network weights $\tilde{\theta}_{1,1}^{\star, m_{1,1}} \in {\tilde{\Theta}_{m_{1,1}}^{1}}$ such that for each $j \in [d_U+1,d]$ the function $\hat F_j$ can be estimated by the $j$-th coordinate of the neural network $\tilde f_{\tilde{\theta}_{1,1}^{\star, m_{1,1}}} \in \tilde{\mathcal{N}}$ on the compact set $[0,T]^2\times \Pi_\varepsilon$ within $\varepsilon/2$. Thus, we obtain
\begin{align*}
    &\sup_{(t, \tau, U) \in [0,T]^2\times \mathcal{P}_\varepsilon } \left| F_j(t, \tau, U ) -  \tilde f_{\tilde \theta_1^{\star, m_1}, j}  (t, \tau, \pi_{m_0}( U -U_0 ), U_0 )\right| \notag\\ 
    & \  \qquad \leq \sup_{(t, \tau, U) \in [0,T]^2\times \mathcal{P}_\varepsilon } \left| F_j(t, \tau, U ) -\hat F_j(t, \tau, \pi_{m_0}( U -U_0 ), U_0 )\right|\notag\\ 
    & \ \qquad + \sup_{(t, \tau, U) \in [0,T]^2\times \mathcal{P}_\varepsilon } \left| \hat F_j(t, \tau, \pi_{m_0}( U -U_0 ), U_0 ) -\tilde f_{\tilde \theta_1^{\star, m_1}, j}  (t, \tau, \pi_{m_0}( U -U_0 ), U_0 )\right| 
     \leq \varepsilon/2 +\varepsilon/2 = \varepsilon.
\end{align*}
Analogously,  there exist $m_{2,1} \in \N$, neural network weights $\tilde{\theta}_{2,1}^{\star, m_{2,1}} \in \tilde{\Theta}_{m_{2,1}}^{2}$ and a network $\tilde \rho_{\tilde \theta_{2,1}^{\star, m_{2,1}}, j} \in \tilde{\mathcal{N}}   $ such that for each $j \in [d_U+1,d]$
\begin{equation*}
\sup_{(t, U) \in [0,T] \times \mathcal{P}_\varepsilon } \left| F_j(t, t, U ) -  \tilde \rho_{\tilde \theta_{2,1}^{\star, m_{2,1}}, j}  (t, \pi_{m_{0}}( U -U_0 ), U_0 )\right| \leq \varepsilon.
\end{equation*}
We note that we can extend the neural network by adding $H_{t^-}$ as additional input variable, without decreasing the approximation quality.

Subsequent, we construct the bounded output neural networks by defining $\gamma_1$ and $\gamma_2$ as
\begin{equation*}
\gamma_1 := \max_{(t, \tau, U) \in [0,T]^2\times \mathcal{P}_\varepsilon } \left|  \tilde f_{\tilde \theta_1^{\star, m_1}}  (t, \tau, \pi_{m_0}( U -U_0 ), U_0 )\right|.
\end{equation*}
\begin{equation*}
\gamma_2 := \max_{(t, \tau, U) \in [0,T]^2\times \mathcal{P}_\varepsilon } \left|  \tilde \rho_{\tilde \theta_2^{\star, m_2}, j}  (t, \pi_{m_0}( U -U_0 ), U_0 )\right|.
\end{equation*}
Since  the maximum of a continuous function on a compact set is finite, both constants are finite. Hence, the bounded output neural networks $\bar f_{\bar \theta_{1,1}^{\star, m_{1,1}}},\  \bar \rho_{\bar  \theta_{2,1}^{\star, m_{2,1}}} \in \mathcal{N}$ are well defined with $\bar \theta_{i,1}^{\star, m_{i,1}} := (\tilde \theta_{i,1}^{\star, m_{i,1}}, \gamma_i)$. As already remarked, the truncation functions  coincide with the identity for $|x|_2 \le \gamma_i$, $i=1,2$, respectively, and so the bounded networks match the original networks on $[0,T]^2\times \mathcal{P}_\varepsilon $ and therefore provide the same $\varepsilon$-approximation.

This boundedness alone is insufficient to demonstrate the convergence referred to in this part of the proof. The bounds do not always result in boundedness by $f_j, F_j$ and depend on $\epsilon$. Specifically, as $\epsilon$ decreases and $m$ increases, $\gamma_i$ may grow so quickly that the expectation of neural networks outside the compact set does not converge to zero as the compact set expands. Therefore, we compensate these neural networks with additional ones. To ascertain whether the input trajectory resides within the pertinent compact set, we employ the random variables 
$\tilde U^{\star}_t := \sup_{0\leq s \leq t} | \tilde U^{\leq t}_s|_1 \leq U^\star_T$, $n_t := \max\{ i \in \N \, | \, t_i \leq t \} \leq n$ and $\delta_t := \min\{ t_i - t_{i-1} | t_i \leq t \} \geq \delta_{\min}$
%
%
%
as supplementary inputs for the neural networks. By assigning the respective weights to zero in the neural networks described previously in \textbf{step 2}, the $\varepsilon$-approximation remains intact.

On the one hand, if $\tilde{U}^{\star}_t \leq 1/\epsilon$, $n_t \leq 1/\epsilon$, and $\delta_t \geq \epsilon$, then $\tilde{U}^{\leq t} - U_0 \in A_{N_\epsilon} \subset \mathcal{P}_\epsilon$, thereby ensuring the $\epsilon$-approximation holds. On the other hand, we have to study the case where either $U^\star_T > 1/\epsilon$, $n > 1/\epsilon$, or $\delta_{\min} < \epsilon$. We define $\tilde \mu_1$ to be the push-forward measure of $dt \times \P$ through the measurable map 
$\tilde \mu_1:[0,T] \times \Omega  \to \R^{2+d'+d_U+3}$ given by 
\begin{equation}\label{equ:thm1 input to NN 1}
    \tilde\mu_1(t) = \left(t, \tau(t), \pi_{m_0}(\tilde{U}^{\leq t} - U_0), U_0,U^{\star}_t, n_t, \delta_t\right).
\end{equation}
Moreover, we define $D_1 := [0,T]^2\times \R^{d'} \times \R^{d_U}\times [0, 1/\epsilon]^2 \times [\epsilon, T]$. We denote the complement as $D_1^\complement$, and  define
\begin{equation*}
    \upsilon := \bar f_{\theta_{1,1}^{\star, m_{1,1}}} \1_{ D_1^\complement } : \R^{2+{d'}+d_U} \to \R^{d_U}.
\end{equation*}
 
Thus, $\tilde\mu_1 (\R^{2+{d'}+d_U+3}) = T$, which means that $\tilde\mu_1$ is a finite measure. Furthermore, $\upsilon$ belongs to $L^2(\tilde\mu_1)$ as it is bounded by $\gamma_1$.
So, according to \citet[Theorem~1]{hornik1991approximation}, there exists $m_{1,2} \in \N$ and neural network weights $\tilde{\theta}_{1,2}^{\star, m_{1,2}} \in \tilde{\Theta}_{m_{1,2}}^{1}$ such that  the corresponding neural network $\tilde f_{\tilde \theta^{\star, m_{1,2}}_{1,2}}$ satisfies $\int |\upsilon - \tilde f_{\tilde \theta^{\star, m_{1,2}}_{1,2}} |_2^2 d\tilde\mu_1 < \epsilon$. We can assume $ \tilde{f}_{\tilde{\theta}^{\star, m_{1,2}}_{1,2}} \in \mathcal{N} $, which means that this is an output-bounded neural network.
Then we define the first new neural network $f_{ \theta^{\star, m_{1}}_{1}} = \bar f_{\bar \theta_{1,1}^{\star, m_{1,1}}} - \tilde f_{\tilde \theta^{\star, m_{1,2}}_{1,2}} \in \mathcal{N}$, with the weights $\theta^{\star, m_{1}}_{1} = (\bar \theta_{1,1}^{\star, m_{1,1}}, \tilde \theta^{\star, m_{1,2}}_{1,2})$ and $m_1 = m_{1,1} + m_{1,2}$.

We follow an analogous approach for $F_j$, noting that the process varies slightly due to the differences in the push-forward map. We define $\tilde \mu_2$ to be the push-forward measure of $dq \times \P$ (where $dq$ denotes the Lebesgue measure) through the measurable map $\tilde \mu_2: [0,1] \times \Omega  \to \R^{1+{d'}+d_U+3}$ defined by 
\begin{equation}\label{equ:thm1 input to NN 2}
   \mu(Q, \omega) = \left(\bar t, \pi_{m_0}(\tilde{U}^{\leq \bar t} - U_0), U_0,U^{\star}_t, n_t, \delta_t \right) ,
\end{equation}
where $\bar t := \bar t (Q, (\omega, \tilde \omega)) := \sum_{i=1}^n \1_{(\frac{i-1}{2n}, \frac{i}{2n}]}(Q) \, t_i(\tilde \omega) + \sum_{i=1}^n \1_{(\frac{n+i-1}{2n}, \frac{n+i}{2n}]}(Q) \, t_{i-1}(\tilde \omega)$ as in \Cref{lem:expectation weighted sum over t_i terms}. Let $D_2 := [0,T]^2\times \R^{d'} \times \R^{d_U}\times [0, 1/\epsilon]^2 \times [\epsilon, T]$ and also define
\begin{equation*}
    \Upsilon :=  \bar \rho_{\theta_{2,1}^{\star, m_{2,1}}}  \1_{ D_2^\complement } : \R^{2+d'+d_U} \to \R^{d_U}.
\end{equation*}

Then, $\tilde \mu_2 (\R^{1+d'+d_U+3}) = 1$, i.e., $\tilde\mu_2$ is a finite measure and $\Upsilon$ is a function in $L^2(\tilde\mu_2)$, since it is bounded by $\gamma_2$.
According to \citet[Theorem~1]{hornik1991approximation}, there exists $m_{2,2} \in \N$ and neural network weights $\tilde \theta^{\star, m_{2,2}}_{2,2} \in \tilde{\Theta}_{m_{2,2}}^{2}$ such that the corresponding neural network $\tilde \rho_{\tilde \theta^{\star, m_{2,2}}_{2,2}}$ satisfies $\int |\Upsilon - \tilde \rho_{\tilde \theta^{\star, m_{2,2}}_{2,2}} |_2^2 d\tilde\mu_2 < \epsilon$. Then we define the (new) neural network $\rho_{ \theta^{\star, m_{2}}_{2}} = \bar \rho_{\bar \theta_{2,1}^{\star, m_{2,1}}} - \tilde \rho_{\tilde \theta^{\star, m_{2,2}}_{2,2}} \in \mathcal{N}$, with the weights $\theta^{\star, m_{2}}_{2} = (\bar \theta_{2,1}^{\star, m_{2,1}}, \tilde \theta^{\star, m_{2,2}}_{2,2})$ and $m_2 = m_{2,1} + m_{2,2}$.
{It is important to recognize that we can't directly approximate $f$, $F$ or the difference between them and their initial approximating neural networks because the input space of $f$ and $F$ is infinite-dimensional. This scenario is not addressed by universal approximation theorems.}

Setting $m:= \max(m_0, m_1, m_2, \gamma_1, \gamma_2, |\tilde \theta_1^{\star, m_1}|_2, |\tilde \theta_2^{\star, m_2}|_2)$, it follows that $\theta^\star_m := (\theta^{\star, m_1}_1, \theta^{\star, m_2}_2, \tilde \theta^\star_3 ) \in \Theta_m$.

We can then finalize the purpose of this \textbf{step 2}. Throughout the rest of the proof, we will use the notation
\begin{equation}
    \begin{split}
        \zeta &:= \left( \1_{\{ U_T^\star \geq 1/\varepsilon \}} + \1_{\{ n \geq  1/\varepsilon \}} + \1_{\{ \delta \leq \epsilon \}} \right),
    \end{split}
\end{equation}
and we know by the argument above that on $D_i$ the $\epsilon$-approximation holds for any input of the form \eqref{equ:thm1 input to NN 1} or \eqref{equ:thm1 input to NN 2} respectively, while we have $\1_{D_i^\complement} \leq \zeta$ for $i=1,2$ and any potential input where the $t$ coordinates are within $[0,T]$.
we set $F = (F_j)_{d_U+1 \leq j \leq d}$ and $f = (f_j)_{d_U+1 \leq j \leq d}$.  For $t \in \{t_1, \dotsc, t_n\} \subset [0,T]$ we have with $\1_{D_2} + \1_{D_2^\complement}=1$ and by triangle inequality that
\begin{align*}
\left\lvert G_t^{\theta_m^\star} - \hat{V}_t  \right\rvert_1 
	& = \left\lvert  \rho_{\theta_2^{\star, m_2}}\left(H_{t-}, t,\pi_{m}( \tilde U^{\leq t} - U_0 ), U_0,U^{\star}_t, n_t, \delta_t \right) - F \left(t, t, \tilde U^{\leq t} \right) \right\rvert_1  \\
	 &= \left\lvert  \rho_{\theta_2^{\star, m_2}} - F \right\rvert_1  (\1_{D_2} + \1_{D_2^\complement}),\\
    &\leq \left( \left\lvert  \bar \rho_{\bar \theta_{2,1}^{\star, m_{2,1}}} - F \right\rvert_1  + \left\lvert  \tilde \rho_{\tilde \theta^{\star, m_{2,2}}_{2,2}} \right\rvert_1 \right) \1_{D_2} + \left( \left\lvert  \bar \rho_{\bar \theta_{2,1}^{\star, m_{2,1}}} - \tilde \rho_{\tilde \theta^{\star, m_{2,2}}_{2,2}} \right\rvert_1 +  \left\lvert F \right\rvert_1 \right) \1_{D_2^\complement} \\
    &\leq \left( \epsilon d  + \left\lvert   \Upsilon - \tilde \rho_{\tilde \theta^{\star, m_{2,2}}_{2,2}} \right\rvert_1 \right) \1_{D_2} + \left( \left\lvert  \Upsilon - \tilde \rho_{\tilde \theta^{\star, m_{2,2}}_{2,2}} \right\rvert_1  +  \left\lvert F \right\rvert_1 \right) \1_{D_2^\complement} \\
    & \leq \epsilon d  + \left\lvert   \left(\Upsilon - \tilde\rho_{\tilde \theta^{\star, m_{2,2}}_{2,2}}\right) (t) \right\rvert_1 +  \left\lvert F(t) \right\rvert_1 \zeta,
\end{align*}
where we write $(t)$ to specify the time input for the respective functions (e.g., $F(t)$).  For $t \in [0,T] \setminus \{ t_1, \dotsc, t_n \}$,
\begin{equation*}
\begin{split}
 \left\lvert G_t^{\theta_m^\star} - \hat{V}_t  \right\rvert_1  
	&\leq  \left\lvert G_{\tau(t)}^{\theta_m^\star} - \hat{V}_{\tau(t)}  \right\rvert_1 \\
 & \quad +   \int_{\tau(t)}^t \left\lvert f_{\theta_1^{\star, m_1}}\left(H_{s-}, s, \tau(t), \pi_{m} (\tilde U^{\leq \tau(t)} -U_0 ), U_0,U^{\star}_t, n_t, \delta_t \right) 
		- f(s, \tau(t), \tilde U^{\leq \tau(t)})  \right\rvert_1 ds\\
	& \leq \left( \epsilon d  + \left\lvert  \left( \Upsilon - \tilde\rho_{\tilde \theta^{\star, m_{2,2}}_{2,2}} \right) (\tau(t))\right\rvert_1 +  \left\lvert F(\tau(t)) \right\rvert_1 \zeta\right) 
    + \int_{0}^T \left( \epsilon d  + \left\lvert   \upsilon - \tilde f_{\tilde \theta^{\star, m_{1,2}}_{1,2}} \right\rvert_1 +  \left\lvert f \right\rvert_1 \zeta \right) ds, 
\end{split}
\end{equation*}
Using the same arguments for norm equivalence as in \textbf{step 1}, there exists a constant $c>0$ such that for all $t \in [0,T]$
\begin{equation*}
\begin{split}
&\left\lvert G_t^{\theta_m^\star} - \hat{V}_t  \right\rvert_2  \leq c \left( \epsilon (1+T)d  + \left\lvert  \left( \Upsilon - \tilde\rho_{\tilde \theta^{\star, m_{2,2}}_{2,2}} \right)\right\rvert_2 +  \left(\lvert F \rvert_2 +\int_0^T \left\lvert f(t) \right\rvert_2 dt \right)\zeta
    + \int_{0}^T  \left\lvert   \upsilon - \tilde f_{\tilde \theta^{\star, m_{1,2}}_{1,2}} \right\rvert_2 ds  \right) .
\end{split}
\end{equation*}
We have the bound
\begin{align}\label{equ:thm1 bounding the difference Y - hat X}
    &\E \left[ \frac{1}{n} \sum_{i=1}^n \left\lvert G_{t_i-}^{\theta_m^\star} - \hat{V}_{t_i-}  \right\rvert_2^2 \right] 
    + \E \left[ \frac{1}{n} \sum_{i=1}^n \left\lvert G_{t_i}^{\theta_m^\star} - \hat{V}_{t_i}  \right\rvert_2^2 \right] \notag \\
    & \leq 7 \left( 
        \epsilon^2 ((1+T)d)^2
        + \E\left[ \left( \frac{1}{n}\sum_{i=1}^n \left( \left\lvert F(t_i) \right\rvert_2^2 +\left\lvert F(t_{i-1}) \right\rvert_2^2 \right) + \left( \int_0^T\left\lvert f(t) \right\rvert_2 dt\right)^2  \right) \zeta  \right] \right. \notag\\
        & \quad+ \E \left[ \frac{1}{n} \sum_{i=1}^n \left\lvert   \left( \Upsilon - \tilde\rho_{\tilde \theta^{\star, m_{2,2}}_{2,2}} \right) (t_{i-1}) \right\rvert_2^2 \right] \notag\\
        &\quad \left. + \E \left[ \frac{1}{n} \sum_{i=1}^n \left\lvert   \left( \Upsilon - \tilde\rho_{\tilde \theta^{\star, m_{2,2}}_{2,2}} \right) (t_{i}) \right\rvert_2^2 \right]
        +  \E \left[ \left( \int_{0}^T \left\lvert   \upsilon - \tilde f_{\tilde \theta^{\star, m_{1,2}}_{1,2}} \right\rvert_2  ds \right)^2 \right]
    \right) \notag\\
    & \leq 7 \left( 
        \epsilon^2 ((1+T)d)^2
        + \E_{\mu_Q \times \P}\left[ \left\lvert F(\bar t) \right\rvert_2^2 \zeta  \right] + \E\left[ T \int_0^T \left\lvert f(t) \right\rvert_2^2 dt \, \zeta  \right] \right. \notag \\
        &\quad\qquad+ \E_{\mu_Q \times \P} \left[ \left\lvert   \left( \Upsilon - \tilde\rho_{\tilde \theta^{\star, m_{2,2}}_{2,2}} \right) (\bar t) \right\rvert_2^2 \right] 
        + \left. \E \left[ T  \int_{0}^T \left\lvert   \upsilon - \tilde f_{\tilde \theta^{\star, m_{1,2}}_{1,2}} \right\rvert_2^2  ds \right]
    \right) \\
    & = 7 \Big( 
        \epsilon^2 ((1+T)d)^2
        + \E_{\mu_Q \times \P}\left[ \left\lvert F(\bar t) \right\rvert_2^2 \zeta  \right] + \E\left[ T \int_0^T \left\lvert f(t) \right\rvert_2^2 dt \, \zeta  \right]  \notag \\
        & \quad\qquad + \left. \int \left\lvert    \Upsilon - \tilde\rho_{\tilde \theta^{\star, m_{2,2}}_{2,2}} \right\rvert_2^2 d \tilde\mu_2 
        + T \int \left\lvert   \upsilon - \tilde f_{\tilde \theta^{\star, m_{1,2}}_{1,2}} \right\rvert_2^2  d \tilde\mu_1 
    \right) \notag\\
    & = 7 \left( 
        \epsilon^2 ((1+T)d)^2
        + \E_{\mu_Q \times \P}\left[ \left\lvert F(\bar t) \right\rvert_2^2 \zeta  \right] + \E\left[ T \int_0^T \left\lvert f(t) \right\rvert_2^2 dt \, \zeta  \right] + \epsilon(1+T)
    \right) , \notag
\end{align}
where we used Cauchy-Schwarz inequality in the first inequality. For the second inequality, we used \Cref{lem:expectation weighted sum over t_i terms} for the second, third and fourth term and for second and the last term, we used Hölder inequality for $p=q=2$. For the third equality we used the definition of the push-forward measures $\tilde \mu_1, \tilde\mu_2$ and  we applied the $L^2$ approximation for $\upsilon$ and $\Upsilon$ through the designated neural networks for the last one.

Now, we want to show that the $\varepsilon$-error converges to $0$ when increasing the truncation level and network size $m \in \N$. To do that, we define $\varepsilon_m \geq 0$ to be the smallest number such that the bounds given above with error $\varepsilon_m$ hold, when using an architecture with signature truncation level $m$ and weights in $\Theta_m$.
Increasing $m$ can only improve the approximation, and we have proven before that for any $\epsilon$ there exists some $m$ such that the $\epsilon$-approximation holds. Hence, $\lim_{m \to \infty} \varepsilon_m = 0 $. Let $\theta^\star_m \in \Theta_m$ denote the choice for the neural networks' weights as described above.
By hypothesis, $\theta_m^{\min} \in \argmin _{\theta \in \Theta_m}\{ L(\theta) \}$, therefore, we obtain for $\bar t := \sum_{i=1}^n \1_{(\frac{i-1}{n}, \frac{i}{n}]}(Q) \, t_i$ with $Q \sim \operatorname{Unif}([0,1])$ independent of $\mathcal{F}$ and  $R_{t_i}^2 := \left\lvert G_{t_i-}^{\theta_m^\star} - \hat{V}_{t_i-}  \right\rvert_2^2 + \left\lvert G_{t_i}^{\theta_m^\star} - \hat{V}_{t_i}  \right\rvert_2^2$

\begin{align}\label{equ:phi convergence 1}
\min_{\eta \in \mathbb{D}} \Psi(\eta) 
    & \leq L(\theta_m^{\min}) \leq L(\theta_m^{\star}) \notag \\
    & = \E\left[ \frac{1}{n} \sum_{i=1}^n \left(  \left\lvert \proj{V}(M_{t_i}) \odot ( V_{t_i} - G_{t_i}^{\theta_m^{\star}} ) \right\rvert_2^2 + \left\lvert \proj{V}(M_{t_i}) \odot (  V_{t_i} - G_{t_{i}-}^{\theta_m^{\star}} ) \right\rvert_2^2 \right) \right] \notag \\
    & = \E\left[ \frac{1}{n} \sum_{i=1}^n \left( \left\lvert \proj{V}(M_{t_i}) \odot ( {V}_{t_i} -  \hat{V}_{t_i} ) \right\rvert_2^2 +  \left\lvert \proj{V}(M_{t_i}) \odot ( \hat{V}_{t_i} - \ G_{t_i}^{\theta_m^{\star}}) \right\rvert_2^2  \right. \right. \notag \\
	&			\left.\left. \qquad \qquad \qquad  + \left\lvert \proj{V}(M_{t_i}) \odot ( {V}_{t_i} -  \hat{V}_{t_i-} ) \right\rvert_2^2 +  \left\lvert \proj{V}(M_{t_i}) \odot ( \hat{V}_{t_i-} - \ G_{t_i-}^{\theta_m^{\star}}) \right\rvert_2^2  \right) \right] \notag \\
    & \leq  \Psi(\hat V) +\E\left[\frac{1}{n} \sum_{i=1}^n R_{t_i}^2 \right]  , 
\end{align}
whereby the $4$th (in)equality follows from \Cref{lem:L2 identity}.
Together with Assumption~\ref{assumption:2} on $\delta_{\min}$, integrability of $|U_T^\star|_2$ and $|n|$ (Assumptions~\ref{assumption:5} and~\ref{assumption:6}) suggest that 
$$\zeta:= \1_{\{ U_T^\star \geq 1/\varepsilon_m \}} + \1_{\{ n \geq  1/\varepsilon_m \}} + \1_{\{ \delta \leq \varepsilon_m \}} \xrightarrow[m \to \infty]{\P-a.s.} 0.$$ 
Using this and \eqref{equ:thm1 bounding the difference Y - hat X} together with \Cref{assumption:4}, we get by dominated convergence that
\begin{equation*}
\begin{split}
\E\left[\frac{1}{n} \sum_{i=1}^n R_{t_i}^2 \right]   \xrightarrow{m \to \infty} 0.
\end{split}
\end{equation*}
With this limit and the fact that $\Psi(\hat{V})= \min_{\eta \in \mathbb{D}} \Psi(\eta)$, we directly obtain from \eqref{equ:phi convergence 1}
\begin{equation*}
\min_{\eta \in \mathbb{D}} \Psi(\eta) 
	 \leq L(\theta_m^{\min}) \leq L(\theta_m^{\star}) \xrightarrow{m \to \infty} \min_{\eta \in \mathbb{D}} \Psi(\eta)  .
\end{equation*}

\textbf{Step 3:} For every $1 \leq k \leq K$ we want to show that $G^{\theta_m^{\min}}$ converges to $\hat{V}$ in the metric $d_k$  as $m \to \infty$ i.e. $\lim_{m\to\infty} d_k \left( \hat{V},  G^{\theta_m^{\min}} \right) = 0$ for all $1 \leq k \leq K$. 
First, Lemma \ref{lem:L2 identity} yields
\begin{equation}\label{equ:proof thm bound 1}
\begin{split}
L(\theta_m^{\min}) - \Psi(\hat{V}) 
	&= \E\left[\frac{1}{n} \sum_{i=1}^n  \left\lvert \proj{V}(M_{t_i}) \odot ( \hat{V}_{t_i-} - G^{\theta_m^{\min}}_{t_{i}-} ) \right\rvert_2^2 \right]\\ 
    & \qquad \qquad \qquad \quad+  \E\left[\frac{1}{n}\sum_{i=1}^n  \left\lvert \proj{V}(M_{i}) \odot (\hat{V}_{t_i} - G^{\theta_m^{\min}}_{t_{i}} ) \right \rvert_2^2\right].
\end{split}
\end{equation}
Hence, applying \eqref{equ:HI}, \eqref{equ:M split} 
and finally \eqref{equ:proof thm bound 1}, yields
\begin{equation}\label{equ:convergence in L1}
\begin{split}
d_k \left( \hat{V},  G^{\theta_m^{\min}} \right)
	& \leq   \frac{c_0(k) \, c_1 \, c_3}{c_2(k)} \, \left( L(\theta_m^{\min}) - \Psi(\hat{V}) \right)^{1/2}  \xrightarrow{m \to \infty} 0,
\end{split}
\end{equation}
completing the proof.
\end{proof}

\subsection{Convergence of the Monte Carlo Approximation}\label{sec:Convergence of the Monte Carlo approximation}

In this section, we present the convergence with respect to the Monte Carlo approximation $\hat{L}_N$ as the number of samples $N$ increases. This result builds on and extends \citet[Theorem~4.4]{krach2022optimal}. Throughout this section, we assume that the size $m$ of the neural network and the signature truncation level are fixed.

\begin{theorem}
\label{thm:MC convergence Yt}
Let $\theta^{\min}_{m,N} \in \Theta^{\min}_{m,N} := \argmin_{\theta \in \Theta_m} \{\hat L_N(\theta)\}$ for all $m, N \in \mathbb{N}$. Then, for each $m \in \mathbb{N}$, we have $(\mathbb{P} \times \tilde{\mathbb{P}})$-almost surely  as $N \to \infty$ that 
\begin{itemize}
    \item $\hat L_N$ converges uniformly to $L$ on $\Theta_m$,  
    \item $L(\theta^{\min}_{m,N})$ converges to $L(\theta^{\min}_{m})$ and
    \item $\hat L_N(\theta^{\min}_{m,N})$ converges to $L(\theta^{\min}_{m})$.
\end{itemize}
Furthermore, there exists an increasing random sequence $(N_m)_{m \in \mathbb{N}}$ in $\mathbb{N}$ such that for every $1 \leq k \leq K$, it holds almost surely that
$\lim_{m\to\infty} d_k ( \hat{V},  G^{\theta_{m, N_m}^{\min}} ) = 0$.
\end{theorem}

Before giving the proof, we introduce some additional notation  and helpful results.
We define the separable Banach space $\mathcal{S} := \ell^1(\R^{d_s}) = \{ x = (x_i)_{ i \in \N} \in (\R^{d_s})^{\N} \; \vert \; \lVert x \rVert_{\ell^1} < \infty \}$ for an appropriate $d_s \in \N$, equipped with the norm $\lVert x \rVert_{\ell^1} := \sum_{i \in \N} \lvert x_i \rvert_2$. Moreover, we define the function
\begin{equation*}
\phi(x,y,z,p) :=  \left\lvert p \odot ( x - y ) \right\rvert_2^2 + \left\lvert p \odot (x - z) \right\rvert_2^2 
\end{equation*}
and $\xi_j := (\xi_{j, 0}, \dotsc, \xi_{j, n^{(j)}}, 0, \dotsc)$, where $\xi_{j,k} := (t_{k}^{(j)}, U_{t_k^{(j)}}^{(j)}, M_{t_k^{(j)}}^{(j)}, \pi_m(\tilde U^{\leq t_k^{(j)}, {(j)}} - U_0^{(j)}), \tilde U^{\star, (j)}_t, k, \delta_{t_k^{(j)}}^{(j)}) \in \R^{d_s}$. 
Let $n^{j}(\xi_j) := \max_{k \in \N}\{ \xi_{j,k} \neq 0\}$, $t_k(\xi_j):= t_{k}^{(j)}$, $V_k (\xi_j) := V_{t_k^{(j)}}^{(j)}$ and $M_k(\xi_j) := M_{t_k^{(j)}}^{(j)}$.
Then we have  $n^{(j)} = n^{j}(\xi_j)$ $\P$-almost-surely. Additionally, we know that $\xi_j$ are i.i.d. random variables taking values in $\mathcal{S}$. As before, we write $G^\theta$ to emphasize the dependence of the model output on the network weights $\theta$ and define 
\begin{equation*}
W(\theta, \xi_j) := \frac{1}{n^{j}(\xi_j)}\sum_{i=1}^{n^{j}(\xi_j)}  \phi \left( V_i(\xi_j), G^\theta_{t_i(\xi_j)}(\xi_j), G^\theta_{t_i(\xi_j)-}(\xi_j), \proj{V}(M_i(\xi_j)) \right).
\end{equation*}

The following results were proven in \citet[Lemmas~4.7 and~4.8]{krach2022optimal}.
\begin{lemma}\label{lem:properties for MC conv thm}
For every $t \in [0, T]$, the random function  $\theta \in \Theta_M \mapsto G_{t}^{\theta}$  is uniformly continuous almost surely.
\end{lemma}

\begin{lemma}
\label{lemma:convergencelocallyuniform}
Let $(\xi_i)_{i \geq 1}$ be a sequence of $i.i.d.$ random variables with values in $\mathcal{S}$ and $W:\mathbb{R}^m\times \mathcal{S}\to \mathbb{R}$ be a measurable function.
Assume that a.s., the function $\theta\in \mathbb{R}^m \mapsto W(\theta, \xi_1)$ is continuous and for all $C>0$, $\E(\sup_{|\theta|_2 \leq C}|W(\theta, \xi_1)|)< + \infty$. Then, a.s. $f_N:  \mathbb{R}^m \to \R, \theta  \mapsto \frac{1}{N}\sum_{i=1}^{N}W(\theta, \xi_i)$ converges locally uniformly to the continuous function $f: \mathbb{R}^m \to \R, \theta \mapsto \E(W(\theta, \xi_1))$, i.e.,
\begin{equation*}
\lim_{N\to\infty} \sup_{|\theta|_2\leq C} \left|\frac{1}{N}\sum_{i=1}^{N}W(\theta, \xi_i) -   \E(W(\theta, \xi_1))\right| = 0 \qquad a.s.
\end{equation*}
Additionally, let $K \subset \R^m$ be compact and define the random variables $v_n := \inf_{x\in K} f_n(x)$. We consider a minimizing sequence of random variables $(x_n)_{n=0}^{\infty}$, given by $f_n(x_n) = \inf_{x\in K} f_n(x) $ and let $v^\star = \inf_{x\in K}f(x)$ and $\mathcal{K}^\star = \{ x\in K: f(x) =v^\star \}$. Then $v_n \to v^\star$ and $\min_{z \in \mathcal{K}^\star}|x_n - z|_2 \to 0 $ a.s. as $n \to \infty$.
\end{lemma}

We can now proceed with the proof of \Cref{thm:MC convergence Yt}.
\begin{proof}[Proof of \Cref{thm:MC convergence Yt}.]
\textbf{Step 1:} we show that $\hat L_N (\theta) \xrightarrow{N \to \infty} L(\theta)$ uniformly on $\Theta_m$.

Fix $m \in \N$. Since $G^\theta_t$ is the output of neural networks with bounded outputs, it is bounded in terms of the input, the weights (bounded by $m$), the time interval $T$, and certain constants depending on the network's architecture and activation functions. Specifically, for all $t \in [0, T]$ and $\theta \in \Theta_m$, we have $|G^\theta_t(\xi_j)| \leq \tilde{B}(1 + U^{\star,(j)})$, where $U^{\star,(j)}$ corresponds to the input $\xi_j$ and $\tilde{B}$ is a constant that may depend on $m$. Therefore, we have
\begin{multline*}
    \phi \left( V_i(\xi_j), G^\theta_{t_i(\xi_j)}(\xi_j), G^\theta_{t_i(\xi_j)-}(\xi_j), \proj{V}(M_i(\xi_j)) \right) \\
    = \left\lvert \proj{V}(M_i(\xi_j)) \odot (V_i(\xi_j) - G^\theta_{t_i(\xi_j)}(\xi_j) ) \right\rvert_2^2 + \left\lvert \proj{V}(M_i(\xi_j)) \odot ( V_i(\xi_j) - G^\theta_{t_i(\xi_j)-}(\xi_j) ) \right\rvert_2^2 \\
    \leq 4 \left(\tilde{B}^2 (1 + U^{\star,(j)})^2 + (V^{(j)}_{t_i^{(j)}})^2 \right).
\end{multline*}
Using the integrability of $U$ and $V$ described in Assumption~\ref{assumption:5} we have
\begin{equation}
\label{equ:dominating bound loss function}
\E\left[\sup_{\theta \in \Theta_m} |W(\theta, \xi_j)|\right] 
\leq   \E\left[\frac{1}{n^{j}(\xi_j)}\sum_{i=1}^{n^{j}(\xi_j)} 4 \left(\tilde{B}^2 (1 + U^{\star,(j)})^2 + (V^{(j)}_{t_i^{(j)}})^2 \right) \right]  < \infty.
\end{equation}
Using Lemma~\ref{lem:properties for MC conv thm}, the function $\theta \mapsto W(\theta, \xi_1)$ is continuous, therefore, Lemma~\ref{lemma:convergencelocallyuniform} implies that a.s for $N \to \infty$, the function 
\begin{equation}\label{equ:unif conv 1}
    \theta  \mapsto \frac{1}{N}\sum_{i=1}^{N}W(\theta, \xi_i) = \hat{L}_N (\theta)
\end{equation}
 converges locally uniformly on $\Theta_m$ to the continuous function 
 \begin{equation}\label{equ:unif conv 2}
    \theta  \mapsto \E[W(\theta, \xi_1)]=L(\theta).
\end{equation}

\textbf{Step 2:} we show that $L(\theta^{\min}_{m,N}) \xrightarrow{N \to \infty} L(\theta^{\min}_{m})$ and $\hat L_N(\theta^{\min}_{m,N}) \xrightarrow{N \to \infty} L(\theta^{\min}_{m})$.

Lemma~\ref{lemma:convergencelocallyuniform} also implies that $\lim_{N \to \infty} \min_{\vartheta \in \Theta^{\min}_m} \left|\theta^{\min}_{m,N} - \vartheta \right|_2= 0$ (a.s.). Thus there exists a sequence $(\hat\theta^{\min}_{m,N})_{N \in \N}$ in $\Theta_m^{\min}$ such that $\lim_{N \to \infty} \lvert \theta^{\min}_{m,N} - \hat\theta^{\min}_{m,N} \rvert_2 = 0$ (a.s.). 
Since the function $\theta \mapsto G_t^\theta$ on $\Theta_m$ is uniformly continuous, we have for any fixed deterministic and bounded $\xi_0$ 
$$\lim_{N \to \infty} \lvert G_{t}^{\theta^{\min}_{m,N}}(\xi_0) - G_{t}^{\hat\theta^{\min}_{m,N}}(\xi_0) \rvert_2 = 0 \text{ a.s. for all  } t \in [0,T].$$ 
Using the fact that $\phi$ is continuous, we get $\lim_{N \to \infty} \lvert W(\theta^{\min}_{m,N}, \xi_0) - W(\hat\theta^{\min}_{m,N}, \xi_0) \rvert = 0$ (a.s.).  We assume now that $\xi_0$ is an i.i.d.\ random variable distributed as the $\xi_j$ defined on a copy $(\Omega_0, \mathbb{F}_0, \mathcal{F}_0, \P_0)$ of the filtered probability space $(\Omega, \mathbb{F}, \mathcal{F}, \P)$. Then for $\P_0$-almost every fixed $w_0$, we have $$\lim_{N \to \infty} \lvert W(\theta^{\min}_{m,N}, \xi_0(w_0)) - W(\hat\theta^{\min}_{m,N}, \xi_0(w_0)) \rvert = 0\quad \text{(a.s.)}.$$ 
Hence we have for $\P$-a.e.\ fixed $\omega \in \Omega$, that $\lvert W(\theta^{\min}_{m,N}\omb, \xi_0) - W(\hat\theta^{\min}_{m,N}\omb, \xi_0) \rvert \to 0$ $\P_0$-a.s.\ as $N \to \infty$.
Using \eqref{equ:dominating bound loss function}, we apply the dominated convergence theorem, leading to
\begin{equation*}
\lim_{N \to \infty} \E_{\xi_0}\left[ \lvert W(\theta^{\min}_{m,N}\omb, \xi_0) - W(\hat\theta^{\min}_{m,N}, \xi_0) \rvert \right]  = 0, 
\end{equation*}
for $\P$-a.e. $\omega \in \Omega$. We note that $\omega$ corresponds to the realisation of the training samples, which lead to the trained parameters, while $\omega_0$ and $\xi_0$ correspond to the validation or test samples (which can also coincide with the training samples, see below).
We know that for every integrable random variable $Z$, we have $0 \leq \lvert \E[Z] \rvert \leq \E[\lvert Z \rvert] $ and noting that $\hat\theta^{\min}_{m,N}\in \Theta_m^{\min}$ we can deduce that for $\P$-almost every fixed $\omega \in \Omega$,
\begin{equation}
\label{equ: MC convergence}
\lim_{N \to \infty} L(\theta^{\min}_{m,N}\omb) = \lim_{N \to \infty} \E_{\xi_0}\left[  W(\theta^{\min}_{m,N}\omb, \xi_0) \right] = \lim_{N \to \infty} \E_{\xi_0}\left[  W(\hat\theta^{\min}_{m,N}\omb, \xi_0) \right] = L(\theta^{\min}_m).
\end{equation}

For $\P$-a.e.\ fixed $\omega$, we have for $\hat L_{\tilde N}$ and $L$ defined through test samples $\tilde{\xi}_j$ on $\Omega_0$ that
\begin{equation}\label{equ: MC convergence 2}
\lvert \hat L_{\tilde N}(\theta^{\min}_{m, N}\omb) -  L(\theta^{\min}_{m}) \rvert \leq \lvert \hat L_{\tilde N}(\theta^{\min}_{m, N}\omb) -  L(\theta^{\min}_{m, N}\omb) \rvert + \lvert L(\theta^{\min}_{m, N}\omb) -  L(\theta^{\min}_{m}) \rvert.
\end{equation}
Using \textbf{step 1} when $\tilde N \to \infty$, the first term on the right side goes to $0$ and with \eqref{equ: MC convergence}, the second term goes to $0$. Now, due to the uniform convergence in \labelcref{equ:unif conv 1,equ:unif conv 2}, we have the same convergence when setting $\tilde N = N$, i.e., when evaluating on the training samples. Therefore, we have shown that $L(\theta^{\min}_{m,N}) \xrightarrow{N \to \infty} L(\theta^{\min}_{m}) \quad \text{and} \quad \hat L_N(\theta^{\min}_{m,N}) \xrightarrow{N \to \infty} L(\theta^{\min}_{m}).$

\textbf{Step 3:} Finally, we show convergence in $d_k$. 

Define $N_0 := 0$ and for each $m \in \mathbb{N}$,
$$
N_m(\omega) := \min\left\{ N \in \mathbb{N} \; \vert \; N > N_{m-1}(\omega), \left| L(\theta^{\min}_{m,N}(\omega)) - L(\theta^{\min}_{m}) \right| \leq \frac{1}{m} \right\},
$$
which is achievable due to \eqref{equ: MC convergence} for almost every $\omega \in \Omega $ under $(\mathbb{P} \times \tilde{\mathbb{P}})$. With Theorem \ref{thm:1}, this choice of $N_m$ implies that for almost every $\omega \in \Omega $,
$$
\left| L(\theta^{\min}_{m, N_m(\omega)}(\omega)) - \Psi(\hat{V}) \right| \leq \frac{1}{m} + \left| L(\theta^{\min}_{m}) - \Psi(\hat{V}) \right| \xrightarrow{m \to \infty} 0.
$$
Thus, by employing similar arguments as in the proof of Theorem \ref{thm:1} (starting from \eqref{equ:proof thm bound 1}), we can demonstrate that
$$
d_k \left( \hat{V} , G^{\theta^{\min}_{m, N_m(\omega)}(\omega)} \right) \leq \frac{c_0(k) \, c_1 \, c_3}{c_2} \left( L(\theta^{\min}_{m, N_m(\omega)}(\omega)) - \Psi(\hat{V}) \right)^{1/2} \xrightarrow{m \to \infty} 0,
$$
for each $1 \leq k \leq K$ and for almost every $\omega \in \Omega $ under $(\mathbb{P} \times \tilde{\mathbb{P}})$.
\end{proof}

\subsection{Different Observation Times for Input and Output Variables}\label{sec:Different observation times for input and output variables}
Our model can also deal with settings, where inputs and outputs are not observed simultaneously. In our notation, this means that certain input or output coordinates are masked via the observation mask $M$.

In some real-world applications, there is zero probability of observing the input and the output at exactly the same time. Then \Cref{assumption:1} would be violated. However, the following corollary shows how this assumption can be relaxed.
\begin{cor}
    When we remove the condition, 
    \begin{equation}
        \P (M_{k,j} =1 ) > 0 \; \text{ for all } \; d_U+1 \leq j \leq d
    \end{equation}
    of \Cref{assumption:1} we still obtain all the results of \Cref{thm:1}, but with the (weaker) pseudo-metric
    \begin{equation*}\label{equ: weaker pseudo metric}
    \tilde d_k (\eta, \xi) = c_0(k)\,  \E\left[ \1_{\{n \geq k\}} \left( | \proj{V}(M_{t_i}) \odot \left( \eta_{t_k-} - \xi_{t_k-} \right) | + | \proj{V}(M_{t_i}) \odot \left(  \eta_{t_k} - \xi_{t_k}  \right) | \right) \right]
\end{equation*}
instead of the pseudo-metric $d_k$.
\end{cor}
\begin{proof}
    Revisiting the proof of \Cref{thm:1}, it is apparent that this condition 
    of \Cref{assumption:1} is only needed to remove $\proj{V}(M_{t_i})$ from the expectation (see \Cref{equ:M split}). 
\end{proof}

\section{Examples Satisfying the Assumptions}\label{sec:Examples}

In the following, we provide examples for which we show that our assumptions of \Cref{sec:Notation and assumptions} are satisfied.
In \Cref{sec:Parameter Filtering} we introduce a class of parameter filtering examples, where it is easy to verify the assumptions.
Additionally, in \Cref{sec:Stochastic Filtering of Brownian Motion} we recall the Brownian motion filtering example \citet[Section 7.6]{krach2022optimal}, which can  be adapted for the input-output setting.
Online classification as a special case of predicting state space models is discussed in \Cref{sec:SSM and classification}.
These examples are chosen sufficiently simple to compute the ground truth conditional expectation $\hat{V}_t = \E\left[ V_t | \mathcal{A}_t \right]$ as a reference. However, the strength of our method is its ability to work for much more general settings, where other methods would fail.

\subsection{Parameter Filtering}\label{sec:Parameter Filtering}

One application area for the IO NJODE model is \emph{parameter filtering} examples. Here we consider a stochastic process $(X^{\alpha_t}_t)_{t\in [0,T]}$, whose distribution depends on unknown random parameters $\alpha = (\alpha_t)_{t \in [0,T]}$ with  values in $\R^{d_V}$. We focus on the special case of parameters with time dependence given by $\alpha_t = \alpha_0 \odot \phi(t)$, where $\alpha_0$ is a random variable in $\R^{d_V}$ and $\phi : [0,T] \to \R^{d_V}$ is a deterministic continuously differentiable function; as before,  $\odot$ is the element-wise (Hadamard) product.  
For any fixed choice of $\alpha$, the distribution of $X^\alpha$ is fixed (and can, for example, be used for sampling). 
In a Bayesian-like approach, we are interested in solving the following problem. Given a (prior) distribution of the parameters $\alpha$, we want to infer (or filter) these parameters only from observations of the process $X^\alpha$. In particular, the input variables for our model are $U=X^\alpha$ and the output variables are $V=\alpha$. Note that \Cref{assumption:3} is automatically satisfied due to continuity of $\alpha$.
We assume that the observation framework is independent of $(X^\alpha, \alpha)$ (\Cref{ass:independence}) and satisfies the associated assumptions (\Cref{assumption:1,assumption:6}). 
Then, the conditional expectation $\hat{V}$ satisfies
\begin{equation*}
    \hat{V}_t = \E\left[ \alpha_t | \mathcal{A}_t \right] = \phi(t) \, \E\left[ \alpha_0 | \mathcal{A}_t \right],
\end{equation*}
hence, the functions $F_j$, are continuously differentiable in their first input-coordinate $t$, since $\phi$ is so. Therefore, \Cref{assumption:4} is satisfied if $\alpha_0$ is square integrable. Indeed, by Jensen's inequality we have
\begin{equation*}
\begin{split}
    \E\left[ \frac{1}{n}\sum_{i=0}^n F_j(t_i, t_i, \tilde U^{\leq t_i})^2\right] 
    &= \E_{\mu_Q \times \P}[\E[ |V_{\bar t}|_2 | \mathcal{A}_{\bar t}]^2] 
    \leq \E_{\mu_Q \times \P}[ |V_{\bar t}|_2^2]
    \leq  \max_{0 \leq t \leq T}( |\phi(t)|_2^2) \, \E[|\alpha_0|_2^2], \\
    \E\left[ \int_0^T | f_j(t, \tau(t), \tilde U^{\leq \tau(t)})  |^2 dt  \right] 
    &= \E\left[ \int_0^T | \phi'(t) \E[\alpha_0 | \mathcal{A}_{\tau(t)}]  |_2^2 dt  \right] 
    \leq \max_{0 \leq t \leq T}(|\phi'(t)|_2^2) \, T \, \E[|\alpha_0|_2^2],
\end{split}
\end{equation*}
where we used \Cref{lem:expectation weighted sum over t_i terms} and that a continuous function is bounded on a compact set.
Furthermore, \Cref{assumption:5} is satisfied for the choice of $X$ by \citet[Lemma 16.1.4]{cohen2015stochastic}.

\subsubsection{Constant Parameters}

In the following, we give explicit examples of parameter filtering settings, where $\alpha$ is constant, hence, $f_j \equiv 0$.

Note that the distributions of the parameters can easily be altered. In the case of the Brownian motion example below, we choose a normal distribution for the drift parameter, since this allows us to derive an explicit formula for the true conditional expectation to be used as reference solution. While this would otherwise not be possible, it is important to note that our model can deal with any other distribution as long as its second moment exists.

\begin{example}\label{exa:brownian motion uncertain drift}
    We consider a scaled Brownian motion with uncertain drift, i.e., a model of the form
\begin{equation}\label{equ:model equ exa BM w drift}
\begin{split}
    X_t & = x_0 + \mu t + \sigma W_t, \\
    \mu &\sim \mathcal{N}(a, b^2),
\end{split}
\end{equation}
where $\sigma, b > 0$, $a, x_0 \in \R$ are fixed parameters and $W$ is a standard Brownian motion (independent of $\mu$).
The goal is to filter out the parameter $\mu$ from observations of $X$, hence, $U=X$ is the input variable and $V=\mu$ the output variable of the input-output NJODE model.
Clearly, integrability of $\mu$ is satisfied, hence the derivations above yield that the assumptions are satisfied.

We follow the approach in \citet[Section~7.6]{krach2022optimal} to compute the true conditional expectation $\hat{V}_t = \E[\mu | \mathcal{A}_t]$ in this setting. For observation times $t_1, \dotsc, t_k$ of the process $X$, let $\tilde W_i := W_{t_i} - W_{t_{i-1}}$ for $1 \leq i \leq k$.
Then 
\begin{equation*}
    v := (\tilde W_1, \dotsc, \tilde W_k, \mu)^\top \sim N((0,\dotsc, 0, a)^\top, \Sigma)
\end{equation*}
is multivariate normally distributed, where 
\begin{equation*}
    \Sigma := \operatorname{diag}(t_1-t_0, \dotsc, t_k-t_{k-1}, b^2).
\end{equation*}
For $I_l \in \R^{k \times k}$ the lower triangular matrix with $1$ entries, we can express the joint vector of (centered) observations and the target variable as multivariate normal distribution
\begin{equation*}
    (X_{t_1 } - x_0, \dotsc, X_{t_k}-x_0, \mu)^\top = \Gamma v \sim N((t_1 a, \dotsc, t_k a, a)^\top, \Gamma \Sigma \Gamma^\top),
\end{equation*}
where
\begin{equation*}
    \Gamma := \begin{pmatrix}
        \sigma I_l & \begin{matrix}
            t_1 \\ \vdots \\ t_k
        \end{matrix} \\
        0 & 1
    \end{pmatrix}, \quad \text{and} \quad 
    \Gamma \Sigma \Gamma^\top =: \tilde \Sigma = \begin{pmatrix}
        \tilde \Sigma_{11} & \tilde \Sigma_{12} \\
        \tilde \Sigma_{21} & \tilde \Sigma_{22}
    \end{pmatrix},
\end{equation*}
with $\tilde \Sigma_{11} \in \R^{k \times k }$, $\tilde \Sigma_{12} = \tilde \Sigma_{21}^\top \in \R^{k \times 1}$ and   $\tilde \Sigma_{22} = b^2 \in \R$.
Then, \citet[Proposition~3.13]{Eaton2007Multi} implies that the conditional distribution of $(\mu \, |\, X_{t_1 } - x_0, \dotsc, X_{t_k}-x_0)$ is again normal with mean $\hat \mu := a +  \tilde \Sigma_{21} \tilde \Sigma_{11}^{-1} (X_{t_1 } - x_0 - t_1 a, \dotsc, X_{t_k}-x_0 - t_k a)^\top $ and variance $\hat \Sigma :=  \tilde \Sigma_{22} -  \tilde \Sigma_{21} \tilde \Sigma_{11}^{-1} \tilde \Sigma_{12}$.
In particular, since $\mu_t = \mu $ is constant in time and $x_0$ is deterministic, we have for $t_k \leq t \leq t_{k+1}$
\begin{equation*}
    \hat{V}_t = \E[\mu_t | \mathcal{A}_t] = \E[\mu | X_{t_1 }, \dotsc, X_{t_k}] = \hat{\mu}. 
\end{equation*}
\end{example}

In the following example, no explicit formula for the conditional expectation can be derived. Therefore, we resort to two alternative approaches to derive reference solutions. First, we use a particle filter, which converges to the optimal solution as the number of particles increases. We note that in contrast to our model, full knowledge of all distributions is necessary to apply the particle filter. The second approach uses a typical financial estimator, which does not need such knowledge.

\begin{example}\label{exa:geometric Brownian motion uncertain params}
    Here we consider a geometric Brownian motion with random drift and volatility parameter, starting from $X_0 = x_0$
    \begin{equation}\label{exa:geometric Brownian motion}
    \begin{split}
        dX_t & = \mu X_t dt + \sigma X_t dW_t, \\
        \mu &\sim \mathcal{N}(a, b^2), \\
        \sigma &\sim \operatorname{Unif}([\sigma_{\min}, \sigma_{\max} ]),
    \end{split}
    \end{equation}
    where $a \in \R$, $b > 0$ and $0 < \sigma_{\min} < \sigma_{\max}$ are fixed and $W$ is a standard Brownian motion.
    The goal is to filter the parameters $V=(\mu, \sigma)$ from observations of $U=X$. The integrability of $\mu, \sigma$ implies that \Cref{assumption:4,assumption:5} are satisfied (cf.\ \Cref{sec:Parameter Filtering}). 
    In this case, there is no explicit form for the true conditional expectation $\hat V_t = \E[(\mu,\sigma) | \mathcal{A}_t]$. 
    Instead we suggest two alternative approaches. The first one approximates the true conditional expectation via particle filtering (see e.g.\ \citet{djuric2003particle}) with sequential importance sampling (SIS), which is composed of the initialization, prediction and correction step. The second one is an alternative reference parameter estimation often used in finance.
    
    For particle filtering, the \emph{initialization} of the parameters $\mu, \sigma$ is done via their joint distribution 
    $$\pi_0(\mu,\sigma) = \varphi(\mu;a,b)\dfrac{1}{\sigma_{\max}-\sigma_{\min}}\ind{\sigma \in [\sigma_{\min}, \sigma_{\max} ]},$$
    where $\varphi(\mu;a,b)$ is the probability density function of the normal distribution $\mathcal{N}(a, b^2)$. 
    In particular, $N_p$ particles $(\mu^{(j)}, \sigma^{(j)})$ for $1\leq j\leq N_p$ are sampled from this distribution and each particle is given the equal weight $w_0^{(j)} = \tfrac{1}{N_p}$.
    After initialization, the prediction step is applied repeatedly until a new observation becomes available, whence the correction step is triggered.
    For the \emph{prediction step} it is enough to note that $\mu$ and $\sigma$ are constant parameters, hence, the predicted distribution at any given time $t$ remains the same as the posterior distribution from the previous step. In the \emph{correction step}, the distribution is updated using Bayes' theorem with the new observation $X_{t_i}=x_i$ at time $t_i$. The transition density $p(x_i|x_{i-1},\mu,\sigma)$ for the system \eqref{exa:geometric Brownian motion} is given by
    \begin{equation*}
        p(x_i|x_{i-1},\mu,\sigma) = \dfrac{1}{x_i \sigma \sqrt{2\pi (t_i - t_{i-1})}}\exp{\left(- \dfrac{\left[\ln \left(\frac{x_i}{x_{i-1}}\right) - \left(\mu-\frac{\sigma^2}{2}\right)(t_i - t_{i-1}) \right]^2}{2\sigma^2 (t_i - t_{i-1})}\right)}.
    \end{equation*}
    Hence, the posterior distribution is updated as
    $\pi_{t_i}(\mu,\sigma) \propto  p(x_i|x_{i-1},\mu,\sigma)\pi_{t_{i-1}}(\mu,\sigma),$
    leading to the updated weights
    \begin{align*}
        \hat{w}^{(j)}_i &= w^{(j)}_{i-1} p(x_i^{(j)}|x^{(j)}_{i-1},\mu^{(j)},\sigma^{(j)}), &
        {w}^{(j)}_i &= \frac{\hat{w}^{(j)}_i}{\sum_{j=1}^{N_p} \hat{w}^{(j)}_i} .
    \end{align*}
    Then, the conditional expectations are given by
    \begin{align*}
            \hat \mu_{t_i} &= \displaystyle\int \mu \pi_{t_i}(\mu,\sigma) d(\mu ,\sigma) \approx \sum_{j=1}^{N_p} \mu^{(j)} w^{(j)}_i, &
            \hat \sigma_{t_i} &= \displaystyle \int \sigma \pi_{t_i}(\mu,\sigma) d(\mu ,\sigma) \approx \sum_{j=1}^{N_p} \sigma^{(j)} w^{(j)}_i .
    \end{align*}
    Compared to our method, which is fully data-driven, the particle filter needs knowledge of the involved distributions.

    For the second approach, we can use (a generalization for irregular time steps of) the standard financial estimator  of $\mu$ and $\sigma$ as reference \citep{khaled2010estimation}. In particular, we have $X_{t_{i+1}} = X_{t_i} \exp{\left((\mu - \sigma^2/2) \Delta t_i + \sigma \Delta W_{i}\right)}$, where $\Delta t_i = t_{i+1} - t_i$ and $\Delta W_i = W_{t_{i+1}}- W_{t_i}$. 
    If we assume that $\Delta t_i \geq c$ for some constant $c > 0$ and for all $i \in \N$, then the law of large numbers \citep[Theorem~1]{csorgHo1983strong} implies for $r_i := \log(X_{t_{i+1}} / X_{t_{i}})$,
    \begin{equation*}
        \hat m_N := \frac{1}{N} \sum_{i=1}^N  \frac{r_i}{\Delta t_i} \xrightarrow[N \to \infty]{\P-a.s.} \mu - \sigma^2/2.
    \end{equation*}
    Moreover, $r_i / \sqrt{\Delta t_i} - (\mu - \sigma^2/2) \sqrt{\Delta t_i} \sim N(0, \sigma^2)$, therefore,
    \begin{equation*}
        \hat \sigma_N^2 := \frac{1}{N} \sum_{i=1}^N (r_i / \sqrt{\Delta t_i} - \hat m_N \sqrt{\Delta t_i})^2 \xrightarrow[N \to \infty]{\P-a.s.} \sigma^2.
    \end{equation*}
    This yields the estimators for drift and diffusion $\hat \mu_N := \hat m_N + \sigma^2_N/2$ and $\sqrt{\hat \sigma^2_N}$, respectively.
    It is important to keep in mind that these estimators are not an approximation of the conditional expectation and will therefore not optimize the loss function \eqref{equ:Phi}.
\end{example}

In the previous examples, the increments of the underlying process have a normal or log-normal distribution, which are relatively easy special cases. In the case of normally distributed increments, the standard Kalman filter \citep[which needs knowledge of the underlying distribution, as the particle filter;][]{kalman1960new} is known to be the optimal filter, retrieving the true conditional expectation (see \Cref{sec:Kalman Filter}). 

To show that our model can also deal with nonstandard cases, we consider a Cox--Ingersoll--Ross (CIR) process in the following example, where the distribution of the increments is much more complex. Again, we use a particle filter approximating the true conditional expectation as reference solution.

\begin{example}\label{exa:CIR constant params}
    We consider a Cox--Ingersoll--Ross (CIR) process with random speed, mean and volatility parameters starting from $X_0=x_0$, given by
    \begin{equation}\label{equ: CIR}
    \begin{split}
        dX_t & = a(b-X_t) dt + \sigma \sqrt{X_t} dW_t, \\
        a &\sim \operatorname{Unif}([a_{\min}, a_{\max} ]), \\
        b &\sim \operatorname{Unif}([b_{\min}, b_{\max} ]), \\
        \sigma &\sim \operatorname{Unif}([\sigma_{\min}, \sigma_{\max} ]),
    \end{split}
    \end{equation}
    where $W$ is a standard Brownian motion and $0 < a_{\min} < a_{\max}$, and similarly for $b$ and $\sigma$, are fixed.
    The goal is to filter $V=(a,b,\sigma)$ from observations of $U=X$. \Cref{assumption:4,assumption:5} are satisfied due to square-integrability of all parameter distributions (cf.\ \Cref{sec:Parameter Filtering}). 
    
    We use the same particle filter approach for deriving a reference solution as in \Cref{exa:geometric Brownian motion uncertain params}. 
    The initial distribution of the parameters is
    \begin{equation*}
        \pi_0(a,b,\sigma) = \mathcal{U}(a; a_{\min}, a_{\max}) \, \mathcal{U}(b; b_{\min}, b_{\max}) \, \mathcal{U}(\sigma; \sigma_{\min}, \sigma_{\max}),
    \end{equation*}
    where $ \mathcal{U}(x; u_1, u_2) = (u_2 - u_1)^{-1} \ind{ \{ x \in [u_1, u_2] \} }$ is the probability density function of the uniform distribution. Moreover, the transition density from $X_{t_{i-1}}=x_{i-1}$ to $X_{t_i}=x_i$ with $\Delta_i = t_i - t_{i-1} $ is given by \citep{CIR}
    \begin{equation}\label{equ: CIR transition density}
        p(x_i | x_{i-1}, a,b,\sigma) = c \, e^{-u-v} \left( \frac{v}{u} \right)^{q/2} \, I_q(2 \sqrt{uv}),
    \end{equation}
    where 
    \begin{equation*}
        c = \frac{2a}{(1-e^{-a \Delta_i}) \sigma^2}, \; q=\frac{2ab - \sigma^2}{\sigma^2}, \; u = c x_{i-1} e^{-a \Delta_i}, \, v = c x_i,
    \end{equation*}
    and $I_q$ is a modified Bessel function of the first kind of order $q$.
\end{example}

\subsubsection{Time-Dependent Parameters}
Finally, we extend the last example by making one of the parameters time-dependent.

\begin{example}\label{exa:CIR time-dep params}
    We consider the Cox--Ingersoll--Ross (CIR) process as defined in \eqref{equ: CIR}, but with time dependent mean parameter 
    \begin{equation*}
        b_t = b_0 (1 + \sin(w \, t) /2), \quad b_0 \sim \operatorname{Unif}([b_{\min}, b_{\max} ]),
    \end{equation*}
    where $w \in R$ is a fixed. 
    For the particle filter, we approximate the transition density by $p(x_i | x_{i-1}, a,b_{t_{i-1}},\sigma)$ as in \eqref{equ: CIR transition density}, i.e., by setting $b$ to the value at the last observation.  
\end{example}

\subsection{Stochastic Filtering of Brownian Motion}\label{sec:Stochastic Filtering of Brownian Motion}
We briefly recall the example of \citet[Section 7.6]{krach2022optimal}, adapted to the context of input-output models, that is, the observation process serves as input, while the signal process is the target output. 
Let $X, W$ be two i.i.d. 1-dimensional Brownian motions, let $\alpha \in \R$, then $X$ is the unobserved signal (output) process and $Y := \alpha X + W$ is the observation (input) process.
We follow \citet[Section 7.6]{krach2022optimal} to derive the analytic expressions of the conditional expectations.
Let $t_i$ be the observation times and define 
$ v := ( Y_{t_1}, \dots, Y_{t_k}, X_{t_k} )^\top \sim N(0,\mathbf\Sigma)$, where $\mathbf\Sigma \in \R^{(k+1)\times (k+1)}$ with entries
$$\mathbf{\Sigma}_{i,j} = \begin{cases}
    (\alpha^2 + 1) t_{\min(i,j)}, & \text{if } i,j \leq k\\
    \alpha t_{\min(i,j)}, & \text{if } (i \leq k,  j = k+1) \text{ or } (i = k+1,  j \leq k) \\
    t_k , & \text{if } i,j = k+1\\
\end{cases}. $$
If we write
$$  \mathbf\Sigma =: \begin{pmatrix}
 \Sigma_{11} &  \Sigma_{12} \\
 \Sigma_{21} &  \Sigma_{22}
\end{pmatrix},$$
where $ \Sigma_{11} \in \R^{k\times k }$, $ \Sigma_{12} =  \Sigma_{21}^\top \in \R^{k \times 1}$ and   $ \Sigma_{22} = \operatorname{Var}(X_{t_k}) = t_k \in \R^{1\times 1 }$, then the conditional distribution of $(X_{t_k} \, | \, Y_{t_1}, \dotsc, Y_{t_{k}})$ is again normal with mean $\hat \mu :=  \Sigma_{21}  \Sigma_{11}^{-1} (Y_{t_1}, \dotsc, Y_{t_{k}})^\top $ and variance $\hat \Sigma :=   \Sigma_{22} -   \Sigma_{21}  \Sigma_{11}^{-1}  \Sigma_{12}$ \citep[Proposition~3.13]{Eaton2007Multi}. 
In particular, we have for any $s>0$ that  $\E[ X_{t_k + s} | \mathcal{A}_{t_k} ] = \hat \mu$.

\subsection{Finite State Space Processes and Online Classification}\label{sec:SSM and classification}
The Input-Output NJODE model can be applied to finite state space processes (SSP) and as a special case to online classification tasks.
Let $(S_t)_{t \in [0,T]}$ be a finite SSP, i.e., a stochastic process taking values in a finite state space $\mathcal{S}:=\{ s_1, \dotsc s_{N_S} \}$, where $N_S \in \N_{\geq 2}$ is the number of states, and assume we want to predict the state of the model given (irregular and incomplete) observation of an (arbitrary) feature process $U$ that satisfies \Cref{assumption:5}. 
Choosing the target output process $V$ to be coordinate-wise given by the indicator functions for the different states, $V_{t,j} := \ind{S_t = s_j} $, yields the conditional expectation
\begin{equation*}
    \hat{V}_t = \E[V_t | \mathcal{A}_t] = (\P[S_t = s_j | \mathcal{A}_t] )_{1 \leq j \leq N_S}.
\end{equation*}
Hence, training our model on the input-output system $(U,V)$ leads to an approximation of the state probabilities $\hat{V}$, if $\hat V$ satisfies \Cref{assumption:4}. Note that $V$ automatically satisfies the integrability \cref{assumption:5}, since it is bounded by $1$, and for the same reason \Cref{equ:assumption4 bound} of \Cref{assumption:5} can be reduced to 
\begin{equation*}
\E\left[ \int_0^T | f_j(t, \tau(t), \tilde U^{\leq \tau(t)})  |^2 dt  \right] < \infty.
\end{equation*}

Online classification is the special case where the states correspond to the different classes. Clearly, we can always reduce the dimension of $V$ by $1$ there, since $|V_t|_1 = 1$. The following example shows an easy case of this online classification problem based on a Brownian motion.

\begin{example}\label{exa:BM classification}
    Let $U := W$ be a standard Brownian motion, $\alpha \in \R$ and define $S_t := \ind{W_t \geq \alpha}$ to be the state space model indicating whether the Brownian motion is larger than $\alpha$. Hence, this is a classification task with $2$ classes such that it is enough to consider the 1-dimensional $V := S$. 
    The conditional expectation of $V$ can be computed analytically as
    \begin{equation*}
    \begin{split}
        \hat{V}_t 
        = \P\left[ W_t \geq \alpha | \mathcal{A}_t \right] 
        = \P\left[ \frac{(W_t - W_{\tau(t)})}{\sqrt{t - \tau(t)}} \geq  \frac{\alpha - W_{\tau(t)}}{\sqrt{t - \tau(t)}} \middle| W_{\tau(t)} \right] = 1 - \Phi\left(\frac{\alpha - w}{\sqrt{t - \tau(t)}}\middle)\right|_{w = W_{\tau(t)}},
    \end{split}
    \end{equation*}
    using the properties of a Brownian motion and $\Phi$ to be the cumulative distribution function of the standard normal distribution. Hence, it is easy to verify with the Leibniz integral rule that \Cref{assumption:4} is satisfied.
\end{example}

\section{Experiments}\label{sec:Experiments}

We test the applicability of the framework on synthetic datasets with reference solutions, which were discussed in \Cref{sec:Examples}.

To measure the quality of the trained models in the synthetic examples we use the \emph{validation loss} or the  \emph{test loss}, i.e., the loss \eqref{equ:Phi} on the validation or test dataset (which should be minimized by the true conditional expectation) and the \emph{evaluation metric} \citep[see][Section~8]{krach2022optimal}, which computes the distance between the model's prediction and the true conditional expectation on a time grid averaged over all test samples. We report the results for the best early stopped model, selected based on the validation loss. 

The code for running the experiments is available at \code\ and further details about the implementation of the experiments can be found in \Cref{sec:Experimental Details}.

\subsection{Scaled Brownian Motion with Uncertain Drift}\label{sec:Scaled Brownian Motion with Drift}
We consider a scaled Brownian motion with uncertain drift, as described in \Cref{exa:brownian motion uncertain drift}. 
In particular, we set $x_0 = 0, \sigma=0.2, a = 0.05$ and $b=0.1$.
The model gets the (irregular) observations of $X$ as input and tries to filter out the randomly sampled drift coefficient $\mu$. Additionally, we let the model predict the conditional expectation of $\mu^2$, which allows us to derive the conditional variance \citep[see][Section 5]{krach2022optimal}
\begin{equation*}
    \operatorname{Var}[\mu \, | \, \mathcal{A}_t] = \E[\mu^2 \, | \, \mathcal{A}_t] - \E[\mu \, | \, \mathcal{A}_t]^2 ,
\end{equation*}
estimating the aleatoric\footnote{If one trains a model on finitely many paths to estimate $\operatorname{Var}[\mu \, | \, \mathcal{A}_t]$, the inaccuracies of this model additionally introduce epistemic uncertainty which is not captured by $\operatorname{Var}[\mu \, | \, \mathcal{A}_t]$. $\operatorname{Var}[\mu \, | \, \mathcal{A}_t]$ can either be seen as aleatoric or epistemic uncertainty depending on the point of view. If we consider each partially observed path as one sample, then $\operatorname{Var}[\mu \, | \, \mathcal{A}_t]$ cannot be reduced by collecting more training samples nor by improving our NJODE, but it can be reduced by modifying the features $\mathcal{A}_t$ inputted to the NJODE. From this perspective, $\operatorname{Var}[\mu \, | \, \mathcal{A}_t]$ is aleatoric uncertainty according to \citet{CLEARcalibratedlearningepistemic}. However, from another perspective, where we consider each observation of one path as a sample, we can also interpret $\operatorname{Var}[\mu \, | \, \mathcal{A}_t]$ as epistemic uncertainty.} uncertainty.
In \Cref{fig:BMwUD} we show one test sample of the dataset, where we plot the conditional mean plus/minus the conditional standard deviation (approximately corresponding to a $68\%$ normal confidence interval). We see that the model correctly learns to predict the mean $a$ of the distribution of $\mu$ at $t=0$, where no additional information is available, and also updates its predictions correctly when new observations of $X$ become available. 
The conditional variance is slightly underestimated by the model and less accurate than the conditional expectation. The lower accuracy can be explained by the fact that the values of $\mu^2$ are mostly smaller in magnitude than the values of $\mu$, implying that they are less important in the loss (due to their smaller scaling). Additionally, by plotting the standard deviation, i.e., the square root of the variance, small errors of the (small) variance predictions are scaled up. 
The underestimation of the variance is further discussed below in \Cref{rem:Underestimating the Variance}.
The evaluation metric (which only compares the conditional mean, not the variance) is $1.7 \cdot 10^{-6}$, where the validation loss (which uses the predictions of $\mu$ and $\mu^2$) of the model has the minimal value (among trained epochs) of $0.01848$, close to the optimal validation loss $0.01846$ of the true conditional expectation.
\begin{figure}
\centering
\includegraphics[width=0.6\textwidth]{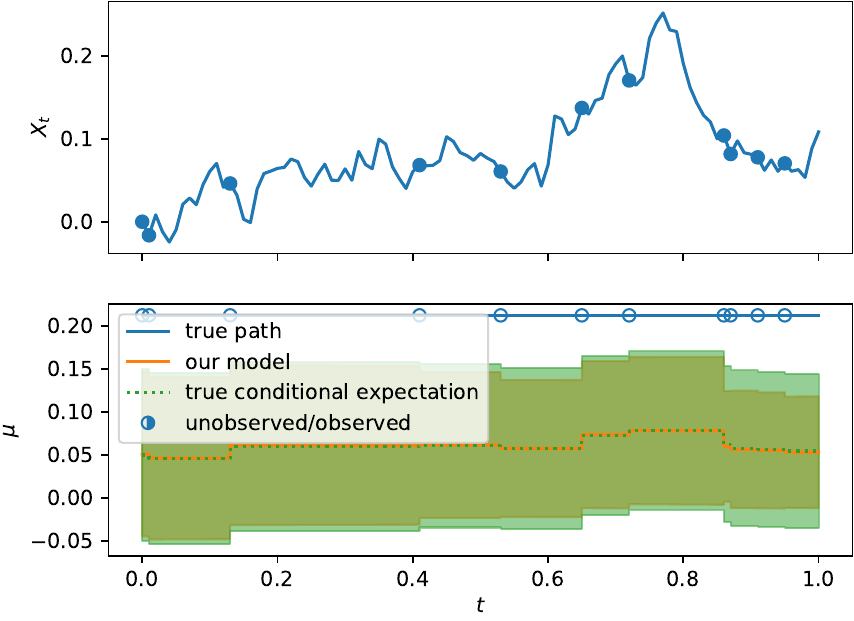}
\caption{Predicted and true conditional expectation $\pm$ standard deviation of the uncertain drift on a test sample of the scaled Brownian motion with random drift. The input coordinate $U=X$ is plotted on top and the output coordinate $V=\mu$ on the bottom. The PD-NJODE replicates well the true conditional expectation of the drift.}
\label{fig:BMwUD}
\end{figure}

\begin{rem}[Underestimating the Variance]\label{rem:Underestimating the Variance}
    \citet[Section~5.1]{krach2022optimal} shows that this estimator of the conditional variance is asymptotically unbiased. However, if we train our model only on a finite number of paths, our model can be slightly biased. In the case of conditional variance, our model tends to underestimate the conditional variance for the following reasons: Even if our model gave unbiased but noisy estimations of $\E[\mu^2 \, | \, \mathcal{A}_t]$ and $\E[\mu \, | \, \mathcal{A}_t]$, we would tend to overestimate $\E[\mu \, | \, \mathcal{A}_t]^2$ due to the strong convexity of the function $x \mapsto x^2$. An overestimation of $\E[\mu \, | \, \mathcal{A}_t]^2$ together with an unbiased estimator of $\E[\mu^2 \, | \, \mathcal{A}_t]$ results in underestimating $\operatorname{Var}[\mu \, | \, \mathcal{A}_t] = \E[\mu^2 \, | \, \mathcal{A}_t] - \E[\mu \, | \, \mathcal{A}_t]^2$. This bias only vanishes as the number of training paths tends to infinity.
\end{rem}

\subsection{Geometric Brownian Motion with Uncertain Parameters}\label{sec:Geometric Brownian Motion with Uncertain Parameters}
We consider a Black--Scholes model (geometric Brownian motion) with uncertain drift and diffusion parameters, as outlined in \Cref{exa:geometric Brownian motion uncertain params}. As already mentioned in the introduction, treating the firm value as (unobserved) geometric Brownian motion following the approach in \cite{merton1974pricing} and \cite{duffie2001term}. 
In particular, we use $x_0 = 1, a = 0.05, b=0.1, \sigma_{\min} = 0.05, \sigma_{\max}=0.5$ and note that $\Delta t_i \geq c := 0.01$ due to our sampling regime (cf.\ \Cref{sec:Details for Synthetic Datasets}). The parameters $\mu, \sigma$ should be filtered out from the irregular observations of $X$.

\begin{figure}[!tb]
\centering
\includegraphics[width=0.49\textwidth]{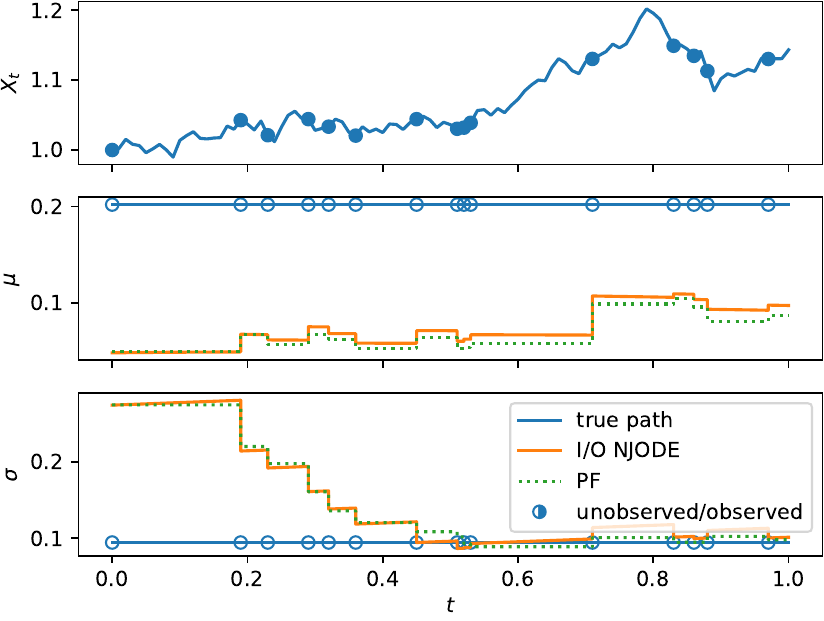}  
\includegraphics[width=0.49\textwidth]{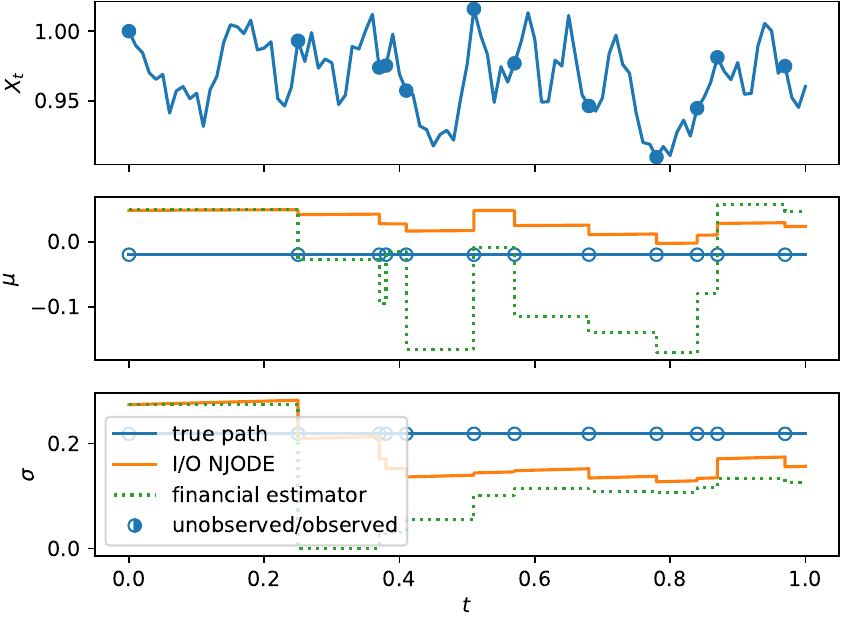}
\caption{Predicted conditional expectation of the uncertain parameters of a geometric Brownian motion on a test sample, as well as a particle filter estimation (left) and a typical financial estimator (right) for those parameters. The input coordinate $U=X$ is plotted on top, and the output coordinates $V=(\mu,\sigma)$ below. As a Bayesian-like estimator, our model's prediction and the particle filter have much less variance than the financial estimator.}
\label{fig:GBM-UP}
\end{figure}

In \Cref{fig:GBM-UP} we show two samples of the test set, where we see that our model correctly learns to predict the mean of the distributions of $\mu$ and $\sigma$ at $t=0$. 
Moreover, our model's prediction closely follows the particle filter's estimation (with $1000$ particles) of the parameters. The particle filter has a slightly better minimal validation loss of $0.0633$ compared to $0.0652$ of our model. 
However, it is important to keep in mind that the particle filter uses the knowledge of the distributions of $\mu, \sigma$ and $X$, while our model only uses the training data.
The evaluation metric on the test set at minimal validation loss\footnote{We note that the evaluation metric depends on the particle filter estimation, which is random, hence, this number varies slightly between evaluations.} is $2.59 \cdot 10^{-4}$.

Compared to the financial reference log-return estimator of the parameters, our model's prediction has less variance, since it is an approximation of the conditional expectation, which is a Bayesian-type estimator. Our model clearly outperforms the reference estimator in terms of the loss function \eqref{equ:Phi}, with a minimal validation loss of $0.0652$ compared to $2.4674$. Importantly, the reference estimator is not an approximation of the conditional expectation, which, in turn, is the unique optimizer of the loss function. Therefore, this outperformance was expected.

\begin{figure}
\centering
\includegraphics[width=0.49\textwidth]{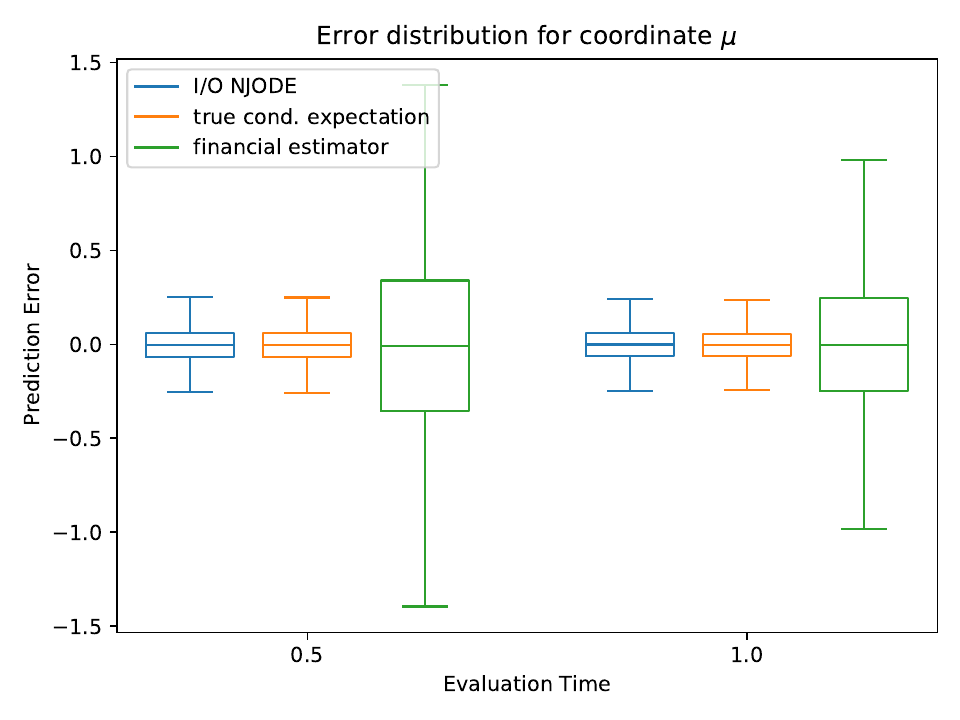}
\includegraphics[width=0.49\textwidth]{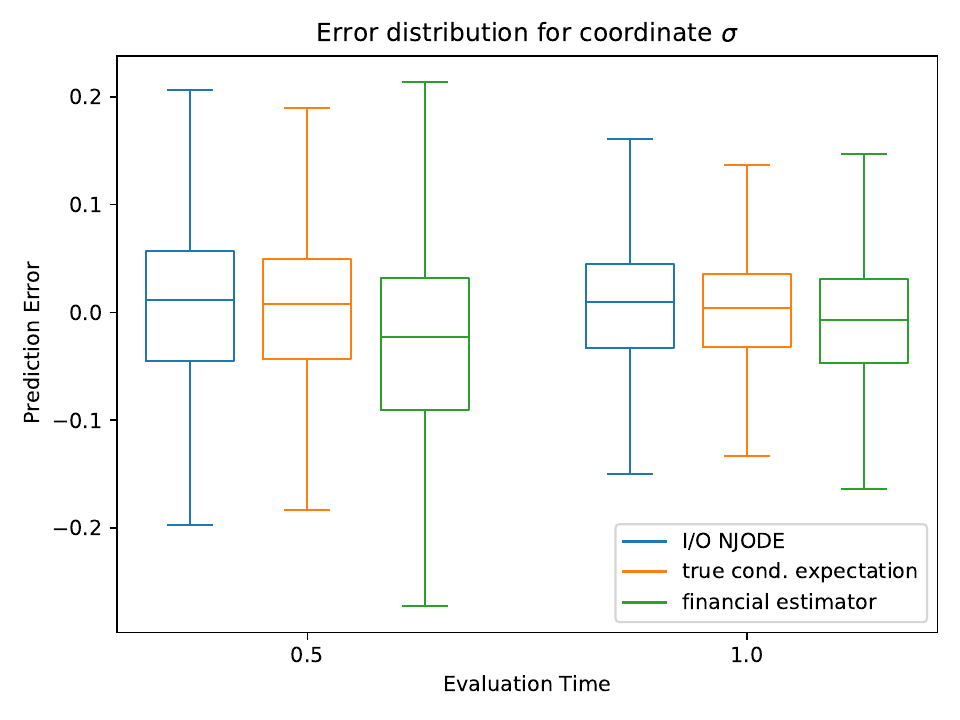}
\caption{Distribution of prediction error of our model and reference methods at evaluation times $t=T/2=0.5$ and $t=T=1$.}
\label{fig:prediction error BS UP}
\end{figure}

We compute the prediction error as the difference between the model prediction and the true value of $\mu$ and $\sigma$, respectively. In \Cref{fig:prediction error BS UP} we compare the distribution of the prediction error for our model, the particle filter approximation of the true conditional expectation and the financial estimator.
The NJODE and the particle filter (PF) are on par predicting the drift, with only a slight difference in the error distribution between $T/2$ and $T$. The financial estimator has much larger variance ($0.7$ compared to $0.1$ for NJODE and PF), which decreases slightly from $T/2$ to $T$ (to $0.5$). 
For the estimation of $\sigma$, the particle filter has slightly smaller variance ($0.077$) then the NJODE model ($0.080$), while the financial estimator (FE) has larger variance ($0.137$) again. From $T/2$ to $T$, the variance of all estimators decreases (to $0.061$ for NJODE, $0.057$ for PF, $0.085$ for FE), as expected, since the volatility estimation becomes more precise with the more observations available.

\subsubsection{Empirical Convergence Study}
The financial estimator should become more precise in estimating $\sigma$ when the number of observations increases. Moreover, its estimation of $\mu$ should improve as the time horizon $T$ increases. In addition, the true conditional expectation should become closer to the true values of $\mu$ and $\sigma$, implying that the error distributions of the NJODE and the particle filter should improve.
We empirically test this with a convergence study based on the  Black-Scholes data model with 
\begin{itemize}
    \item the same parameters as before except for varying $T \in \{ 0.1, 1, 2, 5 \}$, always keeping $100$ grid points leading to increased step sizes for larger $T$;
    \item the same parameters as before, but sampled on a grid with $1000$ instead of $100$ grid points, with varying observation probability $p \in \{ 0.01, 0.025, 0.05, 0.1 \}$ leading to an expected number of observations of $10$, $25$, $50$, $100$ within $[0,T]$, respectively. For comparison, with the original $100$ grid points and $p=0.1$, we had $10$ observations on average.
\end{itemize}

In \Cref{tab:results CS1} we show mean and standard deviations of the prediction errors $\operatorname{err}_\mu = \hat{\mu} - \mu$, where $\hat{\mu}$ is the model's estimation of the true parameter $\mu$ (which varies across samples), and similar for $\operatorname{err}_\sigma$. 
Note that  for a mean error of $0$, the standard deviation of the error coincides with the (root mean square error) RMSE.
We see in \Cref{tab:results CS1} that all methods produce unbiased estimates for $\mu$ and $\sigma$ when evaluated at $t=T$ for varying $T$. The RMSE of the NJODE is very close to the one of the particle filter, which approximates the true conditional expectation. For $\mu$, the RMSE decreases slightly for growing $T$. For the financial estimator, the RMSE is much larger (by a factor of 15 for $T=0.1$), however, it also decreases faster for growing $T$ (for $T=5$ the factor is approximately $3$). The drift is known to be difficult to estimate with the financial estimator, but, as expected, increasing the time horizon substantially improves the prediction. In contrast to this, the NJODE and the PF profit in a Bayesian way from their prior knowledge (in the case of NJODE learned from the training samples and in the case of PF given through knowledge of the initial distribution), while a larger time horizon $T$ only slightly improves their predictions. The RMSE of $\sigma$ stays approximately constant with $T$, since the number of observation does not change.

The results for varying observation probabilities $p$ in \Cref{tab:results CS2} are shown when evaluating the errors at $t = T/2 = 1/2$. The predictions of the NJODE model tend to get worse close to the time horizon $T$, since the model has fewer training samples there\footnote{
The model's prediction at time $t$ only adds (indirectly) to the loss if there is an observation at $s \geq t$. If there is no observation afterwards, the loss is independent of this prediction, hence, minimizing the loss does not necessarily lead to an improvement of this prediction. For uniformly distributed observation times, as in this example, the probability of having an observation after $t$ decreases with increasing $t$, therefore,  the models predictions tend to degrade close to $T$.
}. To avoid this (which is a side effect of the limited amount of training data), we evaluate at $T/2$. 
Again, all methods produce (nearly) unbiased estimates for $\mu$ and $\sigma$. For the NJODE and the particle filter, the RMSE of $\mu$ stays approximately constant with $p$, while it improves for the financial estimator. This is a side-effect of the improved estimates of $\sigma$, since they enter into the computation of $\hat{\mu}$ (cf.\ \Cref{exa:geometric Brownian motion uncertain params}).
The RMSE of $\sigma$ improves for all methods with $p$. For the NJODE the RMSE is slightly larger than for the PF. The relative difference between their RMSEs slightly grows, as the errors become smaller\footnote{
This can be explained by the fact that the NJODE model is jointly trained to predict $\mu$ and $\sigma$. As the predictions for $\sigma$ become better, they also become less significant in the loss, which becomes increasingly dominated by the prediction error of $\mu$. Therefore, the model's training focus shifts towards $\mu$. Training the model to only predict $\sigma$ would avoid this issue.
}. The financial estimator's RMSE of $\sigma$ is much lager for small $p$ (by a factor of $2$ for $p=0.01$), but also decreases faster with growing $p$ and eventually is of the same magnitude as for the other models at $p=0.1$. This shows that a large enough amount of observations can make up for the prior knowledge, hence, this prior knowledge is of less importance for the prediction of $\sigma$ than for the prediction of $\mu$. 

In all experiments, the evaluation metric on the test set is between $3.8 \cdot 10^{-4}$ and $6.5 \cdot 10^{-4}$.

\begin{table}[tb]
\caption{Mean $\pm$ standard deviation of the signed prediction errors for $\mu$ and $\sigma$ evaluated at $t=T$ for varying values of $T$. This table showcases the bias and variance of each method for the predictions of both, $\mu$ and $\sigma$.
}
\label{tab:results CS1}
\begin{center}
\resizebox{\textwidth}{!}{
\begin{tabular}{l | r r | r r | r r}
\toprule
 &\multicolumn{2}{c|}{IO NJODE}&\multicolumn{2}{c|}{Particle Filter}&\multicolumn{2}{c}{Financial Estimator}\\
\cmidrule{1-7}
 $T$ & $\operatorname{err}_{\mu}$ & $\operatorname{err}_{\sigma}$ & $\operatorname{err}_{\mu}$ & $\operatorname{err}_{\sigma}$ & $\operatorname{err}_{\mu} $ & $\operatorname{err}_{\sigma}$ \\
 \midrule[\heavyrulewidth]
 0.1 & $0.00 \pm 0.099$ & $0.00 \pm 0.07$ & $0.00 \pm 0.098$ & $0.00 \pm 0.06$ & $-0.01 \pm 1.605$ & $0.00 \pm 0.09$ \\
 1 & $0.00 \pm 0.092$ & $0.00 \pm 0.06$ & $0.00 \pm 0.091$ & $0.00 \pm 0.06$ & $0.00 \pm 0.508$ & $0.00 \pm 0.09$ \\
 2 & $0.00 \pm 0.086$ & $0.00 \pm 0.06$ & $0.00 \pm 0.086$& $0.00 \pm 0.06$ & $0.00 \pm 0.359$ & $0.00 \pm 0.09$\\
 5 & $0.00 \pm 0.085$ & $0.00 \pm 0.07$ & $0.00 \pm 0.077$ & $0.00 \pm 0.06$ & $0.00 \pm 0.226$ & $0.00 \pm 0.09$\\
\bottomrule
\end{tabular}
}
\end{center}
\end{table}

\begin{table}[tb]
\caption{Mean $\pm$ standard deviation of the signed prediction errors for $\mu$ and $\sigma$ evaluated at $t=T/2=0.5$ for varying values of the observation probability $p$. This table showcases the bias and variance of each method for the predictions of both, $\mu$ and $\sigma$.
}
\label{tab:results CS2}
\begin{center}
\resizebox{\textwidth}{!}{
\begin{tabular}{l | r r | r r | r r}
\toprule
 &\multicolumn{2}{c|}{IO NJODE}&\multicolumn{2}{c|}{Particle Filter}&\multicolumn{2}{c}{Financial Estimator}\\
\cmidrule{1-7}
 $p$ & $\operatorname{err}_{\mu}$ & $\operatorname{err}_{\sigma}$ & $\operatorname{err}_{\mu}$ & $\operatorname{err}_{\sigma}$ & $\operatorname{err}_{\mu} $ & $\operatorname{err}_{\sigma}$ \\
 \midrule[\heavyrulewidth]
 0.01 & $0.00 \pm 0.10$ & $0.00 \pm 0.082$ & $0.00 \pm 0.10$ & $0.00 \pm 0.077$ & $0.02 \pm 0.99$ & $-0.01 \pm 0.189$ \\
 0.025 &  $0.00 \pm 0.09$ & $0.00 \pm 0.059$&  $0.00 \pm 0.09$ & $0.00 \pm 0.052$ & $0.00 \pm 0.84$ & $0.01 \pm 0.088$ \\
 0.05 &  $-0.01 \pm 0.09$ & $0.00 \pm 0.045$ & $0.00 \pm 0.09$ & $0.00 \pm 0.038$ &  $-0.02 \pm 0.78$ & $0.00 \pm 0.051$ \\
 0.1 &  $0.00 \pm 0.09$ & $0.00 \pm 0.038$ & $0.00 \pm 0.09$ & $0.00 \pm 0.027$ & $-0.02 \pm 0.70$ & $0.00 \pm 0.032$ \\
\bottomrule
\end{tabular}
}
\end{center}
\end{table}

\subsection{Cox--Ingersoll-Ross Process with Uncertain Parameters}\label{sec:Cox--Ingersoll-Ross Process with Uncertain Parameters}

\subsubsection{Experiment 1}
We first consider a CIR process with constant parameters as described in \Cref{exa:CIR constant params}, where we set $x_0 = 1$, $a_{\min} = 0.2, a_{\max}=2, b_{\min} = 1, b_{\max} = 5, \sigma_{\min} = 0.05, \sigma_{\max} = 0.5$. While our model can easily deal with this setting, the particle filter is numerically unstable. In particular, when evaluating the transition density \eqref{equ: CIR transition density}, the exponential term often becomes so small that it is numerically equal to $0$, while the modified Bessel function $I_q$ becomes so large that it is numerically equal to $\infty$. Whenever this happens, the density is numerically undefined and we set it to $0$. If this happens for all particles, then all have the weight $0$, such that we reset them to an equal weighting\footnote{Note that this is a way to deal with the practical issue of numerical instability, while theoretically the particle filter approach should work.}.
Since this happens for many of the particles, the resulting approximation of the conditional expectations is not good, as can be seen in the results. In particular, our early stopped model achieves a test loss of $1.62$ compared to a test loss of $2.34$ of the particle filter. 
In \Cref{fig:CIR PF} top left, we see that the weights of the particle filter are often reset to equal weighting, resulting in the prediction of the unconditional mean of the parameters.

\subsubsection{Experiment 2}
To understand how well our model learns to approximate the true conditional expectation, we need a reliable reference solution to which we can compare our model. Therefore, we choose the parameter distributions in such a way that no numerical instabilities occur. In particular, the argument of $I_q$ has to be small enough, which can be achieved through relatively large values of $\sigma$. However, at the same time the inequality $2ab \geq \sigma^2$ has to be satisfied, to ensure that $X > 0$, since otherwise $u$ can become $0$, leading to another numerical instability in the computation of \eqref{equ: CIR transition density}. 
To meet all requirements, we settle on the hyper-parameters $a_{\min} = 2, a_{\max}=3, b_{\min} = 1, b_{\max} = 2, \sigma_{\min} = 1, \sigma_{\max} = 2$ for the parameter distributions in \Cref{exa:CIR constant params}, which turn out to make the computation of the transition probability numerically stable.
Hence, the particle filter can reliably be used as reference approximating the true conditional expectation. Its test loss is $0.4445$, which is nearly matched by our model with a test loss of $0.4456$ at the epoch of minimal validation loss. The evaluation metric is $6.84 \cdot 10^{-4}$ there, confirming that our model approximates the true conditional expectation well, as can also be seen in \Cref{fig:CIR PF} top right.

\subsubsection{Experiment 3}
We use \Cref{exa:CIR time-dep params} to show that our model can also easily deal with more complicated settings, where the parameters are not constant. In particular, we set $w=2\pi$ and otherwise the same parameters as in Experiment 1. Again, the particle filter is numerically unstable having an optimal test loss of $2.75$, which is much larger than our model's test loss of $1.94$. In \Cref{fig:CIR PF} bottom left, we see that our model correctly picks up the time dependence of parameter $b$, and also that the particle filter often needs to be reset to the unconditional mean due to the numerical instability.

\subsubsection{Experiment 4}
Finally, we use \Cref{exa:CIR time-dep params} with $w=2 \pi$ and otherwise the same parameters as in Experiment 2, to get a numerically more stable particle filter as baseline (\Cref{fig:CIR PF} bottom right). As pointed out in \Cref{exa:CIR time-dep params}, the transition probability of the particle filter is only approximated (since the time-dependent parameter $b$ is kept constant at the value of the last observation). Hence, it is not surprising that our model achieves a smaller test loss of $0.4712$ than the particle filter with $0.4745$.
In particular, dealing with time-dependent parameters is straightforward for our model, while the particle filter is only a non-convergent approximation of the true conditional expectation.

Overall, these experiments illustrate the flexibility and robustness of our model in contrast to the particle filter, which not only needs full knowledge of the distributions but is also numerically unstable for the CIR process except for carefully chosen parameter distributions. Moreover, our model allows for high-quality approximations of the conditional expectation, as can be seen in Experiments 2 and 4, where the particle filter works.

\begin{figure}
\centering
\includegraphics[width=0.49\textwidth]{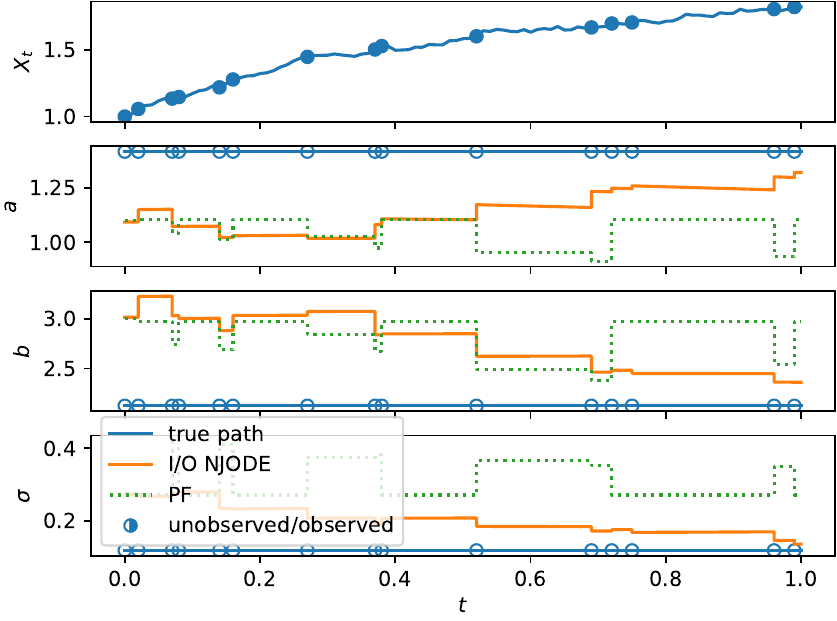}
\includegraphics[width=0.49\textwidth]{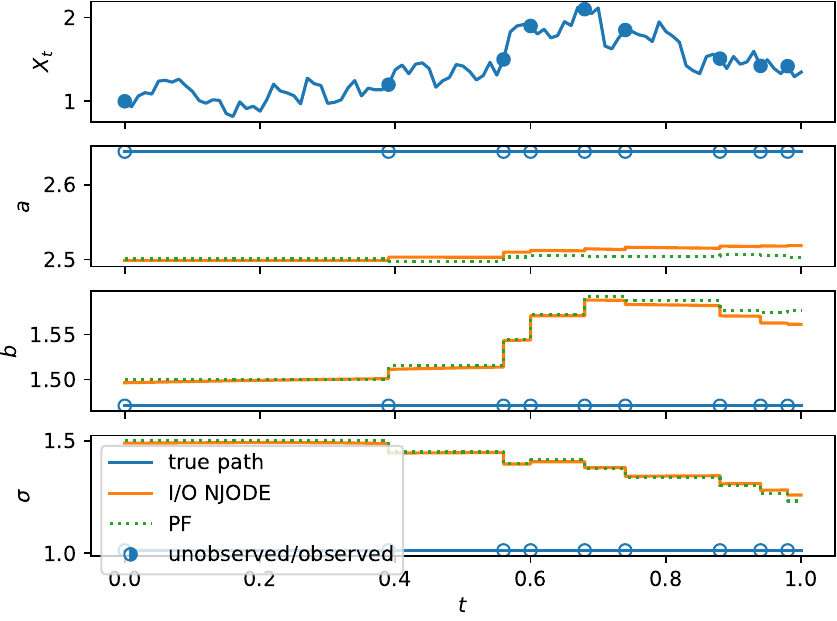}
\includegraphics[width=0.49\textwidth]{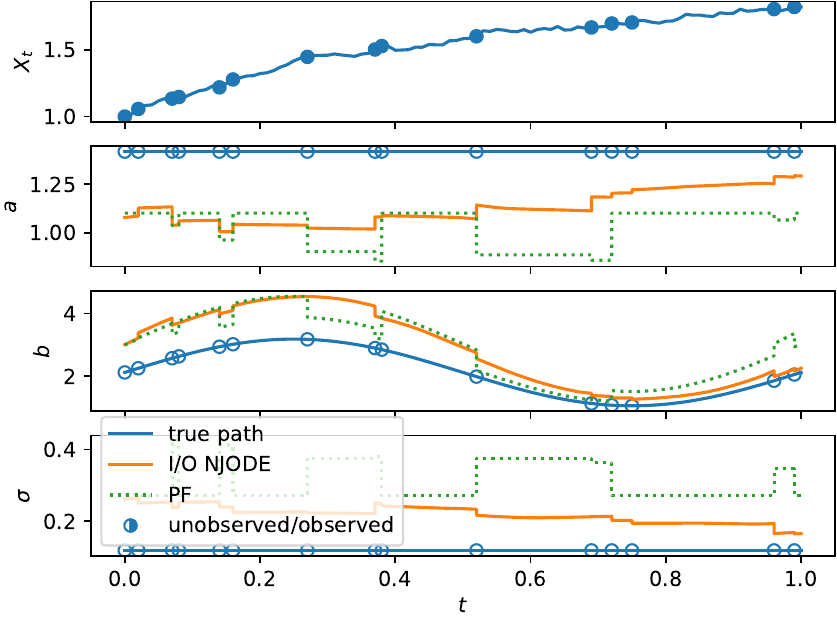}
\includegraphics[width=0.49\textwidth]{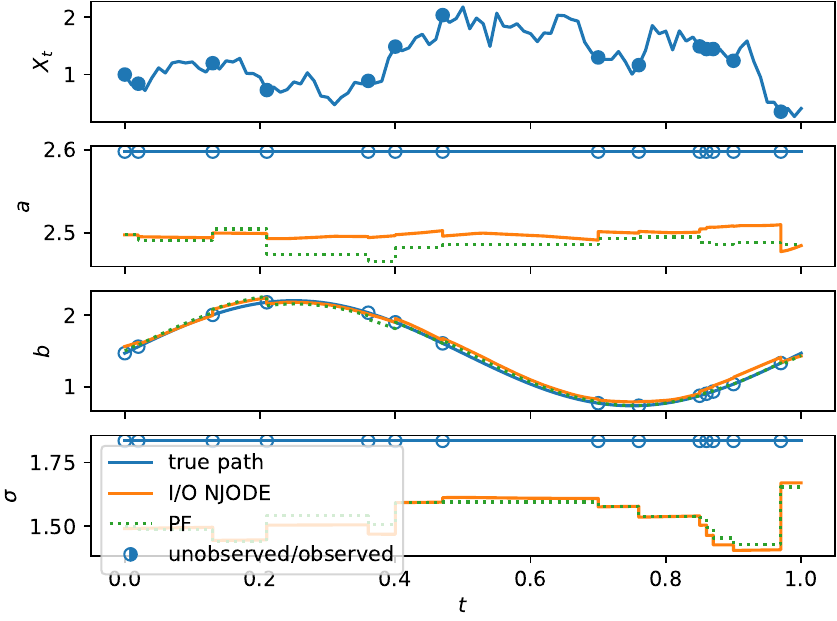}
\caption{Test samples of parameter predictions of the CIR processes in Experiment 1 (top left), Experiment 2 (top right), Experiment 3 (bottom left) and Experiment 4 (bottom right). The input coordinate $U=X$ is plotted on top and the output coordinates $V=(a,b_t, \sigma)$ below.}
\label{fig:CIR PF}
\end{figure}

\subsection{Brownian Motion Filtering}\label{sec:Brownian motion Filtering example}
We consider a signal and observation process $(X,Y)$ defined through Brownian motions as in \Cref{sec:Stochastic Filtering of Brownian Motion} with $\alpha=1$. The model uses the irregular observations of the observation process $Y$ as input and learns to filter out the signal process $X$. 
In \Cref{fig:BMFilter} left, we see that the model learns to filter the signal process well from the observations on one sample of the test set. The minimal validation loss of the model is $0.55186$ close to the optimal validation loss $0.55172$ achieved by the true conditional expectation. The evaluation metric on the test set at minimal validation loss is $9.6 \cdot 10^{-5}$. 

\begin{figure}
\centering
\includegraphics[width=0.49\textwidth]{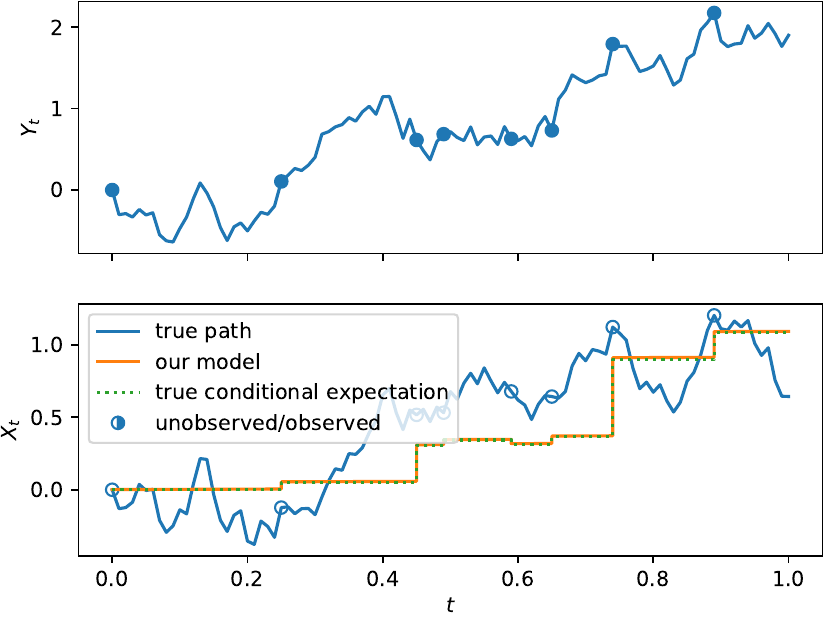}
\includegraphics[width=0.49\linewidth]{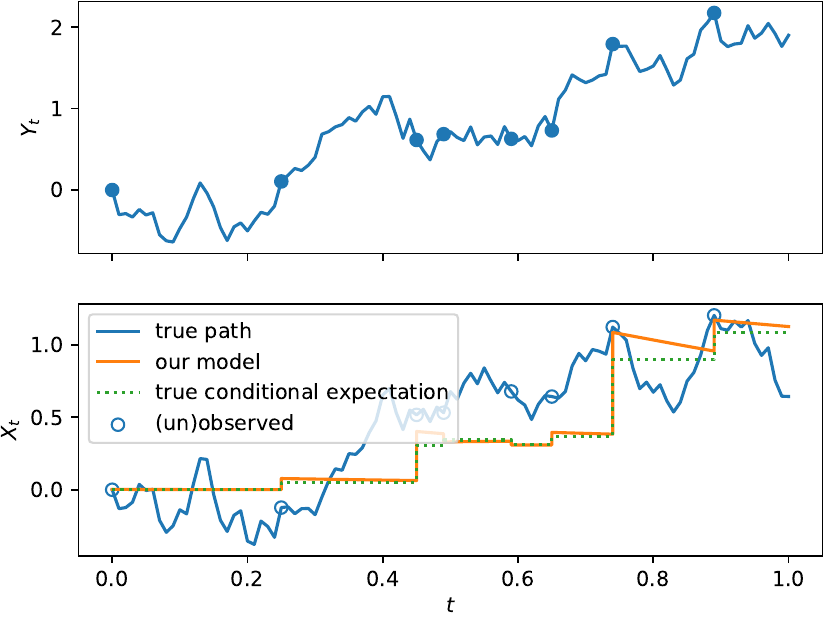}
\caption{Predicted conditional expectation of the signal process $X$ given the observations of the observation process $Y$ on a test sample, as well as the true conditional expectation. The input coordinate $U=Y$ is plotted on top, and the output coordinates $V=X$ below. The model closely replicates the true conditional expectation when trained with the input-output objective \eqref{equ:Psi} (left). Training with the old objective \eqref{equ:Psiold} leads to a different solution than the conditional expectation (right) as we will explain in \Cref{sec:Implications of the choice of Objective Function}.} 
\label{fig:BMFilter}
\label{fig:oldLossBias}
\end{figure}

\subsection{Classifying a Brownian Motion}\label{sec:Classifying a Brownian Motion}
We consider the online classification problem described in \Cref{exa:BM classification} with $\alpha=0$. In particular, given discrete and irregular past observations of the Brownian motion $W$, we want to compute the conditional probability that $W$ is currently larger than $0$. In \Cref{fig:BM Class} we show one sample of the test set, where we see that our model replicates the true conditional probability well. At minimal validation loss, it achieves a test loss of $0.1182$, which is slightly greater than the optimal test loss of $0.1147$ and the evaluation metric is $9.8 \cdot10^{-4}$. 
Examining the plots further, we see that the model does not fully learn the behavior close to the previous observation. In particular, if the previous observation was close to $\alpha = 0$, then the true conditional probability converges to $0.5$ very quickly, while it stays longer close to $1$ (or $0$) if the previous observation was much larger (or smaller) than $0$. To better learn this behavior, the model would need more training samples where two observations are close together, which are relatively scarce in the present setting. 

\begin{figure}
\centering
\includegraphics[width=0.49\textwidth]{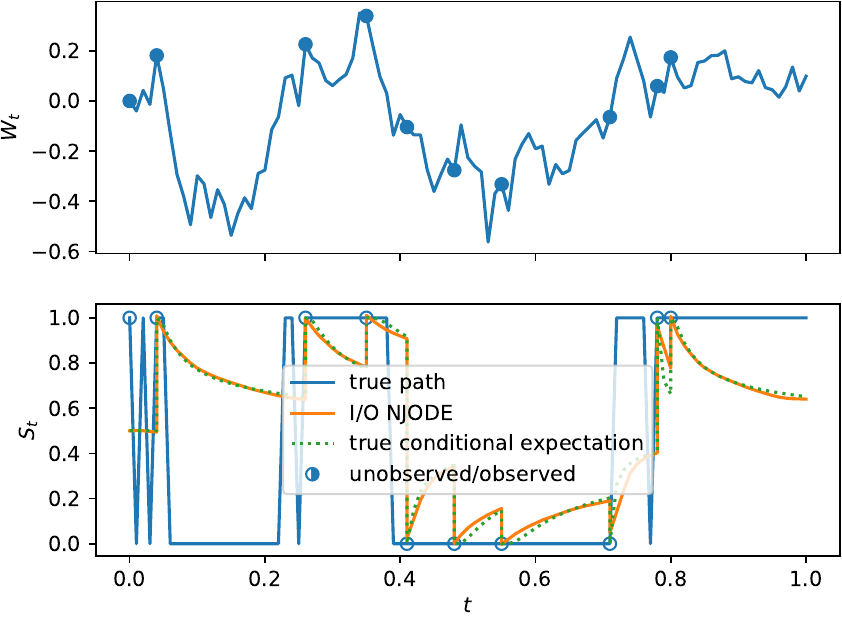}
\includegraphics[width=0.49\textwidth]{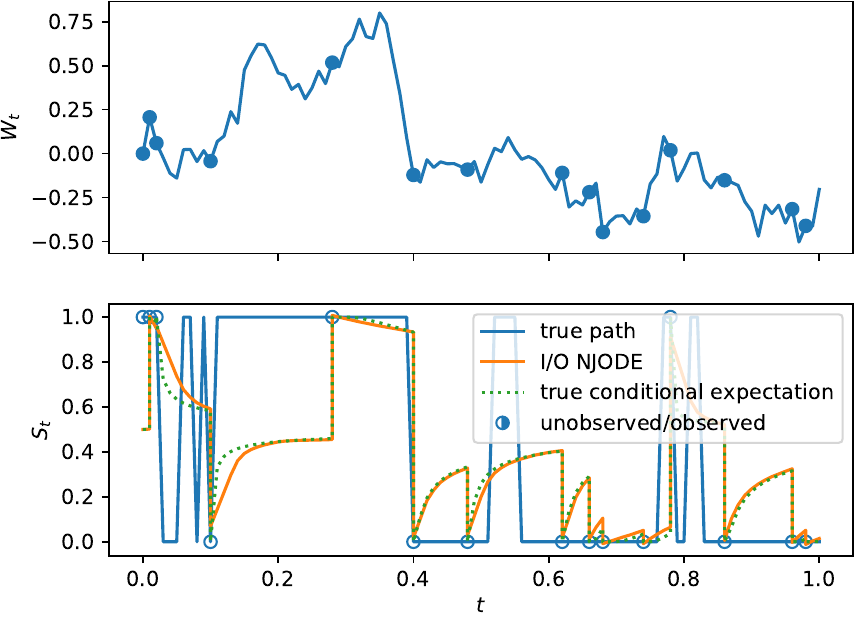}
\caption{Predicted and true conditional probability of the input process $U=W$ being greater than $0$, by computing the conditional expectation of the output process $V=S$ on two test samples.} 
\label{fig:BM Class}
\end{figure}

\section{Implications of the choice of Objective Function}\label{sec:Implications of the choice of Objective Function}
To be able to prove the convergence results as in \citet[Section~4]{krach2022optimal} within the more general input-output setting, we had to adapt the objective function. In particular, in the input-output case (more precisely, in the case where at least one of the output variables is not itself an input variable) the minimizer of the original loss function \citep[Equation~$11$]{krach2022optimal}, 
\begin{align}
\Psi_{\text{old}}(\eta) := \E\left[ \frac{1}{n} \sum_{i=1}^n    \left(\left\lvert \proj{V} (M_i) \odot ( V_{t_i} - \eta_{t_i} ) \right\rvert_2 + \left\lvert \proj{V} (M_i) \odot (V_{t_i} - \eta_{t_{i}-} ) \right\rvert_2\right)^2  \right], \label{equ:Psiold}
\end{align}
is not the conditional expectation. Hence, training the NJODE model with \eqref{equ:Psiold} should yield a solution different from the conditional expectation, with a smaller loss in terms of \eqref{equ:Psiold}, but larger loss with respect to \eqref{equ:Psi}.

To intuitively understand why the optimum of the original loss function differs from the conditional expectation, consider how the squared term outside the sum in the loss function couples the two loss terms: when the second term $\left\lvert \proj{V} (M_i) \odot (V_{t_i} - \eta_{t_{i}-} ) \right\rvert_2$ is far from zero, the first term $\left\lvert \proj{V} (M_i) \odot ( V_{t_i} - \eta_{t_i} ) \right\rvert_2 $ becomes more significant due to the large slope of the square function. Conversely, when the second term is close to zero, the first term's relevance diminishes (due to the diminishing slope of the square function close to zero).
In the following, we consider the setting from \Cref{sec:Brownian motion Filtering example} as an example.
In scenarios where there is a large observed jump, we expect the unobserved process to move in the same direction to some extent, but the NJODE can only respond through the first loss term. This situation creates two possible cases:
\begin{itemize}
    \item[1)] The unobserved process moves even further in the same direction, leading to a large second loss term
    \item[2)] The unobserved process moves less in the same direction, or slightly in the opposite direction, resulting in a smaller second loss term. 
\end{itemize}
The old loss function~\eqref{equ:Psiold} tends to emphasize the first term more in Case $1$, biasing the model to make larger jumps in the observed direction. Consequently, after jumping too far the model rapidly\footnote{Asymptotically the model would adjust back to the true conditional expectation directly after jumping too far, but due to (implicit) regularization for any finite training time, the model takes some time to move back to the correct conditional expectation after jumping too far.} adjusts to align with the correct conditional expectation in subsequent observations (see \Cref{fig:oldLossBias} right).

In general, experimental results show that the model trained with the old loss function~\eqref{equ:Psiold} typically outperforms the true conditional expectation when validated via \eqref{equ:Psiold}. 
For instance, training the NJODE with the old loss function on the same dataset as in \Cref{sec:Brownian motion Filtering example} yields a test loss of $1.05132$, significantly smaller than the test loss of $1.06276$ achieved by the conditional expectation. 
This confirms that the optimum of the old loss function is different from the conditional expectation and that the NJODE manages to capture this during training.
On the one hand, this makes it clear that we should use the new loss~\eqref{equ:Psi} instead of the old one~\eqref{equ:Psiold} when we want to learn the conditional expectation.
And, on the other hand, this showcases the flexibility and powerfulness of our model, allowing it to effectively exploit subtle changes in the loss structure. Moreover, as we saw in \Cref{sec:Brownian motion Filtering example}, training the NJODE model with the new objective function yields a very good approximation of the true conditional expectation, as our theoretical results suggest.

\subsection{Choosing the Objective Function}
In settings where the output process $V$ is a subprocess of the input process $U$, i.e., when all output coordinates are also used as input coordinates \citep[this is, in particular, the case in the setting of][]{krach2022optimal}, both loss functions \eqref{equ:Psiold} and \eqref{equ:Psi} can be used. Indeed, it is easy to verify that in this case the minimizer of \eqref{equ:Psiold} is also given by the conditional expectation, and therefore the model converges to it with both loss functions. 
Hence, the question arises as to which loss function should be preferred in this case. Comparing the two objective functions, we see that the input-output (IO) loss \eqref{equ:Psi} is a pure $L^2$ objective, while the original loss \eqref{equ:Psiold} has an $L^1$ aspect. In particular, by summing the two terms before squaring them, the first unsquared term, $\left\lvert \proj{V} (M_i) \odot ( V_{t_i} - \eta_{t_i} ) \right\rvert_2$,
acts similarly as a lasso regularization. It pushes the model to perfectly and quickly learn the jump such that the first term, and therefore also the interaction term (when computing the square), vanish. On the other hand, \eqref{equ:Psi} does not impose this inductive bias such that the model will only (closely) approximate, but not perfectly replicate, the jump within a finite training time (even though in the limit it would do so).
In particular, the convergence will be slower. 

We experimentally verify this by training the same model on the Black--Scholes dataset of \citet[Appendix~F.1]{herrera2020neural} where $U=V=X$ (cf.\ \Cref{sec:Details for Loss Comparison on Black--Scholes}) with the two different objective functions \eqref{equ:Psiold} and \eqref{equ:Psi}, while monitoring the pure \emph{jump loss}
\begin{align}\label{equ:pure jump loss}
\Psi_{\text{jump}}(\eta) := \E  \left[\frac{1}{n} \sum_{i=1}^n  \left\lvert \proj{V} (M_i) \odot ( V_{t_i} - \eta_{t_i} ) \right\rvert_2^2  \right], 
\end{align}
on the validation set. 
In particular, \eqref{equ:pure jump loss} monitors how well the model reconstructs the observation after it was fed as input.
In \Cref{fig:loss Comparison} we see that the model learns the jump faster and better with the original loss function. The minimal (pure jump) validation loss is $6.1 \cdot 10^{-10}$ when training with the original loss compared to $1.4 \cdot 10^{-6}$ when training with the IO loss. 
We note that a small constant $\epsilon>0$ is added for the computation of the square-roots in the original loss function for numerical stability. This is chosen as $\epsilon=10^{-10}$, hence, this is a natural lower bound for the original loss\footnote{In principle, the lower bound is $0$, however, during training the model does not see a benefit when decreasing this loss term further since it is overshadowed by the impact of $\epsilon$.}.
\begin{figure}
    \centering
    \includegraphics[width=0.6\linewidth]{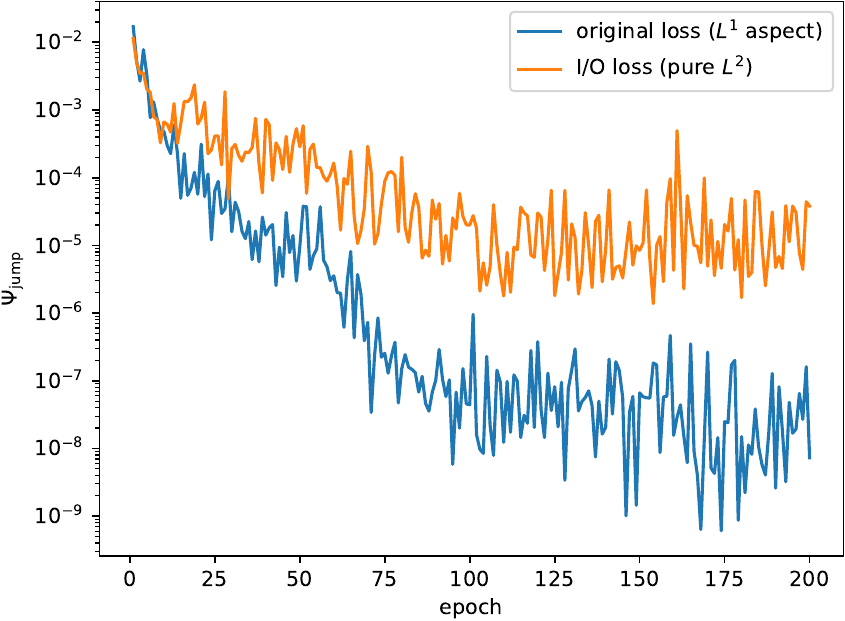}
    \caption{The jump loss $\Psi_{\text{jump}}$ converges faster and to a smaller value when trained with the original loss function (with $L^1$ aspect) \eqref{equ:Psiold} than with the pure $L^2$ IO loss function \eqref{equ:Psi}.}
    \label{fig:loss Comparison}
\end{figure}

Therefore, it is beneficial to use the old loss function when both loss functions are theoretically justified, i.e., when all output coordinates are also provided as input coordinates. As soon as an output coordinate is not provided as input, the IO loss function \eqref{equ:Psi} should be used to guarantee convergence to the conditional expectation.

\section{Conclusion}
In this work, we extended the framework of Neural Jump ODEs to input-output systems, making them applicable to nonparametric, data-driven filtering and classification tasks. 
In contrast to previous work that was limited to coinciding input and output processes, the Input-Output Neural Network Jump ODE  allows for a clear separation between the observation and the estimation processes, opening the way to a variety of real-world applications where data is often incomplete and non-regular and online processing of new incoming observations is necessary for rapid decision-making.

This research combines sound theoretical contributions, providing rigorous convergence proofs that ensure the optimality of the model in an $L^2$-optimization framework, with an open-source ready-to-use implementation and shows several empirical results. 
Our experiments highlight that the IO NJODE offers great flexibility and achieves high accuracy, while being simple to use as fully data-driven framework that does not need any knowledge or estimates of the underlying distributions.

Through this work, we aim to make Neural Jump ODEs accessible for promising applications in fields where input-output system dynamics are pivotal, such as algorithmic finance, real-time health monitoring and control systems, and risk detection in complex environments.

\if\preprint1
    \section*{Acknowledgment}
    The authors thank Josef Teichmann for helpful inputs and discussions. 
    Moreover, we thank the reviewers for their insightful comments that helped to improve the paper.
    The funding for  Félix Ndonfack by the Deutsche Forschungsgemeinschaft (DFG, German Research Foundation) – Project-ID 499552394 – SFB 1597 Small Data - and support from the FDMAI, Freiburg, are gratefully acknowledged.
\else
    \begin{acknowledgement}
      The authors thank Josef Teichmann for helpful inputs and discussions. 
      Moreover, we thank the reviewers for their insightful comments that helped to improve the paper.
    \end{acknowledgement}
    
    \begin{funding}
      The funding for  Félix Ndonfack by the Deutsche Forschungsgemeinschaft (DFG, German Research Foundation) – Project-ID 499552394 – SFB 1597 Small Data - and support from the FDMAI, Freiburg, are gratefully acknowledged.
    \end{funding}
\fi

\renewcommand{\bibname}{References}
\if\preprint1
    \bibliographystyle{tmlr}
\else
    \bibliographystyle{tmlr}
\fi
\bibliography{Ref}

\begin{thebibliography}{47}
\providecommand{\natexlab}[1]{#1}
\providecommand{\url}[1]{\texttt{#1}}
\expandafter\ifx\csname urlstyle\endcsname\relax
  \providecommand{\doi}[1]{doi: #1}\else
  \providecommand{\doi}{doi: \begingroup \urlstyle{rm}\Url}\fi

\bibitem[Andersson et~al.(2024)Andersson, Heiss, Krach, and
  Teichmann]{andersson2024extending}
William Andersson, Jakob Heiss, Florian Krach, and Josef Teichmann.
\newblock Extending path-dependent {NJ}-{ODE}s to noisy observations and a
  dependent observation framework.
\newblock \emph{Transactions on Machine Learning Research}, 2024.
\newblock ISSN 2835-8856.
\newblock URL \url{https://openreview.net/forum?id=0T2OTVCCC1}.

\bibitem[Archer et~al.(2015)Archer, Park, Buesing, Cunningham, and
  Paninski]{archer2015black}
Evan Archer, Il~Memming Park, Lars Buesing, John Cunningham, and Liam Paninski.
\newblock Black box variational inference for state space models.
\newblock \emph{arXiv preprint arXiv:1511.07367}, 2015.

\bibitem[Azizi et~al.(2025)Azizi, Bodik, Heiss, and
  Yu]{CLEARcalibratedlearningepistemic}
Ilia Azizi, Juraj Bodik, Jakob Heiss, and Bin Yu.
\newblock Clear: Calibrated learning for epistemic and aleatoric risk, 2025.
\newblock URL \url{https://arxiv.org/abs/2507.08150}.

\bibitem[Bugajewski \& Gulgowski(2020)Bugajewski and
  Gulgowski]{BUGAJEWSKI2020123752}
Dariusz Bugajewski and Jacek Gulgowski.
\newblock On the characterization of compactness in the space of functions of
  bounded variation in the sense of {Jordan}.
\newblock \emph{Journal of Mathematical Analysis and Applications},
  484\penalty0 (2), 2020.

\bibitem[Caruana et~al.(1997)Caruana, Pratt, and Thrun]{MultitaskCaruana1997}
Rich Caruana, Lorien Pratt, and Sebastian Thrun.
\newblock Multitask learning.
\newblock \emph{Machine Learning 1997 28:1}, 28:\penalty0 41--75, 1997.
\newblock ISSN 1573-0565.
\newblock \doi{10.1023/A:1007379606734}.
\newblock URL \url{https://link.springer.com/article/10.1023/A:1007379606734}.

\bibitem[Chevyrev \& Kormilitzin(2016)Chevyrev and
  Kormilitzin]{Chevyrev2016APO}
Ilya Chevyrev and Andrey Kormilitzin.
\newblock A primer on the signature method in machine learning.
\newblock \emph{arXiv}, 2016.

\bibitem[Cohen \& Elliott(2015)Cohen and Elliott]{cohen2015stochastic}
Samuel~N Cohen and Robert~James Elliott.
\newblock Stochastic calculus and applications.
\newblock \emph{Springer}, 2015.

\bibitem[Corenflos et~al.(2021)Corenflos, Thornton, Deligiannidis, and
  Doucet]{corenflos2021differentiable}
Adrien Corenflos, James Thornton, George Deligiannidis, and Arnaud Doucet.
\newblock Differentiable particle filtering via entropy-regularized optimal
  transport.
\newblock In \emph{International Conference on Machine Learning}, pp.\
  2100--2111. PMLR, 2021.

\bibitem[Cox et~al.(1985)Cox, Ingersoll, and Ross]{CIR}
John~C Cox, Jonathan~E Ingersoll, and Stephen~A Ross.
\newblock A theory of the term structure of interest rates.
\newblock \emph{Econometrica}, 53\penalty0 (2):\penalty0 385--407, 1985.
\newblock ISSN 00129682, 14680262.

\bibitem[Cs{\"o}rg{\H{o}} et~al.(1983)Cs{\"o}rg{\H{o}}, Tandori, and
  Totik]{csorgHo1983strong}
S{\'a}ndor Cs{\"o}rg{\H{o}}, K{\'a}roly Tandori, and Vilmos Totik.
\newblock On the strong law of large numbers for pairwise independent random
  variables.
\newblock \emph{Acta Mathematica Hungarica}, 42:\penalty0 319--330, 1983.

\bibitem[Cuchiero et~al.(2025)Cuchiero, Primavera, and
  Svaluto-Ferro]{cuchiero_universal_2025}
Christa Cuchiero, Francesca Primavera, and Sara Svaluto-Ferro.
\newblock Universal approximation theorems for continuous functions of càdlàg
  paths and {Lévy}-type signature models.
\newblock \emph{Finance and Stochastics}, 29\penalty0 (2):\penalty0 289--342,
  April 2025.
\newblock ISSN 1432-1122.
\newblock \doi{10.1007/s00780-025-00557-5}.

\bibitem[Djuric et~al.(2003)Djuric, Kotecha, Zhang, Huang, Ghirmai, Bugallo,
  and Miguez]{djuric2003particle}
Petar~M Djuric, Jayesh~H Kotecha, Jianqui Zhang, Yufei Huang, Tadesse Ghirmai,
  M{\'o}nica~F Bugallo, and Joaquin Miguez.
\newblock Particle filtering.
\newblock \emph{IEEE signal processing magazine}, 20\penalty0 (5):\penalty0
  19--38, 2003.

\bibitem[Duffie \& Lando(2001)Duffie and Lando]{duffie2001term}
Darrell Duffie and David Lando.
\newblock Term structures of credit spreads with incomplete accounting
  information.
\newblock \emph{Econometrica}, 69\penalty0 (3):\penalty0 633--664, 2001.

\bibitem[Eaton(2007)]{Eaton2007Multi}
Morris~L Eaton.
\newblock \emph{Multivariate statistics: A vector space approach}.
\newblock Institute of Mathematical Statistics, 2007.

\bibitem[Fermanian(2021)]{fermanian2020embedding}
Adeline Fermanian.
\newblock Embedding and learning with signatures.
\newblock \emph{Computational Statistics \& Data Analysis}, 157:\penalty0
  107148, 2021.

\bibitem[Fontana \& Schmidt(2018)Fontana and Schmidt]{fontana2018general}
Claudio Fontana and Thorsten Schmidt.
\newblock General dynamic term structures under default risk.
\newblock \emph{Stochastic Processes and their Applications}, 128\penalty0
  (10):\penalty0 3353--3386, 2018.

\bibitem[Frey \& Schmidt(2009)Frey and Schmidt]{frey2009pricing}
R{\"u}diger Frey and Thorsten Schmidt.
\newblock Pricing corporate securities under noisy asset information.
\newblock \emph{Mathematical Finance}, 19\penalty0 (3):\penalty0 403--421,
  2009.

\bibitem[Frey \& Schmidt(2012)Frey and Schmidt]{frey2012pricing}
R{\"u}diger Frey and Thorsten Schmidt.
\newblock Pricing and hedging of credit derivatives via the innovations
  approach to nonlinear filtering.
\newblock \emph{Finance and Stochastics}, 16\penalty0 (1):\penalty0 105--133,
  2012.

\bibitem[Gehmlich \& Schmidt(2018)Gehmlich and Schmidt]{gehmlich2018dynamic}
Frank Gehmlich and Thorsten Schmidt.
\newblock Dynamic defaultable term structure modeling beyond the intensity
  paradigm.
\newblock \emph{Mathematical Finance}, 28\penalty0 (1):\penalty0 211--239,
  2018.

\bibitem[Hackenberg et~al.(2022)Hackenberg, Harms, Pfaffenlehner, Pechmann,
  Kirschner, Schmidt, and Binder]{hackenberg2022deep}
Maren Hackenberg, Philipp Harms, Michelle Pfaffenlehner, Astrid Pechmann,
  Janbernd Kirschner, Thorsten Schmidt, and Harald Binder.
\newblock Deep dynamic modeling with just two time points: Can we still allow
  for individual trajectories?
\newblock \emph{Biometrical Journal}, 64\penalty0 (8):\penalty0 1426--1445,
  2022.

\bibitem[Hackenberg et~al.(2025)Hackenberg, Pfaffenlehner, Behrens, Pechmann,
  Kirschner, and Binder]{hackenberg2025investigating}
Maren Hackenberg, Michelle Pfaffenlehner, Max Behrens, Astrid Pechmann,
  Janbernd Kirschner, and Harald Binder.
\newblock Investigating a domain adaptation approach for integrating different
  measurement instruments in a longitudinal clinical registry.
\newblock \emph{Biometrical Journal}, 67\penalty0 (1):\penalty0 e70023, 2025.

\bibitem[Heiss(2024)]{HeissInductiveBias2024}
Jakob Heiss.
\newblock \emph{Inductive Bias of Neural Networks and Selected Applications}.
\newblock Doctoral thesis, ETH Zurich, Zurich, 2024.
\newblock URL
  \url{https://www.research-collection.ethz.ch/handle/20.500.11850/699241}.

\bibitem[Heiss et~al.(2022)Heiss, Teichmann, and Wutte]{HeissImplReg3}
Jakob Heiss, Josef Teichmann, and Hanna Wutte.
\newblock How infinitely wide neural networks can benefit from multi-task
  learning - an exact macroscopic characterization.
\newblock \emph{arXiv preprint arXiv:2112.15577}, 2022.
\newblock \doi{10.3929/ETHZ-B-000550890}.
\newblock URL \url{http://hdl.handle.net/20.500.11850/550890}.

\bibitem[Herrera et~al.(2021)Herrera, Krach, and Teichmann]{herrera2020neural}
Calypso Herrera, Florian Krach, and Josef Teichmann.
\newblock Neural jump ordinary differential equations: Consistent
  continuous-time prediction and filtering.
\newblock In \emph{International Conference on Learning Representations}, 2021.
\newblock URL \url{https://openreview.net/forum?id=JFKR3WqwyXR}.

\bibitem[Hornik(1991)]{hornik1991approximation}
Kurt Hornik.
\newblock Approximation capabilities of multilayer feedforward networks.
\newblock \emph{Neural networks}, 4\penalty0 (2):\penalty0 251--257, 1991.

\bibitem[Hornik et~al.(1989)Hornik, Stinchcombe, and
  White]{10.5555/70405.70408}
Kurt Hornik, Maxwell Stinchcombe, and Halbert White.
\newblock Multilayer feedforward networks are universal approximators.
\newblock \emph{Neural Networks}, 2, 1989.

\bibitem[Jordan(1997)]{jordan1997serial}
Michael~I Jordan.
\newblock Serial order: A parallel distributed processing approach.
\newblock In \emph{Advances in Psychology}, volume 121. Elsevier, 1997.

\bibitem[Kallenberg(2021)]{kallenberg2021foundations}
Olav Kallenberg.
\newblock \emph{Foundations of modern probability}.
\newblock Springer, 3rd edition, 2021.

\bibitem[Kalman et~al.(1960)]{kalman1960new}
Rudolph~E Kalman et~al.
\newblock A new approach to linear filtering and prediction problems [j].
\newblock \emph{Journal of basic Engineering}, 82\penalty0 (1):\penalty0
  35--45, 1960.

\bibitem[Karandikar \& Rao(2018)Karandikar and Rao]{KarandikarRao2018}
Rajeeva~L Karandikar and Bhamidi~V Rao.
\newblock \emph{Introduction to Stochastic Calculus}.
\newblock Indian Statistical Institute Series. Springer Singapore, Singapore,
  2018.
\newblock ISBN 978-981-10-8317-4.
\newblock \doi{10.1007/978-981-10-8318-1}.
\newblock URL \url{http://link.springer.com/10.1007/978-981-10-8318-1}.

\bibitem[Karl et~al.(2016)Karl, Soelch, Bayer, and Van~der Smagt]{karl2016deep}
Maximilian Karl, Maximilian Soelch, Justin Bayer, and Patrick Van~der Smagt.
\newblock Deep variational bayes filters: Unsupervised learning of state space
  models from raw data.
\newblock \emph{arXiv preprint arXiv:1605.06432}, 2016.

\bibitem[Khaled \& Samia(2010)Khaled and Samia]{khaled2010estimation}
Khaldi Khaled and Meddahi Samia.
\newblock Estimation of the parameters of the stochastic differential equations
  black-scholes model share price of gold.
\newblock \emph{Journal of Mathematics and Statistics}, 6\penalty0
  (4):\penalty0 421, 2010.

\bibitem[Kiraly \& Oberhauser(2019)Kiraly and Oberhauser]{KiralyOberhauser2019}
Franz~J Kiraly and Harald Oberhauser.
\newblock Kernels for sequentially ordered data.
\newblock \emph{Journal of Machine Learning Research}, 20\penalty0
  (31):\penalty0 1--45, 2019.

\bibitem[Krach(2025)]{KrachPhDThesis}
Florian Krach.
\newblock \emph{Neural Jump Ordinary Differential Equations}.
\newblock Doctoral thesis, ETH Zurich, Zurich, 2025.

\bibitem[Krach \& Teichmann(2024)Krach and Teichmann]{krach2024learning}
Florian Krach and Josef Teichmann.
\newblock Learning chaotic systems and long-term predictions with neural jump
  odes.
\newblock \emph{arXiv preprint arXiv:2407.18808}, 2024.

\bibitem[Krach et~al.(2022)Krach, N{\"u}bel, and Teichmann]{krach2022optimal}
Florian Krach, Marc N{\"u}bel, and Josef Teichmann.
\newblock {Optimal estimation of generic dynamics by path-dependent neural jump
  ODEs}.
\newblock \emph{arXiv preprint arXiv:2206.14284}, 2022.

\bibitem[Krishnan et~al.(2015)Krishnan, Shalit, and Sontag]{krishnan2015deep}
Rahul~G Krishnan, Uri Shalit, and David Sontag.
\newblock Deep {Kalman} filters.
\newblock \emph{arXiv preprint arXiv:1511.05121}, 2015.

\bibitem[Krishnan et~al.(2017)Krishnan, Shalit, and
  Sontag]{krishnan2017structured}
Rahul~G Krishnan, Uri Shalit, and David Sontag.
\newblock Structured inference networks for nonlinear state space models.
\newblock In \emph{Proceedings of the AAAI Conference on Artificial
  Intelligence}, volume~31, 2017.

\bibitem[Lai et~al.(2022)Lai, Domke, and Sheldon]{lai2022variational}
Jinlin Lai, Justin Domke, and Daniel Sheldon.
\newblock Variational marginal particle filters.
\newblock In \emph{International Conference on Artificial Intelligence and
  Statistics}, pp.\  875--895. PMLR, 2022.

\bibitem[Le et~al.(2017)Le, Igl, Rainforth, Jin, and Wood]{le2017auto}
Tuan~Anh Le, Maximilian Igl, Tom Rainforth, Tom Jin, and Frank Wood.
\newblock Auto-encoding sequential monte carlo.
\newblock \emph{arXiv}, 2017.

\bibitem[Maddison et~al.(2017)Maddison, Lawson, Tucker, Heess, Norouzi, Mnih,
  Doucet, and Teh]{maddison2017filtering}
Chris~J Maddison, John Lawson, George Tucker, Nicolas Heess, Mohammad Norouzi,
  Andriy Mnih, Arnaud Doucet, and Yee Teh.
\newblock Filtering variational objectives.
\newblock \emph{Advances in Neural Information Processing Systems}, 30, 2017.

\bibitem[Merton(1974)]{merton1974pricing}
Robert~C Merton.
\newblock On the pricing of corporate debt: The risk structure of interest
  rates.
\newblock \emph{The Journal of finance}, 29\penalty0 (2):\penalty0 449--470,
  1974.

\bibitem[Naesseth et~al.(2018)Naesseth, Linderman, Ranganath, and
  Blei]{naesseth2018variational}
Christian Naesseth, Scott Linderman, Rajesh Ranganath, and David Blei.
\newblock Variational sequential monte carlo.
\newblock In \emph{International conference on artificial intelligence and
  statistics}, pp.\  968--977. PMLR, 2018.

\bibitem[Ralff(2021)]{stackexchange1}
Ralff.
\newblock Kalman filtering: Processing all measurements together vs processing
  them sequentially.
\newblock Mathematics Stack Exchange, 2021.
\newblock URL \url{https://math.stackexchange.com/q/4058151}.
\newblock URL: https://math.stackexchange.com/q/4058151 (version: 2021-03-11).

\bibitem[Revach et~al.(2022)Revach, Shlezinger, Ni, Escoriza, Van~Sloun, and
  Eldar]{revach2022kalmannet}
Guy Revach, Nir Shlezinger, Xiaoyong Ni, Adria~Lopez Escoriza, Ruud~JG
  Van~Sloun, and Yonina~C Eldar.
\newblock Kalmannet: Neural network aided kalman filtering for partially known
  dynamics.
\newblock \emph{IEEE Transactions on Signal Processing}, 70:\penalty0
  1532--1547, 2022.

\bibitem[Rumelhart et~al.(1985)Rumelhart, Hinton, and
  Williams]{rumelhart1985learning}
David~E Rumelhart, Geoffrey~E Hinton, and Ronald~J Williams.
\newblock Learning internal representations by error propagation.
\newblock Technical report, California Univ San Diego La Jolla Inst for
  Cognitive Science, 1985.

\bibitem[Schmidt \& Novikov(2008)Schmidt and Novikov]{schmidt2008structural}
Thorsten Schmidt and Alexander Novikov.
\newblock A structural model with unobserved default boundary.
\newblock \emph{Applied mathematical finance}, 15\penalty0 (2):\penalty0
  183--203, 2008.

\end{thebibliography}

\appendix
\section*{Appendix}

\section{Signatures}\label{sec:Signatures}
We give a brief overview of the signature transform and its universal approximation results, following \citet[Section~3.2]{krach2022optimal}. We start by defining paths of bounded variation.
\begin{definition}\label{def:p-variation}
    Let  $J$ be a closed interval in $\R$ and $d\geq 1$.
    Let ${X}:J\rightarrow\R^{d}$ be a path on $J$.
    The variation of ${X}$ on the interval $J$ is defined by
    \begin{equation*}
        \|{X}\|_{var, J}
        =  \sup_{P(J)}\sum_{t_j \in P(J)} |{X}_{t_{j}}-{X}_{t_{j-1}} |_2 ,
    \end{equation*}
    where the supremum is taken over all finite partitions $P(J)$ of $J$.
\end{definition}
\begin{definition}\label{def:BV}
We denote the set of $\R^{d}$-valued paths of bounded variation  on $J$ by $BV(J, \R^d)$ and endow it with the norm 
        $$\lVert {X} \rVert_{BV} := |{X}_0|_2 +  \|{X}\|_{var, J}.$$ 
\end{definition}
For continuous paths of bounded variation we can define the signature transform.
\begin{definition}\label{def:signature}
    Let $J$ denote a closed interval in $\R$.
    Let ${X}:J\rightarrow\R^{d}$ be a continuous path with finite variation.
    The signature of ${X}$ is defined as
    \begin{equation*}
        S({X}) = \left(1, {X}_J^1, {X}_J^2, \dots\right),
    \end{equation*}
    where, for each $m\geq 1$,
    \begin{equation*}
        {X}_J^m = \int_{\substack{u_1<\dots<u_m \\ u_1,\dots,u_m\in J}} d{X}_{u_1}\otimes\dots\otimes d{X}_{u_m} \in (\R^d)^{\otimes m}
    \end{equation*}
    is a collection of iterated integrals.
    The map from a path to its signature is called signature transform.
\end{definition}

A good introduction to the signature transform with its properties and examples can be found in \citet{Chevyrev2016APO, KiralyOberhauser2019, fermanian2020embedding}. 
In practice, we are not able to use the full (infinite) signature, but instead use a truncated version.
\begin{definition}
    Let $J$ denote a compact interval in $\R$.
    Let ${X}:J\rightarrow\R^{d}$ be a continuous path with finite variation.
    The  truncated signature of ${X}$ of order $m$ is defined as 
    \begin{equation*}
        \pi_m({X}) = \left(1,{X}_J^1,{X}_J^2,\dots,{X}_J^m\right),
    \end{equation*}
    i.e., the first $m+1$ terms (levels) of the signature of ${X}$.
\end{definition}
Note that the size of the truncated signature depends on the dimension of ${X}$, as well as the chosen truncation level.
Specifically, for a path of dimension $d$, the dimension of the truncated signature of order $m$ is given by
\begin{equation}\label{eq:sig_nb_terms}
\begin{cases}
m+1, & \text{if } d =1, \\
 \frac{d^{m+1}-1}{d-1}, & \text{if } d >1. 
\end{cases}   
\end{equation}
When using the truncated signature as input to a model this results in a trade-off between accurately describing the path and model complexity.

A well-known result in stochastic analysis states that  continuous functions  can be uniformly approximated using truncated signatures, which is made precise in the following theorem. For references to the literature and an idea of the proof, see \citet{KiralyOberhauser2019}, Theorem 1. 
This classical result was extended in \citet[Proposition~3.8]{krach2022optimal} to  additionally incorporate the input $c$ from a compact set $C \subset \R^m$ which can be stated as follows:

\begin{prop}\label{prop:universal_approx_sig}
Consider $\mathcal{P}$ as a compact subset of $BV_0^c ([0,1],\mathcal{H})$ consisting of paths that are not tree-like equivalent, and let $C \subset \R^m$ for some $m \in \N$ be compact. We take the Cartesian product $BV_0^c ([0,1],\mathcal{H}) \times \R^m$ with the product norm defined as the sum of the individual norms (variation norm and $1$-norm). Suppose $f:\mathcal{P} \times C \rightarrow \R$ is continuous. Then, for any $\varepsilon>0$, there exists an $M>0$ and a continuous function $\tilde f$ such that
$$
    \sup_{(x,c) \in \mathcal{P} \times C} \left| f(x, c) - \tilde f (\pi_M(x), c) \right| < \varepsilon.
$$
\end{prop}

To apply this result, we need a tractable description of certain compact subsets of $BV_0^c ([0,1],\mathcal{\R}^d)$ that include suitable paths for our considerations. Since $BV_0^c ([0,1],\mathcal{\R}^d)$ is not finite dimensional, not every closed and bounded subset is compact. \citet{BUGAJEWSKI2020123752} prove in Example 4,  that the following set of functions is relatively compact. As already observed in \citet[Remark 3.11]{krach2022optimal}, this also holds for $\R^d$-valued paths.
\begin{prop}\label{prop:compact subset BV}
For any $N \in \N$, the set $A_N \subseteq BV_0^c([0,1], \R)$, consisting of all piecewise linear, bounded, and continuous functions expressible as
\begin{equation*}
f(t) = (a_1 t) \1_{[s_0, s_1]}(t) + \sum_{i=2}^N (a_i t + b_i) \1_{(s_{i-1}, s_i]}(t),
\end{equation*}
is relatively compact. Here, $a_i, b_i \in [-N, N]$, $b_1 = 0$, $a_i s_i +b_i = a_{i+1} s_i + b_{i+1}$ for all $1 \leq i \leq N$, and $0 = s_0 < s_1 < \dotsb < s_N = 1$.
\end{prop}

\section{Kalman Filter}\label{sec:Kalman Filter}
If the observation and signal distributions in a filtering system are Gaussian and independent, then the Kalman filter \citep{kalman1960new} recovers the optimal solution, i.e., the true conditional expectation. This is the case in \Cref{exa:brownian motion uncertain drift}, where the normally distributed drift should be filtered from observations of a Brownian motion with drift, and in the filtering example with two Brownian motions in \Cref{sec:Brownian motion Filtering example}. 

We first recall the Kalman filter in \Cref{sec:Definition of the Kalman Filter} and then show that applying it in \Cref{exa:brownian motion uncertain drift} leads to the same result as the direct computation of the conditional expectation given there. The example of \Cref{sec:Brownian motion Filtering example} works along the same lines. 

\subsection{Definition of the Kalman Filter}\label{sec:Definition of the Kalman Filter}
The Kalman filter (without control input) assumes an underlying system of a discrete, unobserved state process $x$ and an observation process $z$ given as
\begin{equation}\label{equ:state process kalman filter}
    \begin{split}
        x_k = F_k x_{k-1} + w_k, \quad \text{and} \quad z_k = H_k x_k + v_k,
    \end{split}
\end{equation}
where $F_k$ is the state transition matrix, $w_k \sim N(0, Q_k)$ is the process noise, $H_k$ is the observation matrix and $v_k \sim N(0, R_k)$ is the observation noise. The initial state $x_0$ and all noise terms $(w_k)_K$ and $(v_k)_k$ are assumed to be mutually independent.

Then the Kalman filter, which is optimal in this setting, initializes the prediction of $x$ as $\hat{x}_{0|0} = \E[x_0]$ and of its covariance by $P_{0|0} = \operatorname{Cov}(x_0)$.
Then the Kalman filter proceeds in two steps, the prediction and the update step, which can be summarized as
\begin{align*}
        \hat{x}_{k|k-1} &= F_k \hat{x}_{k-1|k-1},  &
        P_{k | k-1} &= F_k P_{k-1|k-1} F_k^\top + Q_k, \tag{Predict} \\
        \hat{x}_{k|k} &= \hat{x}_{k|k-1} + K_k \tilde y_k,  &
        P_{k | k} &= (I - K_k H_k) P_{k|k-1} , \tag{Update} \\
\end{align*}
where
\begin{align*}
        \tilde y_k &= z_k - H_k \hat{x}_{k|k-1}, & 
        S_k &= H_k P_{k|k-1} H_k^\top + R_k,&
        K_k &= P_{k|k-1} H_k^\top S_k^{-1}.
\end{align*}

\subsection{Kalman Filter Applied to \texorpdfstring{\Cref{exa:brownian motion uncertain drift}}{Example 5.1}}\label{sec:Kalman Filter Applied to exa}
There are different ways to set up the Kalman filter for the filtering task of \Cref{exa:brownian motion uncertain drift}, see \Cref{rem:other ways to apply Kalman filter}. 
Here, we shall use the set-up resembling the computation of  conditional expectation as we did it in the example.
In particular, we have the constant signal process $x_k = \mu$, which we want to predict, hence, $F_k =1, w_k = 0, Q_k = 0$, and we use only one step of the observation process, where we assume to make all observations concurrently with\footnote{It is important not to mix up $X_{t_k}$ and $X_0$ which refer to the model of \eqref{equ:model equ exa BM w drift} with $x_k$ which refer to the state process of the Kalman filter \eqref{equ:state process kalman filter}. For more clarity, we write here $X_0$ (instead of $x_0$) for the starting value of $(X_t)_t$.} 
$$z_1 = (X_{t_1}-X_0, \dotsc, X_{t_n}-X_0)^\top = (\mu t_k + \sigma W_{t_k})_{1\leq k \leq n}^\top,$$ 
hence, $H_1 = (t_1, \dotsc, t_n)^\top$ and $v_1 =  (W_{t_1}, \dotsc, W_{t_n})^\top \sigma \sim N(0, R)$, with $R_{i,j} = \sigma^2 \min{(t_i, t_j)}$. The initial values are $\hat{x}_{0|0} = a$ and $P_{0|0} = b^2$ and since the state process is constant we have $\hat{x}_{1|0} = \hat{x}_{0|0} $ and $P_{1|0} =P_{0|0} $. Moreover, 
$$\tilde y_1 = (X_{t_1}-X_0 - t_1 a, \dotsc, X_{t_n}-X_0 -t_n a)^\top,$$ 
and $S_1 = b^2 (t_i t_j)_{i,j} + R$. A direct calculation shows that $S_1 = \operatorname{Cov}(z_1)$, which implies that $S_1 = \tilde\Sigma_{11}$ holds, for $\tilde\Sigma_{11}$ defined in \Cref{exa:brownian motion uncertain drift}. Moreover, note that $P_{1|0} H_k^\top = b^2 (t_1, \cdots, t_n) = \operatorname{Cov}(\mu, z_1^\top) = \tilde \Sigma_{2,1}$, hence, $K_k = \tilde \Sigma_{2,1}  \tilde\Sigma_{11}^{-1}$. Therefore, the update step leads to the a posteriori prediction
\begin{equation*}
    \hat{x}_{1|1} = a + \tilde \Sigma_{2,1} \tilde\Sigma_{11}^{-1} (X_{t_1}-X_0 - t_1 a, \dotsc, X_{t_n}-X_0 -t_n a)^\top,
\end{equation*}
which coincides with the prediction $\hat \mu$ computed in \Cref{exa:brownian motion uncertain drift}. Moreover, it is easy to verify that the a posteriori covariance satisfies $P_{k|k} = \hat{\Sigma}$.

\begin{rem}\label{rem:other ways to apply Kalman filter}
    As mentioned before, the Kalman filter could also be applied differently, leading to the same result. In particular, and probably more naturally, one could use the increments $z_1 = (X_{t_1}- X_{t_0}, \dotsc, X_{t_n} - X_{t_{n-1}})$ as observations. Using them as concurrent observations, one only has to adapt $H_1, v_1, R$ in the above computations. Checking that this coincides with the given solution in \Cref{exa:brownian motion uncertain drift} might be a bit tedious. However, it is easy to adjust the computation in \Cref{exa:brownian motion uncertain drift} to also use the increments, making it easy to verify that this coincides with the Kalman filter. 
    Conditioning on the same information, whether using the increments or the direct values of the path $X_{t_k}$, yields identical conditional distributions. Therefore, the resulting computed $\hat \mu$ must also be the same.
    Hence, all 4 mentioned ways lead to the same result. 
    Since the increments are mutually independent, using them also allows for a multi-step (recursive) application of the Kalman filter (while this does not work with the correlated path values $X_{t_k}$). It is important to note that the increments could be used as observations once at a step or multiple at a step and in arbitrary order, always leading to the same resulting estimator $\hat{\mu}$ as was outlined, e.g., by \citet{stackexchange1}.
\end{rem}

\section{Experimental Details}\label{sec:Experimental Details}
Our experiments are based on the implementation used by \citet{krach2022optimal}, which is available at \url{https://github.com/FlorianKrach/PD-NJODE}. Therefore, we refer the reader to its Appendix for any details that are not provided here.

\subsection{A Practical Note on Self-Imputation}\label{sec:A Practical Note on Self-Imputation}
In real world settings, data is usually incomplete. In experiments, we have seen that self-imputation is an effective way to deal with missing values, outperforming imputing the last observation or $0$s or only using the signature as input without directly passing the current observation to the model. 
Using an input-output setting, where some input variables are not output variables, leads to the problem that self-imputation can not be used, since the model does not predict an estimate for all input variables. 
On the other hand, using all input variables also as output variables, might lead to a worsened performance, if the added input variables overshadow the actual target variables in the loss function.
There are different approaches to overcome this issue. 
\begin{itemize}
    \item One can first train a separate model for predicting all input variables, and then use this model for imputation of missing values while training the second model with the actual target variables. This is probably the best possible imputation, but also the most expensive to train.
    \item An alternative is to jointly train both networks, and to do so efficiently, use only a different readout network $\tilde g_{\tilde \theta}$ for the two models, while sharing the other neural networks $f_\theta$ and $\rho_\theta$. Hence, two different losses are used to train the two models. Variants, where one of the optimizers only has access to the parameters of its respective readout network, can be used. 
    \item Another approach is to use only one model predicting all input and the target variables at once and weighting the different variables in the loss function. Then a reasonable self-imputation and good results for the target prediction can by achieved by changing the weighting throughout the training from putting all weight on the input variables in the beginning of the training to putting all weight on the target variables later on.
\end{itemize}
 
We note that if a large number of training paths is available, training a different model for each output variable will usually lead to an improvement of the results, since no overshadowing of the different variables in the loss can happen.
If this is of importance, then it might be beneficial to use the first approach or to use the second approach with a different readout network for each target variable.

On the other hand, if we can computationally afford to train a very large architecture with many hidden neurons for many epochs with a small learning rate but only have access to a limited number of training paths, then having one model that predicts everything could be even better in terms of generalization to new paths. If we suspect that there might be some hidden features that are valuable for predicting both the outputs and the inputs, 
one can even benefit from multi-task learning \citep{MultitaskCaruana1997,HeissImplReg3,andersson2024extending,HeissInductiveBias2024}.
Training one model that predicts both can help to learn more reliable features in the hidden state rather than overfitting its features to only be useful for predicting the outputs. In practice, it can still be beneficial to weight down the loss for predicting the inputs to mitigate numerical problems related to overshadowing the loss for the outputs.

\subsection{Differences between the Implementation and the Theoretical Description of the NJODE}\label{sec:Differences between the Implementation and the Theoretical Description of the PD-NJODE}
Since we use the same implementation of the PD-NJODE, all differences between the implementation and the theoretical description listed in \citet[Appendix~D.1.1]{krach2022optimal} also apply here.

\subsection{Details for Synthetic Datasets}\label{sec:Details for Synthetic Datasets}
Below we list the standard settings for all synthetic datasets. Any deviations or additions are listed in the respective subsections of the specific datasets.

\textbf{Dataset.} 
We use the Euler scheme to sample paths from the given stochastic processes on the interval $[0,1]$, i.e., with ${T}=1$ and a discretisation time grid with step size $0.01$. At each time point we observe the process with probability $p=0.1$. We sample between $20'000$ and $100'000$ paths of which $80\%$ are used as training set and the remaining $20\%$ as validation set. Additionally, a test set with $4'000$ to $5'000$ paths is generated. 

\textbf{Architecture.}
We use the PD-NJODE with the following architecture. The latent dimension is $d_H \in \{100,200\}$ and all 3 neural networks have the same structure of 1 hidden layer with $\operatorname{ReLU}$ or $\tanh$ activation function and $100$ nodes. The signature is used up to truncation level $3$, the encoder is recurrent and the both the encoder and decoder use a residual connection.

\textbf{Training.}
We use the Adam optimizer with the standard choices $\beta = (0.9, 0.999)$, weight decay of $0.0005$ and learning rate $0.001$. Moreover, a dropout rate of $0.1$ is used for every layer and training is performed with a mini-batch size of $200$ for $200$ epochs.
The PD-NJODE models are trained with the loss function \eqref{equ:Psi}.

\subsubsection{Details for Scaled Brownian Motion with Uncertain Drift}
\textbf{Dataset.} $20'000$ paths are sampled for the combined training and validation set and the test set has $5'000$ independent paths.

\textbf{Architecture.} We use $d_H=100$ and $\tanh$ activation function.

\subsubsection{Details for Geometric Brownian Motion with Uncertain Parameters}
\textbf{Dataset.} $100'000$ paths are sampled for the combined training and validation set and the test set has $5'000$ independent paths.

\textbf{Architecture.} We use $d_H=100$ and $\operatorname{ReLU}$ activation function for the first experiments. In the convergence study we use $d_H=200$.

\subsubsection{Details for CIR Process with Uncertain Parameters}
\textbf{Dataset.} $100'000$ paths are sampled for the combined training and validation set and the test set has $4'000$ independent paths.

\textbf{Architecture.} We use $d_H=200$ and $\operatorname{ReLU}$ activation function for the first experiments. For Experiment 2, the empirical performance is slightly better when the encoder and decoder do not use residual connections (validation loss of $0.446$ compared to $0.450$), therefore, we report these results.

\subsubsection{Details for Brownian Motion Filtering}
\textbf{Dataset.} $40'000$ paths are sampled for the combined training and validation set and the test set has $4'000$ independent paths.

\textbf{Architecture.} We use $d_H=100$ and $\tanh$ activation functions.

\subsubsection{Details for Classifying a Brownian Motion}
\textbf{Dataset.} $40'000$ paths are sampled for the combined training and validation set and the test set has $4'000$ independent paths.

\textbf{Architecture.} We use $d_H=200$ and $\operatorname{ReLU}$ activation functions.

\subsubsection{Details for Loss Comparison on Black--Scholes}\label{sec:Details for Loss Comparison on Black--Scholes}
\textbf{Dataset.} The geometric Brownian motion model described in \citet[Appendix~F.1]{herrera2020neural} is used with the same parameters. $20'000$ paths are sampled for the combined train and validation set, no test set is used.

\textbf{Architecture.} We use $d_H=100$ and $\operatorname{ReLU}$ activation function.

\section{Inductive Bias}
We have proven in \Cref{thm:MC convergence Yt} that the IO NJODE is asymptotically unbiased. In \Cref{sec:Examples,sec:Experiments}, we have studied settings, where we can simulate arbitrarily many training paths. In such setting we can rely on \Cref{thm:MC convergence Yt} if we sample sufficiently many training paths. However, in settings where we observe a limited number of real-world training paths without being able to simulate further training paths, the inductive bias becomes more important. The inductive bias of NJODEs has been discussed in \citet[Appendix~B]{andersson2024extending} and in \citet{HeissInductiveBias2024} and this discussion is also applicable to IO NJODE. These insights on the inductive bias can be helpful for understanding when IO NJODE will be able to generalize well from a few observed training paths to new unseen test paths.

\end{document}